\documentclass{article}

     \PassOptionsToPackage{numbers, compress}{natbib}

\usepackage[main, final]{neurips_2026}
\usepackage{amsmath}
\usepackage{amssymb}
\usepackage{mathtools}
\usepackage{amsthm}
\usepackage{graphicx}
\usepackage{wrapfig}
\usepackage{caption}
\usepackage{stmaryrd}
\theoremstyle{plain}
\newtheorem{theorem}{Theorem}[section]
\newtheorem{proposition}[theorem]{Proposition}

\theoremstyle{definition}

\theoremstyle{remark}
\newtheorem{remark}[theorem]{Remark}
\usepackage{fontawesome5}
\newcommand{\codelink}{\href{https://github.com/Ism-ail11/MCWC}{\faGithub\;MCWC}}

\usepackage[utf8]{inputenc} 
\usepackage[T1]{fontenc}    
\usepackage{hyperref}       
\usepackage{url}            
\usepackage{booktabs}       
\usepackage{amsfonts}       
\usepackage{nicefrac}       
\usepackage{microtype}      
\usepackage{xcolor}         
\usepackage[ruled,vlined,linesnumbered]{algorithm2e}

\usepackage{pgfplots}
\pgfplotsset{compat=1.18}  
\usepgfplotslibrary{
  groupplots,    
  fillbetween    
}

\usepackage{booktabs}        
\usepackage{multirow}        
\usepackage{tabularx}        
\usepackage{threeparttable}  
\usepackage{threeparttablex} 
\usepackage{siunitx}         
\usepackage{pgfplots}
\usepackage{pgfplotstable}
\usetikzlibrary{arrows.meta,positioning,fit,calc,patterns,decorations.pathreplacing}
\pgfplotsset{compat=1.18}
\definecolor{refpink}{RGB}{255,20,147}
\newcommand{\pinkref}[1]{\textcolor{refpink}{\ref{#1}}}
\usepackage{booktabs} \usepackage[table]{xcolor}
 \usepackage{multirow}
\usepackage{siunitx}
 \definecolor{best}{RGB}{198,239,206}   
 \definecolor{second}{RGB}{255,242,204} 
 
\pgfplotsset{
  mygrid/.style={grid=both,minor tick num=1,major grid style={opacity=0.25},minor grid style={opacity=0.1}},
  mylegend/.style={legend columns=2,legend cell align=left,legend pos=south east, fill=white, fill opacity=0.8, draw=none},
  mymark/.style={mark=*},
}
 \usepackage{tikz}
\title{Motion-Compensated Weight Compression}

\usepgfplotslibrary{fillbetween}      
\usetikzlibrary{intersections}        
\usetikzlibrary{arrows.meta,positioning,fit,calc,decorations.pathreplacing}
\usetikzlibrary{arrows.meta,positioning}
\definecolor{icmlBlue}{RGB}{42,92,172}
\definecolor{icmlOrange}{RGB}{219,112,46}
\definecolor{icmlGreen}{RGB}{33,140,70}
\definecolor{icmlPurple}{RGB}{128,94,168}
\definecolor{panelBG}{RGB}{248,249,251}
\definecolor{edgeGray}{RGB}{110,110,110}

\tikzset{
  >={Latex[length=2mm]},
  box/.style={draw, rounded corners=2pt, thick, align=center, inner sep=3pt, minimum height=8mm, fill=white},
  pane/.style={draw=edgeGray!40, rounded corners=3pt, fill=panelBG, inner sep=4pt},
  flow/.style={thick, draw=edgeGray, -{Latex[length=2mm]}},
  flowA/.style={thick, draw=icmlBlue, -{Latex[length=2mm]}},
  flowB/.style={thick, draw=icmlOrange, -{Latex[length=2mm]}},
  flowC/.style={thick, draw=icmlGreen, -{Latex[length=2mm]}},
  note/.style={draw=edgeGray, densely dashed, rounded corners=2pt, inner sep=2pt, font=\scriptsize, fill=white},
  tiny/.style={font=\scriptsize}
}
\definecolor{t3c}{RGB}{0,86,156}
 \newcommand{\best}[1]{\cellcolor{green!15}\textbf{#1}}
 \newcommand{\second}[1]{\cellcolor{yellow!18}\textbf{#1}}
\author{%
  Ismail Lamaakal \\
  Multidisciplinary Faculty of Nador\\
  Mohammed Premier University\\
  Oujda 60000, Morocco\\
  \texttt{ismail.lamaakal@ump.ac.ma} \\
}

\begin{document}

\maketitle

\begin{abstract}
Neural network weights are increasingly a bottleneck for deployment, yet most compression pipelines treat layers independently and overlook cross-layer redundancy induced by function-preserving symmetries. We propose \emph{Motion-Compensated Weight Compression (MCWC)}, a weight-only codec that aligns permutation-symmetric blocks (e.g., hidden units and attention heads) to maximize cross-layer correspondence, turning depth into a predictable sequence. In the aligned coordinate system, MCWC uses a lightweight layer-sequential predictor with periodic keyframes and encodes only quantized prediction residuals using a learned entropy model trained under a rate distortion objective. A simple decoder reconstructs deployable weights by entropy decoding, dequantization, predictor-driven reconstruction, and inverse alignment, enabling fast weight materialization for inference. Across Transformer language modeling and vision classification, MCWC improves the rate accuracy Pareto frontier over strong quantization and learned weight-codec baselines, while maintaining competitive decode time. Ablations confirm that alignment, prediction, entropy modeling, and keyframe scheduling are each necessary for the full gains. Our code is available via \codelink.
\end{abstract}

\section{Introduction} \label{sec:intro}
\begin{wrapfigure}[20]{r}{0.48\linewidth}
\vspace{-3mm}
\captionsetup{font=footnotesize,skip=2pt,width=\linewidth}
\centering
\includegraphics[width=0.48\textwidth]{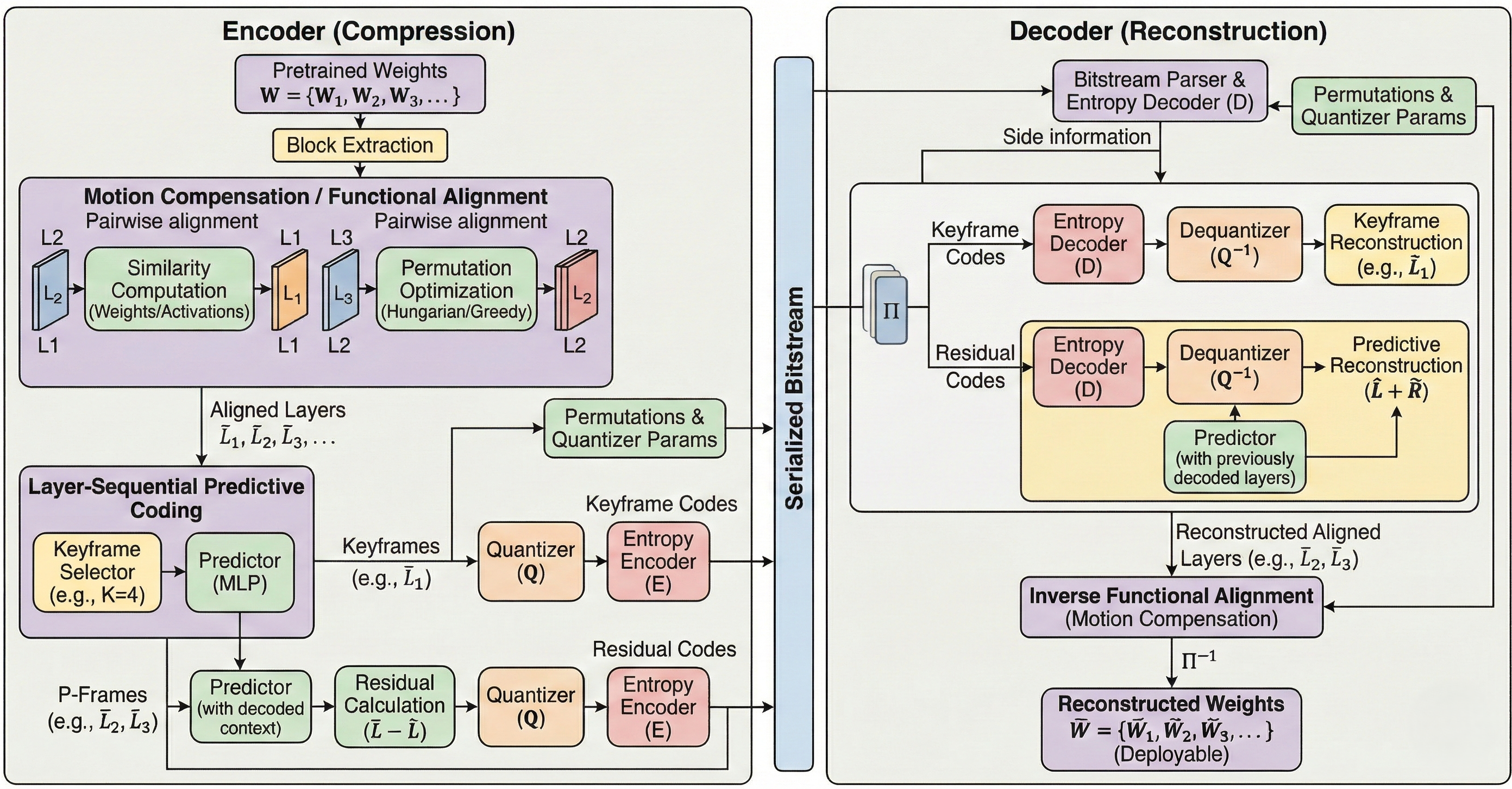}
\caption{\textbf{\textit{Layers-as-video overview of MCWC}}. The encoder transforms depth into a predictable sequence by aligning functionally equivalent blocks via permutation ($\Pi$). Keyframes (e.g., $\bar{L}_1$) are coded absolutely, while subsequent P-frames are encoded as quantized residuals relative to predictions derived from previously \emph{decoded} context (shown as the local decode loop). The decoder reconstructs aligned layers by adding entropy-decoded residuals to predictions, followed by inverse alignment ($\Pi^{-1}$) to restore the original tensor structure for deployment. \textbf{MCWC should not be evaluated as a replacement for low-precision inference kernels; it addresses the orthogonal problem of compact checkpoint storage, transport, and repeated materialization.}}
\label{fig:layers_as_video}
\end{wrapfigure}
Neural networks are increasingly deployed as multi-billion-parameter foundation models \citep{touvron2023llama2}, and the dominant deployment cost is often not compute but the storage, bandwidth, and cold-start overhead required to stage checkpoints into accelerators \citep{fu2024serverlessllm,sheng2023flexgen}. Existing compression pipelines typically treat each layer as an independent object and focus on local quantization or sparsification \citep{xiao2023smoothquant,frantar2022obc,frantar2023gptq,lin2023awq,frantar2023sparsegpt}, which leaves substantial redundancy untapped because modern architectures exhibit strong cross-layer similarity and large function-preserving symmetries \citep{entezari2021permutation,ainsworth2022gitrebasin}. MCWC views depth as a sequential signal and compresses pretrained weights by explicitly exploiting inter-layer predictability after removing symmetry-induced ``motion''.

The objective is weight-only compression for deployment: given pretrained parameters $W$, an encoder produces a bitstream $b$ and a decoder reconstructs deployable weights $\tilde{W}(b)$ such that task performance is preserved while the serialized representation is compact and decoding is fast. The central premise is that many architectures admit re-indexings of hidden units, attention heads, or channel groups that leave the input--output function invariant when adjacent layers are compensated consistently \citep{entezari2021permutation,ainsworth2022gitrebasin}; MCWC uses this freedom to choose a coordinate system in which weights vary slowly across depth. In that aligned coordinate system, a lightweight layer-sequential predictor produces accurate ``P-frame'' predictions for most layers, and only a low-entropy residual stream needs to be quantized and entropy-coded, with periodic ``keyframes'' preventing drift (see Figure~\pinkref{fig:layers_as_video}).

This paper makes four contributions:
\begin{itemize}
    \item  We introduce motion compensation in weight space by selecting blockwise permutations that increase cross-layer correspondence while maintaining the represented function.
    
    \item  We propose a predictive codec with explicit keyframes and residual coding, treating depth as a temporal axis and turning aligned layers into predictable ``frames.''
    
    \item  We train quantization and entropy models end-to-end under a rate--distortion objective so that the discrete residual stream is accurate after decoding and low-entropy under a learned probability model.
    
    \item  We demonstrate improved Pareto frontiers on both language and vision models and provide ablations isolating the effects of alignment, prediction, entropy modeling, and keyframe frequency.
\end{itemize}

\textbf{\textit{MCWC is intended as a checkpoint storage and transport codec rather than an inference-speed method; Appendix~\pinkref{app:deployment_regime} details the deployment regimes, cost amortization model, and break-even analysis.}}

\section{Related Work}
\label{sec:related_work}

Weight compression is traditionally dominated by quantization and pruning.
Post-training quantization and quantization-aware training reduce storage by mapping weights to low-precision representations, often using calibration and per-channel scaling to limit accuracy loss \citep{li2021brecq,yao2022zeroquant,xiao2023smoothquant}.
For large language models, weight-only quantization has become a standard deployment path, including INT8 and 4-bit methods designed for fast inference \citep{dettmers2022llmint8,frantar2023gptq,lin2023awq}.
Pruning pipelines reduce parameter count via unstructured sparsity or structured removal, frequently combined with subsequent quantization and distillation to recover quality \citep{han2016deepcompression,frantar2023sparsegpt,sun2023wanda}.
These approaches are effective but typically operate locally within layers and do not explicitly model cross-layer predictability of weights after accounting for symmetries.

Learned weight codecs and entropy coding have emerged as an alternative viewpoint, where weights are mapped to discrete symbols that are entropy-coded under a learned probability model \citep{choi2018universal,wiedemann2019deepcabac,shi2025entrollm}.
Vector quantization and residual quantization provide codec-like representations, and rate--distortion objectives provide a mechanism for trading task loss against expected codelength \citep{choi2018universal,balle2018variational,minnen2018joint}.
MCWC follows this line but differs in the representation it feeds to the codec: it first removes permutation-induced motion so that depth-wise redundancy becomes visible to a sequential predictor and to an entropy model (see Appx.~\pinkref{app:predictor_entropy} for more details).

Permutation symmetries are a well-known property of multilayer perceptrons and Transformer substructures that admit re-indexings of hidden units, attention heads, or channel groups without changing the represented function when adjacent computations are compensated consistently \citep{entezari2021permutation,ainsworth2022gitrebasin}.
Prior work has leveraged these symmetries for neuron matching and functional alignment, commonly via assignment objectives solved by Hungarian matching or scalable differentiable approximations \citep{ainsworth2022gitrebasin,pena2022sinkhornrebasin}.
MCWC uses this symmetry as a compression primitive by choosing permutations that maximize cross-layer correspondence, analogously to motion estimation in video compression. Permutation overhead is reported as the complete serialized side-information
stream over all aligned tensor families and alignment groups, not as a single
global permutation per layer; Appendix~\pinkref{app:permutation_side_info} gives the
exact accounting.

Predictive coding is the core mechanism in classical video codecs, where inter-frame prediction turns smoothly varying content into low-entropy residuals while keyframes stabilize drift.
Several learned compression systems adopt similar principles for images and videos by learning predictors and entropy models jointly \citep{balle2018variational,minnen2018joint,lu2019dvc}.
MCWC imports the same conceptual structure into weight space by treating depth as a temporal axis, turning aligned layers into predictable ``frames'' and shifting the coding burden to residuals whose entropy is minimized by alignment, prediction, and a learned entropy model (see also Appx.~\pinkref{app:additional_related_work} for more related works).

\vspace{-0.3cm}
\section{Problem Setup}
\label{sec:problem_setup}

\subsection{Problem formulation}
\label{sec:problem_formulation}

Let a pretrained neural network be parameterized by weights $W = \{W_{\ell}\}_{\ell=1}^{L}$, where $W$ denotes the full set of pretrained parameters to be compressed, $L$ is the number of parameterized layers or blocks, and $W_{\ell}$ denotes the weight tensor(s) associated with layer $\ell$.

A compression system consists of an encoder $\mathcal{E}$ and a decoder $\mathcal{D}$, such that $b = \mathcal{E}(W)$ and $\tilde{W}(b) = \mathcal{D}(b)$. Here, $\mathcal{E}(\cdot)$ maps the original weights $W$ to a finite bitstream $b$, while $\mathcal{D}(\cdot)$ reconstructs a compressed set of weights $\tilde{W}(b)$ from the bitstream $b$. The notation $\tilde{W}(b)$ emphasizes that reconstruction is fully determined by the encoded bits.

Let $\mathcal{J}(\cdot)$ denote a task performance functional, such as validation accuracy, negative validation loss, or negative perplexity, with higher values indicating better performance under this convention. The performance drop induced by compression is $\Delta_{\mathcal{J}}(b) = \mathcal{J}(W) - \mathcal{J}\big(\tilde{W}(b)\big)$. Here, $\mathcal{J}(W)$ is the performance of the original pretrained model, $\mathcal{J}(\tilde{W}(b))$ is the performance after decoding the compressed model, and $\Delta_{\mathcal{J}}(b)$ quantifies the degradation associated with the bitstream $b$.

Let $|b|$ denote the number of bits in the bitstream and let $N_{\mathrm{param}}$ denote the number of scalar parameters in $W$; see Appx.~\pinkref{app:codec_spec} for more details. The compression rate in bits per model parameter is $R(b) = \frac{|b|}{N_{\mathrm{param}}}$. Here, $R(b)$ is the average number of bits used to represent one scalar parameter after compression, $|b|$ is the total encoded length in bits, and $N_{\mathrm{param}}$ is the total parameter count of the original model. Total storage in megabytes can be obtained by converting $|b|$ from bits to bytes.

Decoding efficiency is measured by the end-to-end decode time $T_{\mathrm{dec}}(b) = \mathrm{time}\big(\mathcal{D}(b)\big)$, where $T_{\mathrm{dec}}(b)$ denotes the wall-clock time required to reconstruct $\tilde{W}(b)$ from $b$ using the decoder $\mathcal{D}$, and $\mathrm{time}(\cdot)$ denotes the elapsed time of the decoding procedure.

\subsection{Symmetry and motion compensation setup}
\label{sec:symmetry_motion}

Many neural architectures admit function-preserving reparameterizations. The approach exploits permutation symmetries that arise when intermediate units can be re-ordered without changing the computed function, provided adjacent layers are compensated consistently.

Consider a two-layer composition with a pointwise nonlinearity $\phi$, where $h = \phi\!\left(W_{\ell} x\right)$ and $y = W_{\ell+1} h$. Here, $x$ is the input to layer $\ell$, $W_{\ell}$ maps $x$ into a hidden representation, $\phi(\cdot)$ applies elementwise or channelwise to produce the hidden activation $h$, and $W_{\ell+1}$ maps $h$ to the next representation $y$.

Let $\Pi_{\ell}$ be a permutation matrix acting on the hidden dimension of $h$, or a block-permutation operator acting on groups such as attention heads or channel groups. Define the reparameterized weights as $W'_{\ell} = \Pi_{\ell} W_{\ell}$ and $W'_{\ell+1} = W_{\ell+1}\Pi_{\ell}^{-1}$. Here, $\Pi_{\ell}$ permutes the coordinates of the hidden representation. The matrix $\Pi_{\ell}^{-1}$ is the inverse permutation, which exists and equals $\Pi_{\ell}^{\top}$ when $\Pi_{\ell}$ is a permutation matrix. The pair $(W'_{\ell}, W'_{\ell+1})$ is constructed so that the permutation introduced in $W'_{\ell}$ is exactly undone before the next-layer output is formed.

Under architectural conditions specifying where $\phi$ is applied and how the hidden coordinates are consumed, this reparameterization leaves the input--output mapping unchanged, namely $W'_{\ell+1}\,\phi\!\left(W'_{\ell}x\right) = W_{\ell+1}\,\phi\!\left(W_{\ell}x\right)$. The left-hand side is the output computed by the reparameterized weights $(W'_{\ell}, W'_{\ell+1})$, and the right-hand side is the output computed by the original weights $(W_{\ell}, W_{\ell+1})$. The equality states that the permutation does not change the represented function; it only changes the internal coordinate system used to express that function.

This symmetry enables motion compensation across depth by choosing permutations that increase inter-layer predictability; see Appx.~\pinkref{app:symmetry_correctness}. After selecting $\Pi_{\ell}$, define the aligned, motion-compensated weights as $\bar{W}_{\ell} = \Pi_{\ell}(W_{\ell})$, where $\bar{W}_{\ell}$ denotes the aligned version of layer-$\ell$ weights and $\Pi_{\ell}(\cdot)$ denotes applying the permutation or block-permutation along the appropriate hidden dimension(s) of $W_{\ell}$. This aligned representation expresses the same model in a coordinate system where inter-layer redundancy is higher, which lowers the entropy of residuals used by the compressor.

Architectural constraints and sufficient conditions under which the function-invariance relation holds, including blockwise conditions for attention heads and channel groups, are provided in Appendix~\pinkref{app:symmetry_correctness}.

\section{Motion Compensation via Functional Alignment}
\label{sec:alignment}

The motion compensation stage selects function-preserving permutations that increase inter-layer predictability by aligning structurally comparable blocks across depth. The method operates on block partitions that admit permutation symmetries, including intermediate channels in feed-forward networks, attention heads, and channel groups used by normalization or grouped linear maps (see Alg.~\pinkref{alg:codec_full} in Appx.~\pinkref{app:predictor_entropy}) .

\subsection{Block definition}
\label{sec:block_def}

For each layer $\ell$, let $W_{\ell}$ denote the layer parameters. A block-extraction operator $\mathcal{B}_{t}(\cdot)$ maps $W_{\ell}$ to a collection of $B_{\ell,t}$ blocks of a given type $t$, written as $\mathcal{B}_{t}(W_{\ell}) = \left\{\, U_{\ell,t}^{(i)} \,\right\}_{i=1}^{B_{\ell,t}}$. Here, $t$ denotes the block type, $B_{\ell,t}$ is the number of blocks of type $t$ in layer $\ell$, and $U_{\ell,t}^{(i)}$ is the $i$-th extracted block, such as a column group corresponding to a single FFN hidden unit, a weight sub-tensor corresponding to one attention head, or a channel group associated with grouped computation. The alignment objective matches blocks at layer $\ell$ to blocks at layer $\ell-1$ using a permutation $\pi_{\ell,t}$ over block indices, where $\pi_{\ell,t} : \{1,\dots,B_{\ell,t}\} \rightarrow \{1,\dots,B_{\ell,t}\}$. The map $\pi_{\ell,t}$ is a bijection that reindexes blocks within the same type $t$ at layer $\ell$, and $B_{\ell,t}$ is assumed consistent across consecutive layers for the aligned block type.

\subsection{Similarity metrics}
\label{sec:similarity_metrics}

Alignment is driven by a similarity score $s_{\ell,t}(i,j)$ measuring how well
the current candidate block $U_{\ell,t}^{(j)}$ matches the previous aligned
reference block $\bar{U}_{\ell-1,t}^{(i)}$ before applying the current-layer
permutation. A weight-based similarity uses cosine similarity between vectorized blocks, defined as $s^{\mathrm{w}}_{\ell,t}(i,j)
=
\frac{
\left\langle
\mathrm{vec}\!\left(\bar{U}_{\ell-1,t}^{(i)}\right),
\mathrm{vec}\!\left(U_{\ell,t}^{(j)}\right)
\right\rangle
}{
\left\|
\mathrm{vec}\!\left(\bar{U}_{\ell-1,t}^{(i)}\right)
\right\|_{2}
\left\|
\mathrm{vec}\!\left(U_{\ell,t}^{(j)}\right)
\right\|_{2}
}$. Here, $\operatorname{vec}(\cdot)$ vectorizes a block tensor, $\langle \cdot,\cdot\rangle$ denotes the Euclidean inner product, and $\|\cdot\|_{2}$ is the $\ell_{2}$ norm. The quantity $s^{\mathrm{w}}_{\ell,t}(i,j)$ lies in $[-1,1]$ and increases as the two blocks become more aligned in parameter space. An activation-statistic similarity uses a small calibration set $\mathcal{D}_{\mathrm{cal}}$ to compute block-level activation summaries. Let $z_{\ell,t}^{(i)}(x)$ denote the activation vector associated with block $i$ of type $t$ at layer $\ell$ when processing input $x$. The mean activation statistic is $\mu_{\ell,t}^{(i)} = \frac{1}{|\mathcal{D}_{\mathrm{cal}}|}\sum_{x \in \mathcal{D}_{\mathrm{cal}}} z_{\ell,t}^{(i)}(x)$. Here, $|\mathcal{D}_{\mathrm{cal}}|$ is the number of calibration examples, $z_{\ell,t}^{(i)}(x)$ is the activation associated with block $i$, and $\mu_{\ell,t}^{(i)}$ summarizes how that block behaves on the calibration distribution. A cosine similarity between these summaries is $s^{\mathrm{a}}_{\ell,t}(i,j)
=
\frac{
\left\langle
\bar{\mu}_{\ell-1,t}^{(i)},
\mu_{\ell,t}^{(j)}
\right\rangle
}{
\left\|
\bar{\mu}_{\ell-1,t}^{(i)}
\right\|_{2}
\left\|
\mu_{\ell,t}^{(j)}
\right\|_{2}
}$. This quantity measures similarity in activation space, using the same inner product and norm conventions as the weight-based cosine similarity.

A hybrid similarity combines parameter- and activation-based cues as $s_{\ell,t}(i,j) = \alpha\, s^{\mathrm{w}}_{\ell,t}(i,j) + \left(1-\alpha\right)\, s^{\mathrm{a}}_{\ell,t}(i,j)$, where $\alpha \in [0,1]$ is a mixing coefficient. Here, $\alpha$ controls the relative contribution of weight geometry via $s^{\mathrm{w}}_{\ell,t}(i,j)$ and functional behavior on data via $s^{\mathrm{a}}_{\ell,t}(i,j)$.

\subsection{Matching objective and aligned weights}
\label{sec:matching_objective}

Given similarities $s_{\ell,t}(i,j)$, the alignment selects a permutation $\pi_{\ell,t}$ that maximizes total matched similarity, written as $\pi_{\ell,t}^{\star} = \arg\max_{\pi \in \mathfrak{S}_{B_{\ell,t}}} \sum_{i=1}^{B_{\ell,t}} s_{\ell,t}\!\left(i,\pi(i)\right)$, where $\mathfrak{S}_{B_{\ell,t}}$ denotes the set of all permutations of $\{1,\dots,B_{\ell,t}\}$. Here, $\pi(i)$ is the index at layer $\ell$ matched to index $i$ at layer $\ell-1$, and the summation aggregates similarity across all block correspondences.

The permutation $\pi_{\ell,t}^{\star}$ induces a block-permutation operator $\Pi_{\ell,t}$ acting on the appropriate hidden dimension(s) of $W_{\ell}$. The motion-compensated weights for block type $t$ are defined by $\bar{W}_{\ell,t} = \Pi_{\ell,t}\!\left(W_{\ell}\right)$, where $\bar{W}_{\ell,t}$ denotes the aligned parameters associated with type $t$ at layer $\ell$, and $\Pi_{\ell,t}(\cdot)$ permutes blocks according to $\pi_{\ell,t}^{\star}$. The permutation is function-preserving when paired with the corresponding inverse permutation applied to the consuming operators in adjacent computations, consistent with the invariance established earlier.

Alignment can also be viewed through a residual-energy lens when a predictor $\hat{U}_{\ell,t}^{(j)}$ is available. Let the residual for matching $(i,\pi(i))$ be $R_{\ell,t}^{(i)} = U_{\ell,t}^{\left(\pi(i)\right)} - \hat{U}_{\ell,t}^{(i)}$. Here, $U_{\ell,t}^{(\pi(i))}$ is the matched block at layer $\ell$, $\hat{U}_{\ell,t}^{(i)}$ is the predicted block conditioned on aligned context, and $R_{\ell,t}^{(i)}$ is the residual that will be quantized and entropy-coded. Minimizing expected residual energy corresponds to $\pi_{\ell,t}^{\star} = \arg\min_{\pi \in \mathfrak{S}_{B_{\ell,t}}} \sum_{i=1}^{B_{\ell,t}} \left\| U_{\ell,t}^{\left(\pi(i)\right)} - \hat{U}_{\ell,t}^{(i)} \right\|_{F}^{2}$, where $\|\cdot\|_{F}$ denotes the Frobenius norm. The sum aggregates squared residual magnitudes across blocks, so smaller values imply lower-entropy residuals and improved compressibility.

\subsection{Optimization method and complexity}
\label{sec:complexity}

The matching problem above is a linear assignment problem. Using the Hungarian algorithm yields an exact solution with cubic complexity in the number of blocks, namely $\mathrm{cost}_{\mathrm{Hungarian}}(\ell,t) = \mathcal{O}\!\left(B_{\ell,t}^{3}\right)$. Here, $B_{\ell,t}$ is the number of blocks matched for type $t$ at layer $\ell$, and the big-$\mathcal{O}$ notation captures asymptotic runtime scaling.

When $B_{\ell,t}$ is large, an approximate alternative uses a greedy matching with optional local refinement. Let $K$ be the number of candidate matches retained per block after a similarity screening stage. A typical screened greedy approach has complexity $\mathrm{cost}_{\mathrm{screened}}(\ell,t) = \mathcal{O}\!\left(B_{\ell,t}\,K \,\log(B_{\ell,t})\right)$, where $K \ll B_{\ell,t}$ controls the candidate set size. The $\log(B_{\ell,t})$ factor accounts for priority-queue or sorting operations used to select high-similarity pairings.

The total alignment cost across all aligned types and layers is $\mathrm{cost}_{\mathrm{total}} = \sum_{\ell=2}^{L} \sum_{t} \mathrm{cost}(\ell,t)$, where $L$ is the number of layers, $t$ ranges over the aligned block types, and $\mathrm{cost}(\ell,t)$ is instantiated by either the Hungarian complexity or the screened greedy complexity. The summation begins at $\ell=2$ because each layer $\ell$ is matched to its predecessor $\ell-1$.

\vspace{-0.2cm}
\section{Predictive Coding Across Layers}
\label{sec:predictive_coding}

Predictive coding compresses weights by exploiting the sequential structure of depth. After motion compensation, blocks of the same type evolve smoothly across layers, so a predictor can forecast the next block from previously decoded context, and only a low-entropy residual must be encoded (see Figure~\pinkref{fig:alignment_residual}).

\subsection{Layer-sequential predictor}
\label{sec:layer_seq_predictor}

For each layer $\ell$ and block type $t$, let $\bar{W}_{\ell,t}$ denote motion-compensated parameters. A block extraction operator yields the aligned blocks $\mathcal{B}_{t}(\bar{W}_{\ell,t}) = \left\{\, \bar{U}_{\ell,t}^{(i)} \,\right\}_{i=1}^{B_{\ell,t}}$, where $B_{\ell,t}$ is the number of blocks of type $t$ in layer $\ell$, and $\bar{U}_{\ell,t}^{(i)}$ denotes the $i$-th aligned block.

The predictor uses decoded context from earlier layers. Let $\tilde{U}_{\ell-1,t}^{(i)}$ denote the decoded block at layer $\ell-1$ for index $i$, and let $e_{\ell}$ and $e_{t}$ be learned embeddings representing the layer index and block type. The predictor is a parametric mapping $g_{\theta}$, written as $\hat{U}_{\ell,t}^{(i)} = g_{\theta}\!\left(\tilde{U}_{\ell-1,t}^{(i)},\, e_{\ell},\, e_{t}\right)$, where $\hat{U}_{\ell,t}^{(i)}$ is the predicted block, $\theta$ denotes predictor parameters, $\tilde{U}_{\ell-1,t}^{(i)}$ provides decoded context, $e_{\ell}$ encodes depth-dependent effects, and $e_{t}$ encodes block-type-specific structure.

The predictor is conditioned on the \emph{decoded} context
$\tilde{U}_{\ell-1,t}^{(i)}$ rather than the original floating-point block, so
that training and evaluation match the deployment procedure where only
previously decoded weights are available. In the default configuration used for
the main results, alignment is computed using the similarity-based objective
and then kept fixed during continuous codec optimization. Residual-energy
alignment, which recomputes permutations using current predictor outputs, is
treated as an optional variant and is reported only in the ablation study.

Given the prediction, the residual block is defined as $R_{\ell,t}^{(i)} = \bar{U}_{\ell,t}^{(i)} - \hat{U}_{\ell,t}^{(i)}$, where $\bar{U}_{\ell,t}^{(i)}$ is the aligned target block, $\hat{U}_{\ell,t}^{(i)}$ is the predictor output, and $R_{\ell,t}^{(i)}$ is the residual that will be quantized and entropy-coded.
\begin{wrapfigure}[16]{r}{0.47\linewidth}
\vspace{-3mm}
\captionsetup{font=footnotesize,skip=2pt,width=\linewidth}
\centering
\includegraphics[width=\linewidth]{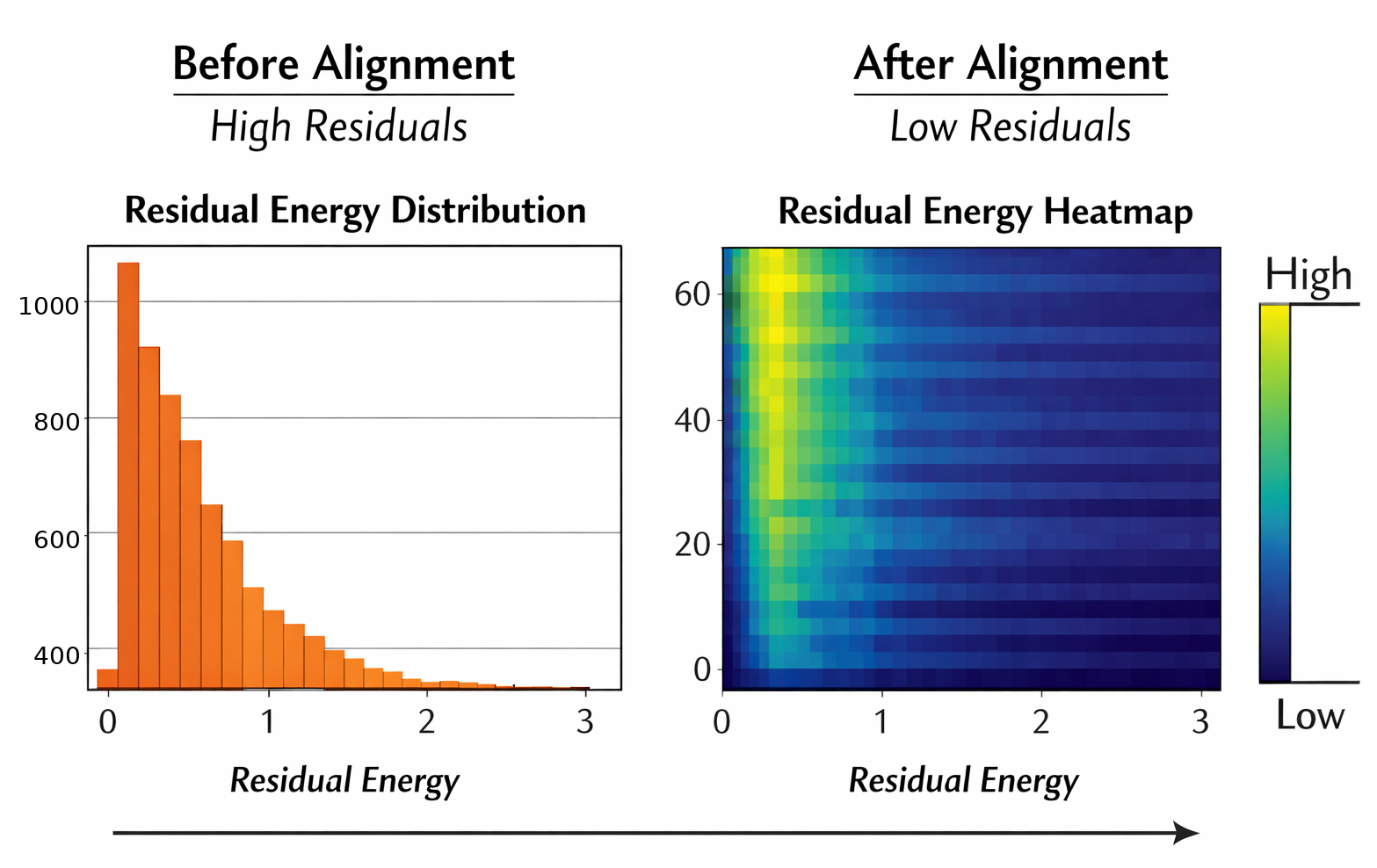}
\caption{
Residual statistics before and after functional alignment.
The visualization compares the distribution or spatial structure of residual magnitudes computed from the residual definition.
The ``before'' case uses identity permutations, while the ``after'' case uses
the default similarity-based assignment permutation $\pi_{\ell,t}^{\star}$.
Residual-energy alignment is treated separately as an optional ablation.
A substantial reduction in residual energy indicates improved predictability and lower entropy for subsequent coding.
}
\label{fig:alignment_residual}
\vspace{-4mm}
\end{wrapfigure}
\vspace{-0.5cm}
\subsection{Keyframes and inter-layer drift control}
\label{sec:keyframes}

Predictive coding is stabilized by keyframes. Let $K$ denote the keyframe interval. A layer index $\ell$ is a keyframe when $k(\ell) = 1$ if $(\ell-1)\bmod K = 0$, and $k(\ell) = 0$ otherwise. Here, $k(\ell)$ is the keyframe indicator, $\ell$ is the layer index, $K$ is the keyframe interval, and $(\ell-1)\bmod K$ is the remainder of $(\ell-1)$ divided by $K$.

For keyframe layers, blocks are encoded in absolute form after alignment. For non-keyframe layers, blocks are encoded as residuals relative to the prediction. This yields the reconstruction rule $\tilde{U}_{\ell,t}^{(i)} = k(\ell)\,\tilde{\bar{U}}_{\ell,t}^{(i)} + \left(1-k(\ell)\right)\left(\hat{U}_{\ell,t}^{(i)} + \tilde{R}_{\ell,t}^{(i)}\right)$, where $\tilde{U}_{\ell,t}^{(i)}$ is the decoded block used for subsequent prediction, $\tilde{\bar{U}}_{\ell,t}^{(i)}$ is the decoded aligned block when $\ell$ is a keyframe, $\hat{U}_{\ell,t}^{(i)}$ is the predicted block, and $\tilde{R}_{\ell,t}^{(i)}$ is the decoded residual. The coefficient $k(\ell)$ selects the keyframe pathway, while $1-k(\ell)$ selects the predictive pathway.

The rate benefit arises because motion compensation increases cross-layer similarity and the predictor captures the smooth component of evolution, reducing the magnitude and entropy of $R_{\ell,t}^{(i)}$. Beyond drift control, keyframes also provide independently decodable segments
that partially restore parallelism during distributed model loading; see
Appendix~\pinkref{app:sequential_loading}.

\vspace{-0.3cm}
\section{Residual Quantization and Rate--Distortion Training}
\label{sec:rate_distortion}

\subsection{Residual quantization and entropy coding}
\label{sec:residual_quant_entropy}

Residuals are encoded after prediction to reduce entropy. For each layer $\ell$, block type $t$, and block index $i$, the residual is $R_{\ell,t}^{(i)} = \bar{U}_{\ell,t}^{(i)} - \hat{U}_{\ell,t}^{(i)}$, where $\bar{U}_{\ell,t}^{(i)}$ is the aligned block, $\hat{U}_{\ell,t}^{(i)}$ is the predicted block, and $R_{\ell,t}^{(i)}$ is the residual to be compressed.

A learned scalar quantizer maps each residual element to a discrete code index. Let $R_{\ell,t}^{(i)}$ denote the residual block for layer $\ell$, tensor-type $t$, and block index $i$. We define a \emph{channel or group axis} for each tensor-type $t$, e.g., output channels for linear/conv weights or attention head groups for multi-head projections. Let $C_{\ell,t}$ be the number of such channels/groups and let $r_{\ell,t}^{(i)} \in \mathbb{R}^{d}$ be the vectorization of $R_{\ell,t}^{(i)}$ with $d$ total elements.

We use \emph{per-channel}, or per-group, quantization parameters: a step-size vector $s_{\ell,t}\in\mathbb{R}_{+}^{C_{\ell,t}}$ and an optional mean vector $m_{\ell,t}\in\mathbb{R}^{C_{\ell,t}}$, which are shared across all elements belonging to the same channel/group.\footnote{Although $r_{\ell,t}^{(i)}$ is flattened to length $d$, the quantizer parameters are \emph{not} learned per-parameter. Instead, each entry of $s_{\ell,t}$ and $m_{\ell,t}$ applies to an entire channel/group and is broadcast to all elements assigned to that channel/group.}

Let $g_{\ell,t}:\{1,\dots,d\}\to\{1,\dots,C_{\ell,t}\}$ map each flattened element index $j$ to its channel/group. Quantization produces integer code indices elementwise as $\big[c_{\ell,t}^{(i)}\big]_j = \left\lfloor \frac{\big[r_{\ell,t}^{(i)}\big]_j - m_{\ell,t,\,g_{\ell,t}(j)}}{s_{\ell,t,\,g_{\ell,t}(j)}} \right\rceil$, where $\lfloor\cdot\rceil$ denotes rounding to the nearest integer and $c_{\ell,t}^{(i)}\in\mathbb{Z}^{d}$ is the vector of code indices. Dequantization reconstructs the residual approximation as $\big[\tilde{r}_{\ell,t}^{(i)}\big]_j = s_{\ell,t,\,g_{\ell,t}(j)}\cdot \big[c_{\ell,t}^{(i)}\big]_j + m_{\ell,t,\,g_{\ell,t}(j)}$, and $\tilde{R}_{\ell,t}^{(i)}$ is obtained by reshaping $\tilde{r}_{\ell,t}^{(i)}$ back to the original block shape. The bitstream stores discrete symbols and side information. The primary encoded symbols are the code indices $\{c_{\ell,t}^{(i)}\}$ for residual-coded layers and the corresponding indices for keyframe-coded layers. Side information includes the permutations defining alignment and the quantizer parameters $(s_{\ell,t}, m_{\ell,t})$ when not fixed or shared.

Entropy coding reduces storage by assigning shorter codes to more likely symbols. Let $p_{\psi}(\cdot)$ be an entropy model with parameters $\psi$ defining a probability mass function over code indices. The ideal codelength for a code index vector is proportional to its negative log-probability, yielding the expected rate proxy $R_{\psi} = \mathbb{E}\!\left[-\log p_{\psi}\!\left(c\right)\right]$, where $c$ denotes the collection of all discrete code indices produced by quantization, $p_{\psi}(c)$ is the probability assigned by the entropy model, and the expectation is taken with respect to the empirical distribution of codes produced by the compressor.

\vspace{-0.5cm}
\subsection{Rate--distortion objective}
\label{sec:rate_distortion_obj}

Training optimizes a rate--distortion trade-off between task performance and compressibility, written as $\mathcal{L} = \mathcal{L}_{\mathrm{task}}\!\left(\tilde{W}\right) + \lambda\,\mathbb{E}\!\left[-\log p_{\psi}\!\left(c\right)\right]$, where $\mathcal{L}$ is the total objective, $\mathcal{L}_{\mathrm{task}}(\tilde{W})$ measures task loss under the decoded weights $\tilde{W}$, $\lambda \ge 0$ controls the rate--distortion trade-off, $c$ denotes the discrete code indices produced by quantization, and $p_{\psi}(c)$ is the entropy model used to estimate coding cost.

The objective uses the expected codelength of the \emph{primary symbol stream}, namely keyframe and residual codes, as a differentiable proxy for the compressed size. The \emph{realized} deployment bitrate additionally includes side information required to reconstruct the model, such as alignment permutations and quantizer parameters. Concretely, the total serialized bitstream length can be written as $R_{\mathrm{tot}} = R_{\mathrm{codes}} + R_{\mathrm{perm}} + R_{\mathrm{qparam}} + R_{\mathrm{meta}}$, where $R_{\mathrm{codes}}$ is the codelength of all quantized symbols, including keyframes and predicted-layer residuals, $R_{\mathrm{perm}}$ is the cost of transmitting alignment permutations, $R_{\mathrm{qparam}}$ is the cost of transmitting quantizer side information such as step sizes and means, and $R_{\mathrm{meta}}$ covers headers and format metadata. The bitrate breakdown reported in Table~\pinkref{tab:bitrate_breakdown} is computed from this full decomposition, and all bits per model parameter/MB numbers in the main tables are measured from the final serialized representation, meaning they include all side information.

\begin{wraptable}[6]{r}{0.46\linewidth}
\vspace{-3mm}
\centering
\captionsetup{font=footnotesize,skip=2pt,width=\linewidth}
\caption{Bitrate breakdown by stream.}
\label{tab:bitrate_breakdown}
\resizebox{\linewidth}{!}{%
\begin{tabular}{lccc}
\hline
Component & Bits & Fraction (\%) & Notes \\
\hline
Keyframe codes        & $6.20\times 10^{8}$ & 25.8  & Absolute-coded layers \\
Residual codes        & $1.55\times 10^{9}$ & 64.6  & Predicted layers \\
Permutation side-info & $1.30\times 10^{8}$ & 5.4   & Alignment permutations \\
Quantizer side-info   & $6.00\times 10^{7}$ & 2.5   & Step sizes and means \\
Other overhead        & $4.00\times 10^{7}$ & 1.7   & Headers, metadata \\
\hline
Total                 & $2.40\times 10^{9}$ & 100.0 &  \\
\hline
\end{tabular}%
}
\vspace{-4mm}
\end{wraptable}

The proxy term $\mathbb{E}[-\log p_{\psi}(c)]$ captures the dominant and smoothly optimizable component of the rate, namely the entropy of the quantized symbol stream under the learned entropy model. In contrast, the side-information terms are either \emph{discrete}, such as permutations updated via the alignment solver, or weakly coupled to the continuous optimization, since quantizer parameters are low-dimensional and transmitted once per layer/type rather than per weight. Since the keyframe interval and block partitions are fixed in a given run, the side-information overhead is comparatively stable across checkpoints within the same configuration, while the symbol-stream entropy varies substantially with predictor quality and quantizer settings. For these reasons, we optimize the differentiable code-stream proxy during training and report the full realized bitrate, including side information, at evaluation time.

Although training uses the proxy objective, operating points are selected using the \emph{total} bitrate computed from the fully serialized stream. For each $\lambda$ sweep, we encode the model, measure the realized bits per model parameter/MB including side information, and choose the checkpoint that best matches the target rate while minimizing task loss. This selection procedure ensures that reported rate--distortion points reflect the true deployment cost, even when side-information terms contribute a non-negligible fraction of the bit budget.

The task term is implemented using knowledge distillation on a small calibration set $\mathcal{D}_{\mathrm{cal}}$. Let $f_{W}(x)$ be the teacher model output using original weights $W$ and let $f_{\tilde{W}}(x)$ be the student output using decoded weights $\tilde{W}$. A distillation loss takes the form $\mathcal{L}_{\mathrm{task}}\!\left(\tilde{W}\right) = \frac{1}{|\mathcal{D}_{\mathrm{cal}}|}\sum_{x \in \mathcal{D}_{\mathrm{cal}}} \ell\!\left(f_{\tilde{W}}(x),\, f_{W}(x)\right)$, where $|\mathcal{D}_{\mathrm{cal}}|$ is the number of calibration examples, $\ell(\cdot,\cdot)$ is a supervised discrepancy such as KL divergence between softened logits or mean-squared error between representations, $f_{\tilde{W}}(x)$ denotes the decoded-model output, and $f_{W}(x)$ denotes the teacher output. The exact default configuration used in all main experiments is summarized in
Appendix~\pinkref{app:default_mcwc_config}.

\definecolor{oursBG}{RGB}{232,245,255}   
\definecolor{bestBG}{RGB}{198,239,206}   
\definecolor{secondBG}{RGB}{255,242,204} %

\vspace{-0.3cm}
\section{Experimental Setup}
\label{sec:experiments}

\paragraph{Models.}
Language modeling experiments use Pythia-1.4B~\citep{biderman2023pythia}, a
decoder-only Transformer language model evaluated under next-token prediction
with a fixed context length of 2048 tokens. The underlying sequence model
follows the Transformer architecture~\citep{vaswani2017attention}. Vision
experiments use ViT-B/16~\citep{dosovitskiy2021vit} as a representative
Transformer-based classifier evaluated with the standard input resolution of
the pretrained checkpoint. All models are compressed at the weight level by
producing a deployment bitstream and reconstructing decoded weights for
inference, without architectural modifications (see Table~\pinkref{tab:bitrate_breakdown}).
\vspace{-0.3cm}
\paragraph{Datasets.}
Language modeling is evaluated on WikiText-103 \citep{merity2016pointer} using the standard validation and test splits. Vision classification is evaluated on the ImageNet-1k validation set \citep{russakovsky2015imagenet}. A separate calibration set is sampled from each training distribution for alignment statistics and for any distillation-based adaptation steps used by compression methods.

\vspace{-0.3cm}
\paragraph{Evaluation metrics.}
Language modeling quality is reported using perplexity on the evaluation set, where lower values indicate better next-token prediction. Vision quality is reported using top-1 classification accuracy on ImageNet-1k, where higher values indicate better classification. Compression effectiveness is reported using the serialized bitstream size (total bits and total megabytes) and bits per model parameter computed from the final deployment representation. Decoding efficiency is reported as wall-clock time to reconstruct deployable weights from the bitstream, including all steps required by the decoder.
\vspace{-0.3cm}
\paragraph{Baselines.}
Comparisons include uniform per-tensor INT8 post-training quantization and uniform per-channel INT8 post-training quantization \citep{jacob2018quantization}, SmoothQuant-style activation smoothing with INT8 \citep{xiao2023smoothquant}, GPTQ-style weight-only INT4 quantization \citep{frantar2023gptq}, AWQ-style activation-aware weight-only INT4 quantization \citep{lin2023awq}, quantization-aware training under INT8 constraints \citep{jacob2018quantization}, quantization-aware training under INT4 constraints using learned step-size quantization \citep{esser2019lsq}, structured block pruning at $50\%$ block sparsity followed by INT8 quantization
with compact block-mask accounting \citep{han2016deepcompression}, a learned scalar quantization baseline with entropy coding \citep{wiedemann2019deepcabac,choi2018universal}.

\vspace{-0.3cm}
\paragraph{Calibration and training protocol.}
The calibration set for language modeling contains 8,192 sequences sampled from the language-model training distribution with context length 2048. The calibration set for vision contains 10,000 training images sampled uniformly at random from ImageNet-1k. Alignment statistics that use activations are computed only on the calibration set. When a baseline requires adaptation, the adaptation budget is fixed across methods and uses teacher outputs from the original model and student outputs from the decoded model, while keeping the evaluation protocol unchanged (see Appx.~\pinkref{app:hyperparams_compute} for more details).

\vspace{-0.2cm}
\subsection{Main results}
\label{sec:main_results}

\paragraph{Pareto frontier points.}
Table~\pinkref{tab:pareto_points} reports representative Pareto operating points for language modeling and vision under matched storage budgets. MCWC is shown at four compression levels ($2\times$, $4\times$, $8\times$, $12\times$), and the baselines are reported at their strongest operating point among the listed budgets. For language modeling, MCWC degrades smoothly as rate decreases, moving from perplexity $15.33$ at $16.6$ bits per model parameter to $16.40$ at $2.8$ bits per model parameter. At comparable low-rate storage, the strongest sparse pipeline in Table~\pinkref{tab:pareto_points} (Prune+Quant at $2.7$ bits per model parameter) reaches perplexity $18.10$, while learned scalar quantization with entropy coding at $2.9$ bits per model parameter reaches perplexity $17.30$, indicating that motion compensation and predictive residual coding reduce the penalty of aggressive compression. At the mid-rate regime around $4.2$--$4.3$ bits per model parameter, MCWC achieves perplexity $15.95$, improving over weight-only INT4 baselines such as uniform PTQ INT4 ($16.85$), and remaining competitive with strong Hessian- or activation-aware quantizers such as GPTQ W4 ($16.40$) and AWQ W4 ($16.25$), which do not exploit cross-layer predictability.

\renewcommand{\arraystretch}{0.8}
\begin{table}[htbp]
\centering
\caption{Pareto frontier operating points for language modeling and vision. Decode time measures end-to-end reconstruction of deployable weights from the serialized bitstream.}
\label{tab:pareto_points}
\small
\resizebox{1\linewidth}{!}{
\begin{tabular}{llrrrr|rrrrr}
\toprule
\multicolumn{1}{c}{Method} & \multicolumn{1}{c}{Comp.} &
\multicolumn{1}{c}{LM bits per model parameter} & \multicolumn{1}{c}{LM MB} & \multicolumn{1}{c}{LM PPL} & \multicolumn{1}{c}{LM s} &
\multicolumn{1}{c}{ViT bits per model parameter} & \multicolumn{1}{c}{ViT MB} & \multicolumn{1}{c}{ViT Acc} & \multicolumn{1}{c}{ViT s} \\
\midrule

\rowcolor{oursBG}
MCWC & $2\times$  & 16.6 & 2910 & \best{15.33} & 4.2 & 16.5 & 178 & \best{81.72} & 0.34 \\
\rowcolor{oursBG}
MCWC & $4\times$  & 8.3  & 1460 & \best{15.46} & 3.1 & 8.3  & 90  & \best{81.45} & 0.23 \\
\rowcolor{oursBG}
MCWC & $8\times$  & 4.2  & 740  & \best{15.95} & 2.5 & 4.2  & 46  & \best{80.92} & 0.18 \\
\rowcolor{oursBG}
MCWC & $12\times$ & 2.8  & 495  & \second{16.40} & 2.3 & 2.8  & 31  & \best{80.15} & 0.16 \\
\midrule

Uniform PTQ INT8 & $4\times$  & 8.6  & 1515 & 15.92 & 2.8 & 8.6  & 93  & 80.95 & 0.21 \\
Uniform PTQ INT4 & $8\times$  & 4.3  & 760  & 16.85 & 2.3 & 4.3  & 47  & 79.35 & 0.16 \\
Per-channel PTQ INT8 & $4\times$ & 8.4  & 1480 & 15.78 & 2.9 & 8.4  & 91  & 81.05 & 0.22 \\
SmoothQuant INT8 & $4\times$  & 8.4  & 1485 & 15.70 & 3.0 & 8.4  & 91  & 81.12 & 0.22 \\
GPTQ W4 & $8\times$  & 4.3  & 755  & \second{16.40} & 2.6 & 4.3  & 47  & 80.10 & 0.19 \\
AWQ W4 & $8\times$   & 4.2  & 745  & \best{16.25} & 2.6 & 4.2  & 46  & \second{80.25} & 0.19 \\
QAT INT8 & $4\times$  & 8.4  & 1480 & \second{15.62} & 3.0 & 8.4  & 91  & \second{81.28} & 0.22 \\
QAT INT4 & $8\times$  & 4.3  & 760  & \second{16.35} & 2.4 & 4.3  & 47  & 80.35 & 0.17 \\
Prune+Quant (50\%) & $12\times$ & 2.7  & 480  & 18.10 & 2.1 & 2.7  & 30  & 77.60 & 0.14 \\
Learned SQ + entropy & $12\times$ & 2.9  & 520  & 17.30 & 2.7 & 2.9  & 32  & 78.70 & 0.19 \\
\bottomrule
\end{tabular}}
\vspace{1mm}

\begin{minipage}{0.98\linewidth}
\scriptsize
\textit{Note.} Rates are reported in bits per model parameter and include all
side information required for reconstruction. The Prune+Quant baseline uses a
structured block-sparse serialized representation with compact block-mask
metadata and entropy-coded retained values; it is not accounted as naive scalar
CSR with one 32-bit index per retained scalar.
\end{minipage}
\end{table}

For vision, MCWC similarly preserves accuracy under increasingly aggressive compression, decreasing from $81.72\%$ at $16.5$ bits per model parameter to $80.15\%$ at $2.8$ bits per model parameter. In the low-rate regime, this corresponds to a $2.55$ point advantage over Prune+Quant (50\%) at $2.7$ bits per model parameter, which reaches $77.60\%$, and a $1.45$ point advantage over learned scalar quantization with entropy coding at $2.9$ bits per model parameter, which reaches $78.70\%$. In the mid-rate regime near $4.2$--$4.3$ bits per model parameter, MCWC attains $80.92\%$ accuracy, outperforming uniform PTQ INT4 ($79.35\%$) and remaining ahead of common weight-only INT4 baselines such as GPTQ W4 ($80.10\%$) and AWQ W4 ($80.25\%$). Decode time remains comparable across methods in Table~\pinkref{tab:pareto_points}; MCWC stays close to PTQ/QAT baselines because predictive reconstruction and inverse-permutation placement are linear-time operations on blocks, and the dominant cost remains entropy decoding and tensor materialization.
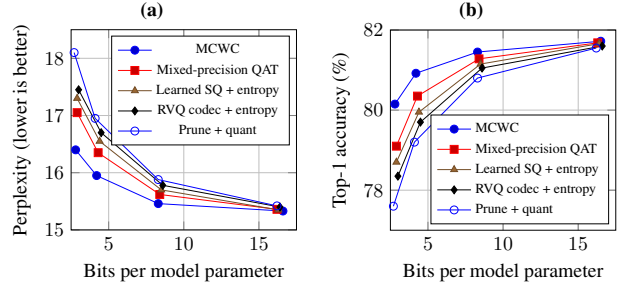
\begin{wrapfigure}[14]{r}{0.58\linewidth}
\vspace{-3mm}
\captionsetup{font=footnotesize,skip=2pt,width=\linewidth}
\centering
\resizebox{0.6\textwidth}{!}{%
\begin{minipage}[htbp]{0.4\textwidth}
\centering
{\bf (a)}\\[-2pt]
\begin{tikzpicture}
\begin{axis}[
  width=1\textwidth,
  height=0.92\textwidth,
  xlabel={Bits per model parameter},
  ylabel={Perplexity (lower is better)},
  xmin=2.5, xmax=17.5,
  ymin=15.0, ymax=18.5,
  grid=both,
  legend style={font=\scriptsize},
  legend pos=north east,
  mark size=2.0
]
\addplot+[mark=*] coordinates {(16.6,15.33) (8.3,15.46) (4.2,15.95) (2.8,16.40)};
\addlegendentry{MCWC}
\addplot+[mark=square*] coordinates {(16.2,15.36) (8.4,15.62) (4.3,16.35) (2.9,17.05)};
\addlegendentry{Mixed-precision QAT}
\addplot+[mark=triangle*] coordinates {(16.3,15.35) (8.5,15.70) (4.4,16.55) (2.9,17.30)};
\addlegendentry{Learned SQ + entropy}
\addplot+[mark=diamond*] coordinates {(16.4,15.40) (8.6,15.78) (4.5,16.70) (3.0,17.45)};
\addlegendentry{RVQ codec + entropy}
\addplot+[mark=o] coordinates {(16.2,15.42) (8.3,15.88) (4.1,16.95) (2.7,18.10)};
\addlegendentry{Prune + quant}
\end{axis}
\end{tikzpicture}
\end{minipage}
\hfill
\begin{minipage}[htbp]{0.4\textwidth}
\centering
{\bf (b)}\\[-2pt]
\begin{tikzpicture}
\begin{axis}[
  width=1\textwidth,
  height=0.92\textwidth,
  xlabel={Bits per model parameter},
  ylabel={Top-1 accuracy (\%)},
  xmin=2.5, xmax=17.5,
  ymin=77.0, ymax=82,
  grid=both,
  legend style={font=\scriptsize},
  legend pos=south east,
  legend cell align=left,
  legend columns=1,
  mark size=2.0
]
\addplot+[mark=*] coordinates {(16.5,81.72) (8.3,81.45) (4.2,80.92) (2.8,80.15)};
\addlegendentry{MCWC}
\addplot+[mark=square*] coordinates {(16.3,81.68) (8.4,81.28) (4.3,80.35) (2.9,79.10)};
\addlegendentry{Mixed-precision QAT}
\addplot+[mark=triangle*] coordinates {(16.4,81.66) (8.5,81.15) (4.4,79.95) (2.9,78.70)};
\addlegendentry{Learned SQ + entropy}
\addplot+[mark=diamond*] coordinates {(16.6,81.60) (8.6,81.05) (4.5,79.70) (3.0,78.35)};
\addlegendentry{RVQ codec + entropy}
\addplot+[mark=o] coordinates {(16.2,81.55) (8.3,80.80) (4.1,79.20) (2.7,77.60)};
\addlegendentry{Prune + quant}
\end{axis}
\end{tikzpicture}
\end{minipage}%
}
\caption{Rate--distortion curves for (a) language modeling and (b) vision classification. Mixed-precision QAT and RVQ codec + entropy baselines are defined in Appendix~\pinkref{app:baseline_configs}.}
\label{fig:rd_curve}
\end{wrapfigure}
\vspace{-0.7cm}
\paragraph{Rate--distortion curves.}
Figure~\pinkref{fig:rd_curve} plots performance as a function of bits per model parameter, using the operating points reported in Table~\pinkref{tab:pareto_points}. For language modeling (Figure~\pinkref{fig:rd_curve}a), the MCWC curve consistently lies below the baselines, meaning lower perplexity at matched storage. The largest margin occurs between roughly $8$ and $3$ bits per model parameter, where alignment makes blocks consistent across depth and prediction removes the smooth component of variation, leaving residuals that compress effectively. For vision (Figure~\pinkref{fig:rd_curve}b), MCWC similarly dominates the frontier, retaining higher top-1 accuracy at the lowest rates where conventional pipelines exhibit steep degradation. Across both domains, the curves confirm that improvements hold throughout the operating range rather than being confined to a single compression point (see Appx~\pinkref{app:additional_experiments} and~\pinkref{sec:ablations} for more results). Appendix~\pinkref{app:alignment_diagnostics} provides direct diagnostics showing
that functional alignment increases adjacent-layer predictability, improves
predictor fit, and reduces normalized residual energy. Appendix~\pinkref{app:encoding_complexity}
analyzes offline encoding cost and alignment scalability, showing that the main
experiments use scalable screened matching rather than cubic Hungarian matching.
Appendix~\pinkref{app:architecture_scope} discusses restricted-symmetry cases,
including GQA, MQA, heterogeneous blocks, head-mixing modules, and MoE-style
routing.


\vspace{-0.4cm}
\section{Conclusion}
\label{sec:conclusion}

MCWC compresses pretrained weights by exploiting cross-layer redundancy instead of treating layers independently. It uses function-preserving alignment to remove permutation-induced variation, then applies layer-sequential predictive coding with keyframes so that most layers are represented by low-entropy residuals. With learned quantization and entropy modeling trained under a rate--distortion objective, MCWC improves storage--quality trade-offs on both language and vision models, and ablations show that alignment, prediction, and entropy modeling are each required for the full benefit. \emph{Several directions} follow naturally from this framing. Richer predictors that exploit longer temporal context across depth, improved side-information coding for permutations and quantizer parameters, and hardware-aware decoding implementations can further reduce overhead and improve wall-clock reconstruction.

\section*{Impact Statement}
This work proposes a weight-only compression codec for neural networks that targets reduced storage, bandwidth, and cold-start latency when deploying large models. By improving the storage--quality trade-off, the method can lower the cost of distributing and serving models, which may reduce energy use and enable deployment on more resource-constrained hardware. The same capability may also increase the ease of replicating and disseminating powerful models, potentially amplifying downstream risks when models are deployed without appropriate safeguards. MCWC does not introduce new model capabilities, training data, or inference-time tools beyond improved weight serialization and reconstruction, and it operates on already-trained parameters. Responsible use therefore depends primarily on the policies and controls surrounding which pretrained models are compressed, how the resulting artifacts are distributed, and what monitoring is applied in deployment.

\section*{Acknowledgments}

We gratefully acknowledge the support of everyone who contributed computational resources, infrastructure access, and technical assistance that made this work possible. The availability of these compute resources was essential for running the compression experiments, calibration and distillation procedures, ablation studies, and evaluation sweeps reported in this paper.

\bibliographystyle{plainnat} 
\bibliography{references}    

\clearpage
\appendix

\clearpage
\appendix

\begin{center}
\rule{\textwidth}{1pt}\\[1.2em]

{\Large \textbf{Motion-Compensated Weight Compression}}\\[1.2em]

\rule{\textwidth}{0.6pt}\\[1.6em]

{\LARGE \textbf{Appendix}}\\[1.2em]
\end{center}

\section*{Table of Contents}

\newcommand{\appcontentssec}[2]{%
  \noindent\hyperref[#1]{\textbf{\ref*{#1}\quad #2}}\dotfill\pageref*{#1}\par
}
\newcommand{\appcontentssub}[2]{%
  \noindent\hspace*{1.8em}\hyperref[#1]{\ref*{#1}\quad #2}\dotfill\pageref*{#1}\par
}
\newcommand{\appcontentssubsub}[2]{%
  \noindent\hspace*{3.6em}\hyperref[#1]{\ref*{#1}\quad #2}\dotfill\pageref*{#1}\par
}

\vspace{0.5em}

\appcontentssec{app:default_mcwc_config}{Default MCWC Configuration Used in the Main Experiments}

\appcontentssec{app:permutation_side_info}{Permutation Side-Information Accounting}

\appcontentssec{app:sequential_loading}{Layer-Sequential Decoding and Distributed Loading}

\appcontentssec{app:symmetry_correctness}{Correctness and Symmetry Conditions}
\appcontentssub{app:notation_fragments}{Notation and computational graph fragments}
\appcontentssub{app:mlp_permutation}{Function-preserving permutations for two-layer MLPs}
\appcontentssub{app:ffn_transformer}{Feed-forward sublayers in Transformer blocks}
\appcontentssub{app:mha_perm}{Multi-head attention head permutations}
\appcontentssub{app:transformer_block_comp}{Composing permutations inside a Transformer block}
\appcontentssub{app:edge_cases}{Edge cases and architectural conditions}
\appcontentssubsub{app:layernorm_case}{Layer normalization}
\appcontentssubsub{app:residual_case}{Residual connections}
\appcontentssubsub{app:tied_embeddings_case}{Tied embeddings and output heads}
\appcontentssubsub{app:grouped_case}{Grouped linear maps and normalization channel groups}
\appcontentssubsub{app:head_mixing_case}{Attention variants and head mixing}
\appcontentssub{app:symmetry_summary}{Summary of admissible symmetry group used by MCWC}

\appcontentssec{app:codec_spec}{Full Codec Specification}
\appcontentssub{app:bitstream_primitives}{Bitstream primitives and conventions}
\appcontentssub{app:header_format}{Header format and global side information}
\appcontentssub{app:perm_format}{Permutation side information format}
\appcontentssub{app:qinfo_format}{Quantizer side information format}
\appcontentssub{app:symbol_streams}{Symbol streams and mode flag}
\appcontentssub{app:decoding_procedure}{Exact decoding procedure}

\appcontentssec{app:predictor_entropy}{Predictor and Entropy Model Details}
\appcontentssub{app:predictor_arch}{Predictor architecture}
\appcontentssub{app:embeddings}{Embedding parameterization}
\appcontentssub{app:entropy_param}{Entropy model parameterization}
\appcontentssub{app:keyframe_rationale}{Keyframe strategy and rationale}

\appcontentssec{app:hyperparams_compute}{Hyperparameters and Compute}
\appcontentssub{app:calib_distill_hparams}{Calibration and distillation hyperparameters}
\appcontentssub{app:codec_hparams}{Quantizer, entropy model, and predictor hyperparameters}
\appcontentssub{app:training_schedule}{Optimization and training schedule}
\appcontentssub{app:compute_hardware}{Compute and hardware}
\appcontentssub{app:baseline_configs}{Full baseline configurations}
\appcontentssub{app:perm_coding}{Permutation side-information coding}

\appcontentssec{app:additional_experiments}{Additional Experiments}
\appcontentssub{app:more_models_datasets}{More models and datasets}
\appcontentssub{app:more_points}{More compression points}
\appcontentssub{app:robustness}{Robustness tests}
\appcontentssub{app:latency}{Latency and decoding benchmarks}

\appcontentssec{app:alignment_diagnostics}{Does Functional Alignment Increase Cross-Layer Predictability?}

\appcontentssec{app:deployment_regime}{Deployment Regime and Cost Amortization}

\appcontentssec{app:encoding_complexity}{Encoding Complexity and Alignment Scalability}

\appcontentssec{sec:ablations}{Ablation Studies}

\appcontentssec{app:architecture_scope}{Architecture Scope and Restricted-Symmetry Cases}

\appcontentssec{app:additional_related_work}{Additional Related Work on Symmetry-Aware and Cross-Layer Weight Compression}

\appcontentssec{sec:discussion_limitations}{Discussion, Limitations and Future Work}
\appcontentssub{sec:discussion}{Discussion}
\appcontentssub{sec:limitations}{Limitations}
\appcontentssub{sec:future_directions}{Future directions}

\clearpage

\section{Default MCWC Configuration Used in the Main Experiments}
\label{app:default_mcwc_config}

This section shows the default MCWC configuration used in all main
experiments. The purpose is to remove ambiguity between optional variants of
the codec and the configuration used to produce the reported results (see Table~\pinkref{tab:default_mcwc_config}).

\begin{table}[htbp]
\centering
\caption{
Default MCWC configuration used in all main experiments. The table specifies
the alignment objective, the reference frame used for matching, the quantizer
initialization rule, the entropy-model context used for keyframes and
predictive frames, and the side information included in the reported bitstream.
This configuration is the one used for the reported rate--quality and
ablation results.
}
\label{tab:default_mcwc_config}
\scriptsize
\renewcommand{\arraystretch}{1.18}
\setlength{\tabcolsep}{4pt}
\resizebox{\linewidth}{!}{
\begin{tabular}{p{3.8cm}p{10.5cm}}
\toprule
\textbf{Component}
&
\textbf{Default configuration}
\\
\midrule
Alignment objective
&
Similarity-based functional alignment. For each layer pair and block type,
MCWC selects the permutation that maximizes blockwise cosine similarity between
the current candidate blocks and the previous aligned reference blocks.
Residual-energy alignment is treated as an optional variant and is not used in
the main results unless explicitly stated.
\\
\midrule
Alignment reference
&
The matching reference is the previous aligned layer, not the original
unaligned canonical ordering. For layer $\ell$, candidate blocks from the
current unaligned layer are matched against the aligned representation of
layer $\ell-1$. This creates a consistent aligned trajectory across depth
instead of independent pairwise alignments.
\\
\midrule
Predictor context
&
For non-keyframe layers, the predictor uses previously decoded aligned blocks
as context. This matches the deployment setting because the decoder only has
access to already decoded weights, not the original floating-point weights.
\\
\midrule
Keyframe handling
&
Keyframes are encoded without predictive residual coding. For entropy-model
conditioning at keyframes, predictor-dependent context variables are set to
zero and a binary keyframe indicator is provided to the entropy model. This
prevents undefined predictor statistics for layers where prediction is
bypassed.
\\
\midrule
Quantizer initialization
&
Quantizer step sizes are initialized from empirical residual statistics. For
each layer and block type, the initial step size is proportional to the
standard deviation of the residual tensor computed on the calibration set.
Reconstruction-error optimization of step sizes is treated as an optional
refinement and is not the default initialization used in the main results.
\\
\midrule
Entropy model input
&
For non-keyframe layers, the entropy model is conditioned on layer index,
block type, quantizer scale, predictor statistics, and residual statistics
available to the decoder. For keyframes, predictor-dependent statistics are
zero-filled and the keyframe flag identifies the coding mode.
\\
\midrule
Side information counted in the rate
&
The reported rate includes all information required for decoding: quantized
residual symbols, keyframe symbols, entropy-model side information, quantizer
scales, permutation side information, block masks when applicable, tensor
headers, coding-mode flags, and metadata required to reconstruct the checkpoint.
\\
\midrule
Assignment solver
&
The main experiments use scalable screened approximate matching with optional
local refinement. Hungarian matching is an exact reference solver but is not
the default solver used for the reported main results.
\\
\bottomrule
\end{tabular}
}
\end{table}

\paragraph{Alignment objective.}
The default alignment stage is similarity-based. For each block type $t$ and
layer $\ell$, MCWC computes a similarity matrix between the previous aligned
reference blocks and the current candidate blocks:
\begin{equation}
S_{\ell,t}(i,j)
=
\frac{
\left\langle
\operatorname{vec}\!\left(\bar{U}_{\ell-1,t}^{(i)}\right),
\operatorname{vec}\!\left(U_{\ell,t}^{(j)}\right)
\right\rangle
}{
\left\|
\operatorname{vec}\!\left(\bar{U}_{\ell-1,t}^{(i)}\right)
\right\|_{2}
\left\|
\operatorname{vec}\!\left(U_{\ell,t}^{(j)}\right)
\right\|_{2}
}.
\label{eq:default_alignment_similarity}
\end{equation}
Here, $\bar{U}_{\ell-1,t}^{(i)}$ is the $i$-th block in the previous aligned
layer, while $U_{\ell,t}^{(j)}$ is the $j$-th candidate block in the current
unaligned layer. The selected permutation is
\begin{equation}
\pi_{\ell,t}^{\star}
=
\arg\max_{\pi\in\mathfrak{S}_{B_{\ell,t}}}
\sum_{i=1}^{B_{\ell,t}}
S_{\ell,t}\big(i,\pi(i)\big).
\label{eq:default_alignment_assignment}
\end{equation}
This choice ensures that alignment is chained through the already aligned
coordinate system, rather than recomputed independently between two unaligned
canonical orderings.

\paragraph{Predictive coding mode.}
For non-keyframe layers, the predictor operates on decoded aligned context:
\begin{equation}
\hat{U}_{\ell,t}^{(i)}
=
g_{\theta}
\left(
\tilde{U}_{\ell-1,t}^{(i)}, e_{\ell}, e_t
\right),
\label{eq:default_predictor_context}
\end{equation}
where $\tilde{U}_{\ell-1,t}^{(i)}$ is the decoded aligned block from the
previous layer, $e_{\ell}$ is the layer embedding, and $e_t$ is the block-type
embedding. The residual is then
\begin{equation}
R_{\ell,t}^{(i)}
=
\bar{U}_{\ell,t}^{(i)}
-
\hat{U}_{\ell,t}^{(i)}.
\label{eq:default_residual}
\end{equation}
Using decoded context avoids train--test mismatch because the decoder cannot
access the original full-precision previous layer during reconstruction.

\paragraph{Keyframe mode.}
For a keyframe layer, the predictor is bypassed and the layer is encoded
without depending on the previous decoded layer. The coding mode is indicated
by a keyframe flag. When the entropy model requires predictor-dependent
conditioning variables, these variables are set to zero for keyframes:
\begin{equation}
c_{\ell,t}^{\mathrm{key}}
=
\big[
e_{\ell}, e_t, s_{\ell,t}, \mathbf{0}, k_{\ell}=1
\big].
\label{eq:keyframe_context}
\end{equation}
where $s_{\ell,t}$ denotes quantizer-scale information and $\mathbf{0}$ denotes
zero-filled predictor-statistic fields. For predictive frames, the context is
\begin{equation}
c_{\ell,t}^{\mathrm{pred}}
=
\big[
e_{\ell}, e_t, s_{\ell,t},
\psi(\hat{U}_{\ell,t}), k_{\ell}=0
\big],
\label{eq:predictive_context}
\end{equation}
where $\psi(\hat{U}_{\ell,t})$ denotes predictor-derived statistics available
during decoding. This makes the entropy-model conditioning well-defined for
both keyframe and predictive coding modes.

\paragraph{Quantizer initialization.}
The default quantizer initialization uses empirical residual statistics. For
each layer and block type, the initial quantizer step size is
\begin{equation}
\Delta_{\ell,t}^{(0)}
=
\gamma
\cdot
\operatorname{Std}
\left(
R_{\ell,t}
\right).
\label{eq:default_quantizer_init}
\end{equation}
where $\operatorname{Std}(R_{\ell,t})$ is the empirical standard deviation of
the residual tensor for block type $t$ at layer $\ell$, and $\gamma$ is a
fixed scaling coefficient. The step sizes are then updated during
rate--distortion optimization. Any reconstruction-error-based initialization
is treated as an optional variant and should not be interpreted as the default
used in the main experiments.

\paragraph{Rate accounting.}
All reported rates correspond to complete serialized checkpoint bitstreams.
The bit count includes all information required for reconstruction:
\begin{equation}
|b|
=
|b_{\mathrm{sym}}|
+
|b_{\mathrm{key}}|
+
|b_{\mathrm{perm}}|
+
|b_{\mathrm{scale}}|
+
|b_{\mathrm{entropy}}|
+
|b_{\mathrm{meta}}|,
\label{eq:default_bitstream_accounting}
\end{equation}
where $|b_{\mathrm{sym}}|$ denotes quantized residual or value symbols,
$|b_{\mathrm{key}}|$ denotes keyframe-coded symbols, $|b_{\mathrm{perm}}|$
denotes permutation side information, $|b_{\mathrm{scale}}|$ denotes quantizer
scale parameters, $|b_{\mathrm{entropy}}|$ denotes entropy-model side
information, and $|b_{\mathrm{meta}}|$ denotes tensor headers, coding-mode
flags, block descriptions, and metadata needed by the decoder. The reported
rate is then
\begin{equation}
R(b)
=
\frac{|b|}{N_{\mathrm{param}}}
\quad
\text{bits/parameter}.
\label{eq:default_rate_accounting}
\end{equation}

\section{Permutation Side-Information Accounting}
\label{app:permutation_side_info}

This appendix clarifies how MCWC accounts for permutation side information.
Because MCWC uses blockwise functional alignment, the decoder must receive the
permutation information required to invert the aligned coordinate system. The
reported bitstream rate therefore includes permutation side information as part
of the total serialized checkpoint.

\paragraph{Granularity of permutation coding.}
Permutation side information is not counted as one permutation per Transformer
layer only. Instead, it is counted at the same granularity at which alignment is
performed. For each aligned layer pair $\ell$, block type $t$, and alignment
group $g$, MCWC stores the permutation needed to reconstruct the original block
ordering:
\begin{equation}
b_{\mathrm{perm}}
=
\bigcup_{\ell,t,g}
b_{\mathrm{perm}}^{(\ell,t,g)} .
\label{eq:perm_side_info_union}
\end{equation}
Here, $\ell$ indexes the layer or block, $t$ indexes the aligned tensor family
such as FFN channels, attention heads, or channel groups, and $g$ indexes
subgroups introduced by tensor partitioning, block chunking, or stage-restricted
alignment. This accounting is stricter than counting only one global
permutation per layer, because it includes all local permutations required by
the decoder.

\paragraph{Theoretical coding cost.}
For a permutation of $B$ elements, the information-theoretic cost of an
unstructured permutation is
\begin{equation}
\log_2(B!)
\quad
\text{bits}.
\label{eq:perm_log_factorial}
\end{equation}
Using Stirling's approximation, this scales as
\begin{equation}
\log_2(B!)
=
B\log_2 B
-
B\log_2 e
+
\mathcal{O}(\log B).
\label{eq:perm_stirling}
\end{equation}
Therefore, the total uncompressed permutation description length across all
alignment groups is
\begin{equation}
|b_{\mathrm{perm}}|_{\mathrm{raw}}
=
\sum_{\ell,t,g}
\log_2
\left(
B_{\ell,t,g}!
\right),
\label{eq:raw_perm_cost}
\end{equation}
where $B_{\ell,t,g}$ is the number of elements permuted in group $g$ for block
type $t$ at layer $\ell$.

\paragraph{Delta-coded Lehmer representation.}
MCWC represents permutations using Lehmer codes and then delta-codes consecutive
Lehmer digits across adjacent aligned groups. Let
\begin{equation}
\mathbf{a}_{\ell,t,g}
=
(a_{\ell,t,g}^{(1)},\ldots,a_{\ell,t,g}^{(B_{\ell,t,g})})
\label{eq:lehmer_digits}
\end{equation}
denote the Lehmer digit sequence for one permutation. Instead of coding each
sequence independently with a fixed-length representation, MCWC codes the
difference
\begin{equation}
\Delta \mathbf{a}_{\ell,t,g}
=
\mathbf{a}_{\ell,t,g}
-
\mathbf{a}_{\ell-1,t,g}
\label{eq:delta_lehmer_digits}
\end{equation}
whenever the corresponding previous aligned group exists. The resulting symbols
are entropy-coded. The serialized permutation cost is therefore
\begin{equation}
|b_{\mathrm{perm}}|
=
\sum_{\ell,t,g}
\left[
-\log_2
p_{\phi}
\left(
\Delta \mathbf{a}_{\ell,t,g}
\right)
\right]
+
|b_{\mathrm{perm,meta}}|,
\label{eq:entropy_perm_cost}
\end{equation}
where $p_{\phi}$ is the entropy model for delta-coded Lehmer symbols and
$|b_{\mathrm{perm,meta}}|$ includes the metadata required to identify alignment
groups, block sizes, coding modes, and reset points.

\paragraph{Why permutation overhead can be larger than a single-assignment estimate.}
A simple estimate based only on the number of attention-head or FFN-channel
assignments may underestimate the reported permutation overhead. The reported
permutation side information includes all aligned tensor families, all
subgroups, all layer pairs, reset markers, stage boundaries, keyframe
boundaries, block partition metadata, and entropy-model overhead required for
decoding. Therefore, the reported percentage should be interpreted as the
complete serialized permutation stream, not merely as the theoretical cost of a
single global assignment per layer.

\paragraph{Scaling limitation.}
Although delta-coded Lehmer symbols reduce overhead in standard regimes,
permutation side information may become non-negligible at extreme compression
rates or for very wide models. This is because the useful residual stream
shrinks as the target rate decreases, while the permutation metadata does not
necessarily shrink at the same rate. For this reason, MCWC disables or coarsens
alignment groups when the expected reduction in residual entropy is smaller
than the cost of transmitting the corresponding permutation:
\begin{equation}
\Delta H_{\mathrm{residual}}^{(\ell,t,g)}
>
|b_{\mathrm{perm}}^{(\ell,t,g)}|.
\label{eq:perm_gain_condition}
\end{equation}
Alignment is retained only when the residual-entropy reduction exceeds the
permutation coding cost. This rule prevents MCWC from spending bits on
permutations that do not improve the final rate--distortion trade-off.

\paragraph{Takeaway.}
The permutation overhead reported by MCWC corresponds to the complete serialized
side-information stream required for exact decoding, including local alignment
groups, tensor-family identifiers, reset points, block metadata, and entropy
coding overhead. Delta-coded Lehmer coding reduces this cost, but the overhead
can still become important at extreme compression rates or in very wide models.
Accordingly, MCWC treats permutation coding as part of the rate--distortion
objective and applies alignment only when the saved residual bits justify the
additional side information.

\section{Layer-Sequential Decoding and Distributed Loading}
\label{app:sequential_loading}

MCWC uses predictive coding across depth. As a result, non-keyframe layers are
decoded from previously reconstructed aligned context. This design improves
compression because it turns depth-wise redundancy into a low-entropy residual
stream, but it also introduces a layer-sequential dependency during decoding.

\paragraph{Sequential dependency.}
For a non-keyframe layer $\ell$, the decoder reconstructs the aligned block as
\begin{equation}
\tilde{U}_{\ell,t}^{(i)}
=
g_{\theta}
\left(
\tilde{U}_{\ell-1,t}^{(i)}, e_{\ell}, e_t
\right)
+
\tilde{R}_{\ell,t}^{(i)},
\label{eq:sequential_decode_dependency}
\end{equation}
where $\tilde{U}_{\ell-1,t}^{(i)}$ is the previously decoded aligned block,
$g_{\theta}$ is the predictor, and $\tilde{R}_{\ell,t}^{(i)}$ is the decoded
residual. Thus, layer $\ell$ cannot be fully reconstructed before the required
context from layer $\ell-1$ is available.

\paragraph{Impact on distributed loading.}
This sequential dependency can restrict layer-wise parallelization during model
loading in distributed serving systems, including pipeline-parallel and
tensor-parallel environments. In a purely independent layerwise representation,
different layers can be loaded or materialized in parallel. In MCWC, predictive
segments must be decoded in order unless they are separated by keyframes.
Therefore, MCWC trades some loading parallelism for a smaller serialized
checkpoint representation.

\paragraph{Keyframe-based segment parallelism.}
Periodic keyframes mitigate this limitation. Let $K$ denote the keyframe
interval. A layer $\ell$ is a keyframe when
\begin{equation}
k(\ell)
=
\mathbb{I}
\left[
(\ell-1)\bmod K = 0
\right].
\label{eq:keyframe_indicator_parallel}
\end{equation}
Each keyframe starts an independently decodable segment. If the model has $L$
layers and keyframe interval $K$, the number of decoding segments is
\begin{equation}
N_{\mathrm{seg}}
=
\left\lceil
\frac{L}{K}
\right\rceil .
\label{eq:num_decode_segments}
\end{equation}
These segments can be assigned to different workers during model loading. Within
each segment, decoding remains sequential, but different segments can be
decoded concurrently.

\paragraph{Parallel decoding schedule.}
Let segment $s$ contain layers
\begin{equation}
\mathcal{S}_s
=
\left\{
(s-1)K+1,\ldots,\min(sK,L)
\right\}.
\label{eq:segment_definition}
\end{equation}
The decoder first reconstructs the keyframe layer of each segment, then
processes the remaining predictive layers within the segment. The loading
schedule is therefore
\begin{equation}
\mathrm{Decode}(W)
=
\bigparallel_{s=1}^{N_{\mathrm{seg}}}
\mathrm{DecodeSegment}(\mathcal{S}_s),
\label{eq:segment_parallel_decode}
\end{equation}
where $\bigparallel$ denotes segment-level parallel execution. This design
preserves the compression benefit of predictive coding while exposing
coarse-grained parallelism through keyframes.

\begin{table}[htbp]
\centering
\caption{
Decoding parallelism induced by the keyframe interval. Smaller keyframe
intervals create more independently decodable segments and therefore expose more
parallelism during distributed loading, but they also increase the number of
absolute-coded layers and can reduce compression efficiency. Larger intervals
improve predictive compression but impose longer sequential chains during
decoding. MCWC therefore uses keyframes as a rate--parallelism trade-off.
}
\label{tab:keyframe_parallelism}
\scriptsize
\renewcommand{\arraystretch}{1.15}
\setlength{\tabcolsep}{4pt}
\resizebox{\linewidth}{!}{
\begin{tabular}{cccc}
\toprule
\textbf{Number of layers $L$}
&
\textbf{Keyframe interval $K$}
&
\textbf{Independent segments $\lceil L/K\rceil$}
&
\textbf{Loading implication}
\\
\midrule
24 & 4 & 6 & More segment-level parallelism; higher keyframe cost. \\
24 & 6 & 4 & Balanced compression and loading parallelism. \\
24 & 8 & 3 & Lower keyframe cost; longer sequential chains. \\
32 & 4 & 8 & Stronger distributed loading parallelism. \\
32 & 8 & 4 & Moderate segment-level parallelism. \\
32 & 16 & 2 & Stronger compression, weaker loading parallelism. \\
\bottomrule
\end{tabular}
}
\end{table}

\paragraph{Practical implication.}
MCWC should not be interpreted as a fully layer-independent representation.
Instead, it is a segmented predictive representation. Keyframes allow the
decoder to trade compression rate against loading parallelism. For
single-machine loading, the sequential dependency has limited impact because
decode time remains small compared with offline encoding. For distributed
loading, shorter keyframe intervals can be selected when parallel materialization
is more important than maximum compression.

\paragraph{Takeaway.}
Layer-sequential prediction improves compression but restricts fully independent
layer-wise loading. Periodic keyframes reduce this restriction by creating
independently decodable segments. Thus, MCWC exposes a tunable trade-off between
rate, prediction length, and distributed loading parallelism.

\section{Correctness and Symmetry Conditions}
\label{app:symmetry_correctness}

\subsection{Notation and computational graph fragments}
\label{app:notation_fragments}

Let a network be a composition of parameterized modules. For a module $m$ with input $x$ and output $y$, write
\begin{equation}
y \;=\; F_m(x; \theta_m),
\label{eq:app_module_def}
\end{equation}
where $F_m(\cdot;\cdot)$ is the module mapping and $\theta_m$ is its parameter set.

A \emph{permutation operator} acting on a channel or head dimension of size $n$ is represented by a permutation matrix $\Pi \in \{0,1\}^{n\times n}$ satisfying
\begin{equation}
\Pi^\top \Pi \;=\; I_n,
\qquad
\Pi^{-1} \;=\; \Pi^\top,
\label{eq:app_perm_props}
\end{equation}
where $I_n$ is the $n\times n$ identity matrix, and $\Pi^{-1}$ exists and equals $\Pi^\top$.

A \emph{block permutation} generalizes $\Pi$ to act on grouped channels (e.g., heads or channel groups). For group size $g$ and number of groups $n_g$, define
\begin{equation}
\Pi^{\mathrm{blk}}
\;=\;
\Pi \otimes I_g,
\label{eq:app_block_perm}
\end{equation}
where $\Pi \in \{0,1\}^{n_g\times n_g}$ permutes groups, $I_g$ is the $g\times g$ identity, and $\otimes$ denotes the Kronecker product. In \eqref{eq:app_block_perm}, $\Pi^{\mathrm{blk}}$ permutes groups while preserving the internal order within each group.

\subsection{Function-preserving permutations for two-layer MLPs}
\label{app:mlp_permutation}

Consider a two-layer MLP with a coordinate-wise nonlinearity applied on the hidden dimension:
\begin{equation}
h \;=\; \phi\!\left(W_1 x + b_1\right),
\qquad
y \;=\; W_2 h + b_2,
\label{eq:app_mlp}
\end{equation}
where $x \in \mathbb{R}^{d_{\mathrm{in}}}$, $h \in \mathbb{R}^{d_{\mathrm{hid}}}$, and $y \in \mathbb{R}^{d_{\mathrm{out}}}$.
The matrices satisfy $W_1 \in \mathbb{R}^{d_{\mathrm{hid}}\times d_{\mathrm{in}}}$ and $W_2 \in \mathbb{R}^{d_{\mathrm{out}}\times d_{\mathrm{hid}}}$,
and the biases satisfy $b_1 \in \mathbb{R}^{d_{\mathrm{hid}}}$ and $b_2 \in \mathbb{R}^{d_{\mathrm{out}}}$.

\paragraph{Coordinate-wise nonlinearity assumption.}
Throughout this appendix, ``elementwise'' means that $\phi:\mathbb{R}^{d_{\mathrm{hid}}}\!\to\mathbb{R}^{d_{\mathrm{hid}}}$ acts independently
and identically on each coordinate, i.e.,
\begin{equation}
\phi(u)
\;=\;
\big(\varphi(u_1),\dots,\varphi(u_{d_{\mathrm{hid}}})\big),
\label{eq:app_phi_coordinatewise}
\end{equation}
for a shared scalar function $\varphi:\mathbb{R}\!\to\!\mathbb{R}$ (e.g., ReLU, GELU).
This definition excludes nonlinearities with \emph{distinct learned parameters per channel} unless those parameters are permuted accordingly
(see Remark~\pinkref{rem:param_act_perm}).

Let $\Pi \in \{0,1\}^{d_{\mathrm{hid}}\times d_{\mathrm{hid}}}$ be a permutation matrix acting on the hidden dimension.
Define reparameterized weights and biases by
\begin{equation}
W_1' \;=\; \Pi W_1,
\qquad
b_1' \;=\; \Pi b_1,
\qquad
W_2' \;=\; W_2 \Pi^{-1},
\qquad
b_2' \;=\; b_2.
\label{eq:app_mlp_reparam}
\end{equation}
In \eqref{eq:app_mlp_reparam}, the first-layer output channels are permuted by left-multiplying $W_1$ and $b_1$ with $\Pi$,
and the second-layer input channels are compensated by right-multiplying $W_2$ with $\Pi^{-1}$.

\begin{proposition}
\label{prop:mlp_perm_invariance}
For any permutation matrix $\Pi$ and any coordinate-wise nonlinearity $\phi$ of the form \eqref{eq:app_phi_coordinatewise},
the MLP input--output mapping is invariant under \eqref{eq:app_mlp_reparam}, meaning that for all $x$,
\begin{equation}
W_2'\,\phi\!\left(W_1' x + b_1'\right) + b_2'
\;=\;
W_2\,\phi\!\left(W_1 x + b_1\right) + b_2.
\label{eq:app_mlp_invariance}
\end{equation}
\end{proposition}

\begin{proof}
Let $u = W_1 x + b_1$ and $u' = W_1' x + b_1'$. Using \eqref{eq:app_mlp_reparam},
\begin{equation}
u'
\;=\;
\Pi W_1 x + \Pi b_1
\;=\;
\Pi (W_1 x + b_1)
\;=\;
\Pi u.
\label{eq:app_uprime}
\end{equation}
Because $\phi$ is coordinate-wise as in \eqref{eq:app_phi_coordinatewise}, permuting coordinates before applying $\phi$
permutes the outputs after applying $\phi$:
\begin{equation}
\phi(u')
\;=\;
\phi(\Pi u)
\;=\;
\Pi \phi(u).
\label{eq:app_phi_commute}
\end{equation}
Substituting \eqref{eq:app_phi_commute} into the output with $W_2'$ gives
\begin{equation}
W_2' \phi(u') + b_2'
\;=\;
W_2 \Pi^{-1} \Pi \phi(u) + b_2
\;=\;
W_2 \phi(u) + b_2,
\label{eq:app_mlp_proof_end}
\end{equation}
which proves \eqref{eq:app_mlp_invariance}.
\end{proof}

\begin{remark}[Parameterized per-channel nonlinearities]
\label{rem:param_act_perm}
If the nonlinearity has per-channel parameters, e.g. $\phi_a(u)_i = \varphi(u_i; a_i)$ (as in PReLU),
then the commutation in \eqref{eq:app_phi_commute} holds only if the parameters are permuted consistently.
Specifically, defining $a'=\Pi a$ yields
\begin{equation}
\phi_{a'}(\Pi u) \;=\; \Pi \phi_{a}(u),
\end{equation}
and the invariance remains valid under the augmented reparameterization
$(W_1',b_1',W_2',b_2',a') = (\Pi W_1,\Pi b_1,W_2\Pi^{-1},b_2,\Pi a)$.
In our experiments, Transformer MLP activations are parameter-free (e.g., GELU), so \eqref{eq:app_phi_coordinatewise} applies directly.
\end{remark}

\subsection{Feed-forward sublayers in Transformer blocks}
\label{app:ffn_transformer}

A standard Transformer feed-forward network (FFN) in a pre-norm block takes the form
\begin{equation}
\mathrm{FFN}(x)
\;=\;
W_{\mathrm{down}}\,\phi\!\left(W_{\mathrm{up}} x + b_{\mathrm{up}}\right) + b_{\mathrm{down}},
\label{eq:app_ffn}
\end{equation}
where $x \in \mathbb{R}^{d_{\mathrm{model}}}$, $W_{\mathrm{up}} \in \mathbb{R}^{d_{\mathrm{ff}}\times d_{\mathrm{model}}}$, $W_{\mathrm{down}} \in \mathbb{R}^{d_{\mathrm{model}}\times d_{\mathrm{ff}}}$, and $d_{\mathrm{ff}}$ is the intermediate width.

Let $\Pi_{\mathrm{ff}} \in \{0,1\}^{d_{\mathrm{ff}}\times d_{\mathrm{ff}}}$ permute intermediate FFN channels. Define
\begin{equation}
W_{\mathrm{up}}' \;=\; \Pi_{\mathrm{ff}} W_{\mathrm{up}},
\qquad
b_{\mathrm{up}}' \;=\; \Pi_{\mathrm{ff}} b_{\mathrm{up}},
\qquad
W_{\mathrm{down}}' \;=\; W_{\mathrm{down}} \Pi_{\mathrm{ff}}^{-1},
\qquad
b_{\mathrm{down}}' \;=\; b_{\mathrm{down}}.
\label{eq:app_ffn_reparam}
\end{equation}
Then the same argument as Proposition~\pinkref{prop:mlp_perm_invariance} yields
\begin{equation}
\mathrm{FFN}'(x)
\;=\;
W_{\mathrm{down}}'\,\phi\!\left(W_{\mathrm{up}}' x + b_{\mathrm{up}}'\right) + b_{\mathrm{down}}'
\;=\;
\mathrm{FFN}(x)
\label{eq:app_ffn_invariance}
\end{equation}
for all $x$. In \eqref{eq:app_ffn_invariance}, the intermediate permutation is canceled by the downstream compensation and elementwise $\phi$ commutes with permutation.

\subsection{Multi-head attention head permutations}
\label{app:mha_perm}

Consider multi-head self-attention with $H$ heads, head dimension $d_h$, and $d_{\mathrm{model}} = H d_h$. Let the per-head projection matrices be assembled into block rows:
\begin{equation}
W_Q \;=\; \begin{bmatrix} W_Q^{(1)} \\ \vdots \\ W_Q^{(H)} \end{bmatrix},
\quad
W_K \;=\; \begin{bmatrix} W_K^{(1)} \\ \vdots \\ W_K^{(H)} \end{bmatrix},
\quad
W_V \;=\; \begin{bmatrix} W_V^{(1)} \\ \vdots \\ W_V^{(H)} \end{bmatrix},
\label{eq:app_qkv_stack}
\end{equation}
where each $W_Q^{(h)},W_K^{(h)},W_V^{(h)} \in \mathbb{R}^{d_h \times d_{\mathrm{model}}}$ and the full matrices satisfy $W_Q,W_K,W_V \in \mathbb{R}^{(H d_h)\times d_{\mathrm{model}}}$.

Let $\Pi_H \in \{0,1\}^{H\times H}$ permute head indices, and define the induced block permutation on the concatenated head dimension:
\begin{equation}
\Pi_{\mathrm{head}}
\;=\;
\Pi_H \otimes I_{d_h}.
\label{eq:app_head_perm}
\end{equation}
Define reparameterized attention projections by
\begin{equation}
W_Q' \;=\; \Pi_{\mathrm{head}} W_Q,
\qquad
W_K' \;=\; \Pi_{\mathrm{head}} W_K,
\qquad
W_V' \;=\; \Pi_{\mathrm{head}} W_V,
\qquad
W_O' \;=\; W_O \Pi_{\mathrm{head}}^{-1},
\label{eq:app_mha_reparam}
\end{equation}
where $W_O \in \mathbb{R}^{d_{\mathrm{model}}\times (H d_h)}$ is the output projection and $W_O'$ compensates the head permutation.

The attention operator computes per-head outputs and concatenates them. Let $\mathrm{Head}^{(h)}(x)$ denote the output of head $h$ before concatenation, and let $\mathrm{Concat}(\cdot)$ concatenate head outputs along the feature dimension. The multi-head attention output is
\begin{equation}
\mathrm{MHA}(x)
\;=\;
W_O\,\mathrm{Concat}\!\left(\mathrm{Head}^{(1)}(x),\dots,\mathrm{Head}^{(H)}(x)\right).
\label{eq:app_mha_def}
\end{equation}
Under \eqref{eq:app_mha_reparam}, the permutation simply reorders the heads inside the concatenation:
\begin{equation}
\mathrm{Concat}\!\left(\mathrm{Head}'^{(1)}(x),\dots,\mathrm{Head}'^{(H)}(x)\right)
\;=\;
\Pi_{\mathrm{head}}\,
\mathrm{Concat}\!\left(\mathrm{Head}^{(1)}(x),\dots,\mathrm{Head}^{(H)}(x)\right).
\label{eq:app_concat_perm}
\end{equation}
Substituting \eqref{eq:app_concat_perm} into the output with $W_O'$ yields invariance:
\begin{equation}
\mathrm{MHA}'(x)
\;=\;
W_O' \,\mathrm{Concat}\!\left(\mathrm{Head}'^{(1)}(x),\dots,\mathrm{Head}'^{(H)}(x)\right)
\;=\;
W_O \Pi_{\mathrm{head}}^{-1} \Pi_{\mathrm{head}} \mathrm{Concat}(\cdots)
\;=\;
\mathrm{MHA}(x).
\label{eq:app_mha_invariance}
\end{equation}
In \eqref{eq:app_mha_invariance}, the head permutation cancels exactly, provided the same permutation is applied consistently to $W_Q,W_K,W_V$ and compensated in $W_O$.

\subsection{Composing permutations inside a Transformer block}
\label{app:transformer_block_comp}

A pre-norm Transformer block can be written as
\begin{equation}
x^{(1)}
\;=\;
x + \mathrm{MHA}(\mathrm{LN}_1(x)),
\qquad
x^{(2)}
\;=\;
x^{(1)} + \mathrm{FFN}(\mathrm{LN}_2(x^{(1)})),
\label{eq:app_pre_norm_block}
\end{equation}
where $\mathrm{LN}_1$ and $\mathrm{LN}_2$ are layer normalization operators and the residual connections add vectors in $\mathbb{R}^{d_{\mathrm{model}}}$.

The invariances in \eqref{eq:app_mha_invariance} and \eqref{eq:app_ffn_invariance} imply that if the internal permutations act only on intermediate head or FFN dimensions and are compensated within the sublayer, then the outputs of $\mathrm{MHA}(\cdot)$ and $\mathrm{FFN}(\cdot)$ are unchanged for any input. Because residual addition is performed in the unchanged model dimension, the block mapping remains unchanged:
\begin{equation}
x^{(2)}{}'
\;=\;
x^{(2)}.
\label{eq:app_block_invariance}
\end{equation}
In \eqref{eq:app_block_invariance}, invariance holds because each residual branch computes the same vector in $\mathbb{R}^{d_{\mathrm{model}}}$ after reparameterization, and the skip connection adds the same $x$.

\subsection{Edge cases and architectural conditions}
\label{app:edge_cases}

\subsubsection{Layer normalization}
\label{app:layernorm_case}

Layer normalization over the model dimension has parameters $\gamma,\beta \in \mathbb{R}^{d_{\mathrm{model}}}$ and acts as
\begin{equation}
\mathrm{LN}(x)
\;=\;
\gamma \odot \frac{x - \mu(x)\mathbf{1}}{\sqrt{\sigma^2(x) + \epsilon}}
\;+\;
\beta,
\label{eq:app_layernorm}
\end{equation}
where $\mu(x)$ and $\sigma^2(x)$ are the mean and variance computed across the model dimension, $\mathbf{1}$ is the all-ones vector, $\epsilon$ is a small constant, and $\odot$ is elementwise multiplication.

MCWC permutations act on intermediate FFN channels and head indices, not on the model dimension $d_{\mathrm{model}}$ at the LN interface. Therefore \eqref{eq:app_layernorm} is unaffected by the alignment permutations. If a permutation is applied on the model dimension, invariance would require permuting $\gamma$ and $\beta$ accordingly and ensuring every consumer of the permuted model coordinates is compensated, which is not done in MCWC.

\subsubsection{Residual connections}
\label{app:residual_case}

Residual addition requires both summands to lie in the same coordinate system. In \eqref{eq:app_pre_norm_block}, permutations are confined within $\mathrm{MHA}$ and $\mathrm{FFN}$ and are compensated internally by $W_O$ and $W_{\mathrm{down}}$, so the residual branch outputs remain in the original model coordinate system. This ensures that the residual sums are well-defined and identical before and after reparameterization.

\subsubsection{Tied embeddings and output heads}
\label{app:tied_embeddings_case}

In language models with tied input and output embeddings, the token embedding matrix $E \in \mathbb{R}^{V\times d_{\mathrm{model}}}$ is shared with the output projection. Because MCWC does not permute the model dimension at the embedding interface, tying is preserved without modification. If a method permuted the model dimension, invariance under weight tying would require applying the same permutation to both the embedding rows in the forward path and to the output projection, maintaining exact sharing, which imposes a global constraint across the network. MCWC avoids this complication by restricting permutations to symmetry dimensions internal to sublayers.

\subsubsection{Grouped linear maps and normalization channel groups}
\label{app:grouped_case}

For grouped linear maps with group size $g$ and $n_g$ groups, parameters can be partitioned into group blocks. The invariance statement remains valid when permutations are applied at the group level using \eqref{eq:app_block_perm}, provided every operator that consumes the grouped dimension is compensated by the inverse permutation on that grouped dimension. This is the setting in which MCWC aligns channel groups, including group-normalization-adjacent channel partitions, by treating each group as a block.

\subsubsection{Attention variants and head mixing}
\label{app:head_mixing_case}

The head-permutation invariance in \eqref{eq:app_mha_invariance} assumes that heads are concatenated and then linearly mixed only through $W_O$. Any operation that mixes head channels before $W_O$, such as explicit cross-head mixing in the attention computation or head-wise nonlinearities that couple heads, would break the separability used in \eqref{eq:app_concat_perm}. In such architectures, only permutations that preserve the coupling structure remain function-preserving, and MCWC alignment must respect the admissible symmetry group defined by that coupling.

\subsection{Summary of admissible symmetry group used by MCWC}
\label{app:symmetry_summary}

MCWC applies permutations only along dimensions that satisfy both of the following conditions:
\begin{equation}
\Pi \text{ acts on a dimension consumed only through pointwise or block-separable operations,}
\label{eq:app_cond_separable}
\end{equation}
\begin{equation}
\text{Every introduction of } \Pi \text{ is paired with an adjacent compensation by } \Pi^{-1}.
\label{eq:app_cond_compensate}
\end{equation}
In \eqref{eq:app_cond_separable}, separability ensures commutation with elementwise nonlinearities and concatenation structure, and in \eqref{eq:app_cond_compensate}, compensation ensures that the overall input--output mapping is unchanged. These are satisfied by intermediate FFN channels and by attention head indices in standard Transformer blocks, and they define the function-preserving permutation family exploited by MCWC for motion compensation.

\section{Full Codec Specification}
\label{app:codec_spec}

\subsection{Bitstream primitives and conventions}
\label{app:bitstream_primitives}

The codec serializes a deployable model as a single bitstream
\begin{equation}
b \;=\; \mathrm{Serialize}\!\left(\mathcal{H},\, \mathcal{S}\right),
\label{eq:app_bitstream_def}
\end{equation}
where $\mathcal{H}$ is a fixed-size header region containing metadata required to parse the stream and $\mathcal{S}$ is a concatenation of entropy-coded segments corresponding to layers and block types.

The bitstream is parsed in a deterministic order. Let the canonical traversal order be
\begin{equation}
\mathcal{O}
\;=\;
\left\{(\ell,t)\;:\;\ell \in \{1,\dots,L\},\; t \in \mathcal{T}_{\ell}\right\},
\label{eq:app_traversal_order}
\end{equation}
where $L$ is the number of layers and $\mathcal{T}_{\ell}$ is the ordered set of block types present at layer $\ell$.

For each pair $(\ell,t)\in \mathcal{O}$, the stream contains a \emph{block record} with the following components:
\begin{equation}
\mathrm{Record}_{\ell,t}
\;=\;
\left(\mathrm{Mode}_{\ell},\; \mathrm{Perm}_{\ell,t},\; \mathrm{QInfo}_{\ell,t},\; \mathrm{Codes}_{\ell,t}\right),
\label{eq:app_record_def}
\end{equation}
where $\mathrm{Mode}_{\ell}$ indicates keyframe or predictive coding, $\mathrm{Perm}_{\ell,t}$ stores the alignment permutation side information, $\mathrm{QInfo}_{\ell,t}$ stores any quantizer side information needed to dequantize codes, and $\mathrm{Codes}_{\ell,t}$ is the entropy-coded symbol stream (either absolute codes for keyframes or residual codes for predicted layers).

\subsection{Header format and global side information}
\label{app:header_format}

The header stores all information that must be known before decoding any record. The header fields are
\begin{equation}
\mathcal{H}
\;=\;
\left(
\mathrm{Magic},\,
\mathrm{Version},\,
L,\,
K,\,
\mathrm{ArchID},\,
\mathrm{ParamShapes},\,
\mathrm{BlockSpec},\,
\mathrm{EntropySpec}
\right),
\label{eq:app_header_fields}
\end{equation}
where each component is defined as follows.

The stream identifier and version are
\begin{equation}
\mathrm{Magic} \in \{0,1\}^{32},
\qquad
\mathrm{Version} \in \{0,1\}^{16},
\label{eq:app_magic_version}
\end{equation}
where $\mathrm{Magic}$ is a constant signature and $\mathrm{Version}$ identifies the codec version.

The layer count and keyframe interval are
\begin{equation}
L \in \mathbb{N},
\qquad
K \in \mathbb{N},
\label{eq:app_LK}
\end{equation}
where $L$ is the number of serialized layers and $K$ is the keyframe interval.

The architecture identifier is
\begin{equation}
\mathrm{ArchID} \in \{0,1\}^{32},
\label{eq:app_archid}
\end{equation}
where $\mathrm{ArchID}$ selects a known model family (e.g., decoder-only Transformer, ViT) so the decoder can reconstruct module wiring and block extraction logic deterministically.

Parameter tensor shapes are stored as
\begin{equation}
\mathrm{ParamShapes}
\;=\;
\left\{\mathrm{shape}(W_{\ell,p})\right\}_{\ell=1,p=1}^{L,P_{\ell}},
\label{eq:app_paramshapes}
\end{equation}
where each layer $\ell$ contains $P_{\ell}$ named parameter tensors $W_{\ell,p}$ and $\mathrm{shape}(\cdot)$ provides the integer dimension tuple required to allocate and reshape decoded arrays.

The block specification defines, for each layer $\ell$ and block type $t\in\mathcal{T}_{\ell}$,
\begin{equation}
\mathrm{BlockSpec}
\;=\;
\left\{
\left(\ell,t, B_{\ell,t}, \mathrm{BlockShape}_{\ell,t}, \mathrm{BlockAxis}_{\ell,t}\right)
\right\},
\label{eq:app_blockspec}
\end{equation}
where $B_{\ell,t}$ is the number of blocks, $\mathrm{BlockShape}_{\ell,t}$ is the per-block tensor shape, and $\mathrm{BlockAxis}_{\ell,t}$ defines which axis (or axes) in the original parameter tensors correspond to block indices.

The entropy model specification stores the parameters needed to decode entropy-coded symbols:
\begin{equation}
\mathrm{EntropySpec}
\;=\;
\left(
\mathrm{EntropyModelID},\,
\psi,\,
\mathrm{CodeAlphabet}
\right),
\label{eq:app_entropyspec}
\end{equation}
where $\mathrm{EntropyModelID}$ selects the entropy model family, $\psi$ are the learned model parameters, and $\mathrm{CodeAlphabet}$ specifies the discrete symbol alphabet and any range-coder tables required for decoding.

\subsection{Permutation side information format}
\label{app:perm_format}

For each $(\ell,t)$, the permutation $\pi_{\ell,t}^{\star}$ is stored as a sequence of $B_{\ell,t}$ integers
\begin{equation}
\mathrm{Perm}_{\ell,t}
\;=\;
\left(\pi_{\ell,t}^{\star}(1),\dots,\pi_{\ell,t}^{\star}(B_{\ell,t})\right),
\label{eq:app_perm_seq}
\end{equation}
where $\pi_{\ell,t}^{\star}(i)\in\{1,\dots,B_{\ell,t}\}$ and $\pi_{\ell,t}^{\star}$ is a bijection.

To reduce overhead, $\mathrm{Perm}_{\ell,t}$ is itself entropy-coded using a permutation code. Let $p_{\eta}(\cdot)$ be a permutation entropy model with parameters $\eta$. The permutation codelength proxy is
\begin{equation}
R_{\eta}^{\mathrm{perm}}(\ell,t)
\;=\;
\mathbb{E}\!\left[-\log p_{\eta}\!\left(\mathrm{Perm}_{\ell,t}\right)\right],
\label{eq:app_perm_rate_proxy}
\end{equation}
where the expectation is taken over permutations produced by alignment across layers. In implementation, $p_{\eta}$ can be instantiated by factorizing $\pi_{\ell,t}^{\star}$ into Lehmer codes or by encoding deltas from the identity permutation when alignments are near-identity.

The decoder reconstructs both $\pi_{\ell,t}^{\star}$ and its inverse $\left(\pi_{\ell,t}^{\star}\right)^{-1}$, defined by
\begin{equation}
\left(\pi_{\ell,t}^{\star}\right)^{-1}(j)
\;=\;
i
\quad \text{such that} \quad
\pi_{\ell,t}^{\star}(i)=j.
\label{eq:app_perm_inverse}
\end{equation}
In \eqref{eq:app_perm_inverse}, the inverse mapping is required to place decoded blocks back into the canonical tensor ordering.

\subsection{Quantizer side information format}
\label{app:qinfo_format}

Each record provides the information needed to dequantize codes into either aligned blocks (keyframes) or residuals (predicted layers). For learned scalar quantization, the quantizer side information is
\begin{equation}
\mathrm{QInfo}_{\ell,t}
\;=\;
\left(s_{\ell,t},\, m_{\ell,t},\, \mathrm{Clip}_{\ell,t}\right),
\label{eq:app_qinfo_fields}
\end{equation}
where $s_{\ell,t}$ is the step-size vector, $m_{\ell,t}$ is the per-channel mean, and $\mathrm{Clip}_{\ell,t}$ optionally stores finite support bounds for arithmetic decoding (e.g., minimum and maximum symbol values per channel).

If $s_{\ell,t}$ and $m_{\ell,t}$ are shared across layers within a block type, the stream stores a reference index rather than the full vectors:
\begin{equation}
\mathrm{QInfo}_{\ell,t}
\;=\;
\left(\mathrm{SharedID}_{t}\right),
\label{eq:app_qinfo_shared}
\end{equation}
where $\mathrm{SharedID}_{t}$ points to a global quantizer table stored once in the header.

Dequantization uses the inverse mapping. For a code vector $c_{\ell,t}^{(i)}$, the dequantized vector is
\begin{equation}
\tilde{r}_{\ell,t}^{(i)}
\;=\;
s_{\ell,t}\odot c_{\ell,t}^{(i)} + m_{\ell,t},
\label{eq:app_dequant_again}
\end{equation}
and reshaping yields the decoded tensor $\tilde{R}_{\ell,t}^{(i)}$ or $\tilde{\bar{U}}_{\ell,t}^{(i)}$ depending on mode.

\subsection{Symbol streams and mode flag}
\label{app:symbol_streams}

Each layer $\ell$ uses a mode flag equal to the keyframe indicator
\begin{equation}
\mathrm{Mode}_{\ell}
\;=\;
k(\ell),
\label{eq:app_mode_flag}
\end{equation}
$\mathrm{Mode}_{\ell}=1$ indicates a keyframe layer and $\mathrm{Mode}_{\ell}=0$ indicates a predicted layer.

The symbol payload for each $(\ell,t)$ is a concatenation of per-block code vectors:
\begin{equation}
\mathrm{Codes}_{\ell,t}
\;=\;
\mathrm{Concat}\!\left(c_{\ell,t}^{(1)},\dots,c_{\ell,t}^{(B_{\ell,t})}\right)
\quad \text{if } \mathrm{Mode}_{\ell}=1,
\label{eq:app_codes_keyframe}
\end{equation}
\begin{equation}
\mathrm{Codes}_{\ell,t}
\;=\;
\mathrm{Concat}\!\left(d_{\ell,t}^{(1)},\dots,d_{\ell,t}^{(B_{\ell,t})}\right)
\quad \text{if } \mathrm{Mode}_{\ell}=0,
\label{eq:app_codes_residual}
\end{equation}
where $c_{\ell,t}^{(i)}$ are absolute codes for aligned blocks and $d_{\ell,t}^{(i)}$ are residual codes. Each concatenated stream is entropy-coded using the learned model $p_{\psi}$.

\subsection{Exact decoding procedure}
\label{app:decoding_procedure}

The decoder reconstructs deployable weights deterministically:
\begin{equation}
\tilde{W}
\;=\;
\mathcal{D}(b).
\label{eq:app_decode_def}
\end{equation}
Decoding proceeds in the traversal order \eqref{eq:app_traversal_order}. For each $(\ell,t)$, the decoder performs the following steps.

First, the decoder parses $\mathrm{Mode}_{\ell}$, decodes $\mathrm{Perm}_{\ell,t}$, and builds both $\pi_{\ell,t}^{\star}$ and its inverse using \eqref{eq:app_perm_inverse}. The permutation is then used to define a block-placement operator $\Pi_{\ell,t}$ and inverse placement operator $\Pi_{\ell,t}^{-1}$.

Second, the decoder obtains $\mathrm{QInfo}_{\ell,t}$ and constructs the dequantizer parameters $s_{\ell,t}$ and $m_{\ell,t}$.

Third, the decoder entropy-decodes the symbol stream $\mathrm{Codes}_{\ell,t}$ into code vectors. The dequantization yields either keyframe-aligned blocks or residual blocks:
\begin{equation}
\left\{\tilde{\bar{U}}_{\ell,t}^{(i)}\right\}_{i=1}^{B_{\ell,t}}
\;=\;
\mathrm{Dequantize}\!\left(\left\{c_{\ell,t}^{(i)}\right\}_{i=1}^{B_{\ell,t}},\, s_{\ell,t},\, m_{\ell,t}\right)
\quad \text{if } \mathrm{Mode}_{\ell}=1,
\label{eq:app_decode_keyframe_blocks}
\end{equation}
\begin{equation}
\left\{\tilde{R}_{\ell,t}^{(i)}\right\}_{i=1}^{B_{\ell,t}}
\;=\;
\mathrm{Dequantize}\!\left(\left\{d_{\ell,t}^{(i)}\right\}_{i=1}^{B_{\ell,t}},\, s_{\ell,t},\, m_{\ell,t}\right)
\quad \text{if } \mathrm{Mode}_{\ell}=0.
\label{eq:app_decode_residual_blocks}
\end{equation}
In \eqref{eq:app_decode_keyframe_blocks}, the decoder directly obtains aligned blocks. In \eqref{eq:app_decode_residual_blocks}, the decoder obtains residual blocks that must be added to predictions.

Fourth, if $\mathrm{Mode}_{\ell}=0$, the decoder predicts blocks from previously decoded context. Let $\tilde{U}_{\ell-1,t}^{(i)}$ denote the decoded block at layer $\ell-1$ for block index $i$. The predictor $g_{\theta}$ produces
\begin{equation}
\hat{U}_{\ell,t}^{(i)}
\;=\;
g_{\theta}\!\left(\tilde{U}_{\ell-1,t}^{(i)},\, e_{\ell},\, e_{t}\right),
\label{eq:app_decode_predict}
\end{equation}
where $e_{\ell}$ and $e_t$ are the same embeddings. Reconstruction then follows
\begin{equation}
\tilde{U}_{\ell,t}^{(i)}
\;=\;
\hat{U}_{\ell,t}^{(i)} + \tilde{R}_{\ell,t}^{(i)}.
\label{eq:app_decode_reconstruct_pred}
\end{equation}
If $\mathrm{Mode}_{\ell}=1$, the decoded block equals the decoded aligned block:
\begin{equation}
\tilde{U}_{\ell,t}^{(i)}
\;=\;
\tilde{\bar{U}}_{\ell,t}^{(i)}.
\label{eq:app_decode_reconstruct_key}
\end{equation}

Fifth, the decoder applies inverse placement to restore canonical ordering within the layer tensor(s). Let $\tilde{U}_{\ell,t}$ denote the concatenation of blocks in the aligned index order. The placement step is
\begin{equation}
\tilde{U}_{\ell,t}^{\mathrm{canon}}
\;=\;
\Pi_{\ell,t}^{-1}\!\left(\tilde{U}_{\ell,t}\right),
\label{eq:app_inverse_perm_place}
\end{equation}
where $\tilde{U}_{\ell,t}^{\mathrm{canon}}$ is the block tensor arranged in the canonical ordering expected by the architecture implementation.

Finally, the decoder assembles the layer parameter tensors from all block types using the inverse of the extraction map . Let $\mathcal{B}^{-1}$ denote the deterministic assembly operator:
\begin{equation}
\tilde{W}_{\ell}
\;=\;
\mathcal{B}^{-1}\!\left(\left\{\tilde{U}_{\ell,t}^{\mathrm{canon}}\right\}_{t\in\mathcal{T}_{\ell}},\, \mathrm{ParamShapes}_{\ell}\right),
\label{eq:app_layer_assemble}
\end{equation}
where $\mathrm{ParamShapes}_{\ell}$ provides the shapes required to reshape each assembled tensor.

Full encoder and decoder pseudocode (Algorithm~\pinkref{alg:codec_full})
\label{app:full_pseudocode}

\begin{algorithm}[!t]
\caption{Full MCWC codec: bitstream layout, encoder, and decoder}
\label{alg:codec_full}
\KwIn{Pretrained weights $W=\{W_{\ell}\}_{\ell=1}^{L}$, keyframe interval $K$, block specs $\mathrm{BlockSpec}$, entropy models $(p_{\psi},p_{\eta})$, quantizer params (shared or per-layer), predictor $g_{\theta}$}
\KwOut{Bitstream $b$}
$b \leftarrow \emptyset$\;
Write header $\mathcal{H}$;

\For{$\ell \leftarrow 1$ \KwTo $L$}{
  Compute $\mathrm{Mode}_{\ell} \leftarrow k(\ell)$;
  \ForEach{block type $t \in \mathcal{T}_{\ell}$}{
    Extract blocks $\{U_{\ell,t}^{(i)}\}_{i=1}^{B_{\ell,t}} \leftarrow \mathcal{B}_{t}(W_{\ell})$;
    Obtain similarity scores $s_{\ell,t}(i,j)$;
    Solve for alignment $\pi_{\ell,t}^{\star}$;
    Apply alignment to produce aligned blocks $\{\bar{U}_{\ell,t}^{(i)}\} \leftarrow \Pi_{\ell,t}(\{U_{\ell,t}^{(i)}\})$\;

    Encode permutation side information $\mathrm{Perm}_{\ell,t}$ using $p_{\eta}$ and append to $b$\;
    Encode quantizer side information $\mathrm{QInfo}_{\ell,t}$;

    \eIf{$\mathrm{Mode}_{\ell}=1$}{
      Quantize aligned blocks to codes $\{c_{\ell,t}^{(i)}\}$;
      Entropy-encode $\mathrm{Codes}_{\ell,t}$ using $p_{\psi}$ and append to $b$\;
      Set decoded context buffers $\tilde{U}_{\ell,t}^{(i)} \leftarrow Q^{-1}(c_{\ell,t}^{(i)})$;
    }{
      Predict $\hat{U}_{\ell,t}^{(i)} \leftarrow g_{\theta}(\tilde{U}_{\ell-1,t}^{(i)}, e_{\ell}, e_t)$;
      Form residuals $R_{\ell,t}^{(i)} \leftarrow \bar{U}_{\ell,t}^{(i)} - \hat{U}_{\ell,t}^{(i)}$;
      Quantize residuals to codes $\{d_{\ell,t}^{(i)}\}$;
      Entropy-encode $\mathrm{Codes}_{\ell,t}$  using $p_{\psi}$ and append to $b$\;
      Update decoded context buffers $\tilde{U}_{\ell,t}^{(i)} \leftarrow \hat{U}_{\ell,t}^{(i)} + Q^{-1}(d_{\ell,t}^{(i)})$\;
    }
  }
}

\BlankLine
\KwIn{Bitstream $b$}
\KwOut{Reconstructed weights $\tilde{W}$}
Read header $\mathcal{H}$ from $b$ and parse fields;

\For{$\ell \leftarrow 1$ \KwTo $L$}{
  Compute $\mathrm{Mode}_{\ell} \leftarrow k(\ell)$ ;
  \ForEach{block type $t \in \mathcal{T}_{\ell}$}{
    Entropy-decode $\mathrm{Perm}_{\ell,t}$ using $p_{\eta}$ and construct $\pi_{\ell,t}^{\star}$ and $\left(\pi_{\ell,t}^{\star}\right)^{-1}$;
    Parse $\mathrm{QInfo}_{\ell,t}$ and construct dequantizer parameters $(s_{\ell,t},m_{\ell,t})$\;

    \eIf{$\mathrm{Mode}_{\ell}=1$}{
      Entropy-decode $\mathrm{Codes}_{\ell,t}$ into $\{c_{\ell,t}^{(i)}\}$ using $p_{\psi}$\;
      Dequantize to aligned blocks $\tilde{\bar{U}}_{\ell,t}^{(i)} \leftarrow Q^{-1}(c_{\ell,t}^{(i)})$;
      Set $\tilde{U}_{\ell,t}^{(i)} \leftarrow \tilde{\bar{U}}_{\ell,t}^{(i)}$\;
    }{
      Predict $\hat{U}_{\ell,t}^{(i)} \leftarrow g_{\theta}(\tilde{U}_{\ell-1,t}^{(i)}, e_{\ell}, e_t)$;
      Entropy-decode $\mathrm{Codes}_{\ell,t}$ into $\{d_{\ell,t}^{(i)}\}$ using $p_{\psi}$\;
      Dequantize residuals $\tilde{R}_{\ell,t}^{(i)} \leftarrow Q^{-1}(d_{\ell,t}^{(i)})$;
      Reconstruct blocks $\tilde{U}_{\ell,t}^{(i)} \leftarrow \hat{U}_{\ell,t}^{(i)} + \tilde{R}_{\ell,t}^{(i)}$;
    }

    Apply inverse placement $\tilde{U}_{\ell,t}^{\mathrm{canon}} \leftarrow \Pi_{\ell,t}^{-1}(\tilde{U}_{\ell,t})$;
  }
  Assemble layer tensors $\tilde{W}_{\ell} \leftarrow \mathcal{B}^{-1}(\{\tilde{U}_{\ell,t}^{\mathrm{canon}}\}_{t\in\mathcal{T}_{\ell}}, \mathrm{ParamShapes}_{\ell})$;
}
\end{algorithm}
\section{Predictor and Entropy Model Details}
\label{app:predictor_entropy}

\subsection{Predictor architecture}
\label{app:predictor_arch}

The predictor maps previously decoded context to a prediction of the next aligned block. For each layer $\ell$, block type $t$, and block index $i$, the predictor conditions on the decoded block from the previous layer $\tilde{U}_{\ell-1,t}^{(i)}$ together with lightweight embeddings that encode depth and block structure.

Let $\mathrm{vec}(\cdot)$ flatten a block into a vector and define the flattened context
\begin{equation}
u_{\ell-1,t}^{(i)}
\;=\;
\mathrm{vec}\!\left(\tilde{U}_{\ell-1,t}^{(i)}\right)
\;\in\;
\mathbb{R}^{d_t},
\label{eq:app_flat_context}
\end{equation}
where $d_t$ is the number of scalar parameters in a block of type $t$.

The predictor first maps the flattened block to a shared latent width $d_{\mathrm{lat}}$ using a type-specific projection:
\begin{equation}
z_{\ell-1,t}^{(i)}
\;=\;
P_t\,u_{\ell-1,t}^{(i)} + p_t,
\qquad
P_t \in \mathbb{R}^{d_{\mathrm{lat}}\times d_t},\;\;
p_t \in \mathbb{R}^{d_{\mathrm{lat}}},
\label{eq:app_type_proj}
\end{equation}
where $z_{\ell-1,t}^{(i)} \in \mathbb{R}^{d_{\mathrm{lat}}}$.

\paragraph{Layer and type conditioning.}
We use learned embeddings for the layer index and block type. Specifically, we maintain a layer embedding $e_{\ell}\in\mathbb{R}^{d_{\mathrm{emb}}}$ and a block-type embedding $e_t\in\mathbb{R}^{d_{\mathrm{emb}}}$ with $d_{\mathrm{emb}}=64$. These embeddings are projected into the latent space and added to the latent representation:
\begin{equation}
\bar{z}_{\ell,t}^{(i)}
\;=\;
z_{\ell-1,t}^{(i)} + A_{\ell}\,e_{\ell} + A_{t}\,e_t,
\qquad
A_{\ell},A_t \in \mathbb{R}^{d_{\mathrm{lat}}\times d_{\mathrm{emb}}},
\label{eq:app_latent_conditioned}
\end{equation}
where $A_{\ell}$ and $A_t$ are learned projection matrices shared across blocks.
This formulation encodes depth-dependent drift ($e_{\ell}$) and structural differences across block types ($e_t$) while keeping the conditioning overhead small.

\paragraph{Relation to concatenation.}
Equivalently, \eqref{eq:app_latent_conditioned} can be written as a single linear map applied to the concatenated vector:
\begin{equation}
\bar{z}_{\ell,t}^{(i)}
\;=\;
W_{\mathrm{cond}}
\begin{bmatrix}
z_{\ell-1,t}^{(i)}\\
e_{\ell}\\
e_t
\end{bmatrix}
+b_{\mathrm{cond}},
\label{eq:app_cond_concat}
\end{equation}
with $W_{\mathrm{cond}}\in\mathbb{R}^{d_{\mathrm{lat}}\times(d_{\mathrm{lat}}+2d_{\mathrm{emb}})}$.
In our implementation, we use the additive form in \eqref{eq:app_latent_conditioned}, which is a structured instance of \eqref{eq:app_cond_concat}.

The core predictor is a residual MLP applied in the conditioned latent space:
\begin{equation}
h_{\ell,t}^{(i)}
\;=\;
\bar{z}_{\ell,t}^{(i)} + \mathrm{MLP}_{\theta}\!\left(\bar{z}_{\ell,t}^{(i)}\right),
\label{eq:app_latent_mlp}
\end{equation}
where $\mathrm{MLP}_{\theta}(\cdot)$ is a two-layer MLP with hidden width $4d_{\mathrm{lat}}$ and GELU nonlinearity.

The latent prediction is mapped back to block space using a type-specific output projection:
\begin{equation}
\hat{u}_{\ell,t}^{(i)}
\;=\;
O_t\,h_{\ell,t}^{(i)} + o_t,
\qquad
O_t \in \mathbb{R}^{d_t\times d_{\mathrm{lat}}},\;\;
o_t\in\mathbb{R}^{d_t},
\label{eq:app_out_proj}
\end{equation}
and reshaped to the original block tensor:
\begin{equation}
\hat{U}_{\ell,t}^{(i)}
\;=\;
\mathrm{vec}^{-1}\!\left(\hat{u}_{\ell,t}^{(i)}\right).
\label{eq:app_pred_reshape}
\end{equation}

The residual used by the codec is
\begin{equation}
R_{\ell,t}^{(i)}
\;=\;
\bar{U}_{\ell,t}^{(i)} - \hat{U}_{\ell,t}^{(i)},
\label{eq:app_residual_again}
\end{equation}
where $\bar{U}_{\ell,t}^{(i)}$ is the aligned target block.
\subsection{Embedding parameterization} \label{app:embeddings} The layer embedding table is \begin{equation} \mathcal{E}_{\mathrm{layer}} \;=\; \left\{E_{\ell}\right\}_{\ell=1}^{L}, \label{eq:app_layer_embed_table} \end{equation} where each $E_{\ell}\in\mathbb{R}^{d_{\mathrm{lat}}}$ and $L$ is the number of layers. The block-type embedding table is \begin{equation} \mathcal{E}_{\mathrm{type}} \;=\; \left\{E_t\right\}_{t\in\mathcal{T}}, \label{eq:app_type_embed_table} \end{equation} where $\mathcal{T}$ is the global set of block types and each $E_t\in\mathbb{R}^{d_{\mathrm{lat}}}$. If finer conditioning is needed, a block-index embedding can be added. For $i\in\{1,\dots,B_{\ell,t}\}$, \begin{equation} \bar{z}_{\ell,t}^{(i)} \;=\; z_{\ell-1,t}^{(i)} + E_{\ell} + E_t + E_{t,i}, \label{eq:app_index_embed} \end{equation} where $E_{t,i}\in\mathbb{R}^{d_{\mathrm{lat}}}$ is an embedding tied to block index $i$ within type $t$.

\subsection{Entropy model parameterization}
\label{app:entropy_param}

The entropy model assigns probabilities to quantized codes to estimate the expected compressed size and to enable entropy coding at deployment.
Let $c$ denote a scalar code index produced by the scalar quantizer.
For a fixed $(\ell,t)$ record (layer $\ell$ and tensor-type $t$), we flatten the code stream into a sequence
$\{c_j\}_{j=1}^{M}$, where each $c_j \in \mathbb{Z}$ corresponds to one quantized symbol.

\paragraph{Factorization.}
We use a lightweight conditional factorization across symbols,
\begin{equation}
p_{\psi}(c)
\;=\;
\prod_{j=1}^{M}
p_{\psi}\!\left(c_j \mid \xi_j\right),
\label{eq:app_entropy_factor}
\end{equation}
where $\xi_j$ is a compact conditioning vector for symbol $j$ and $\psi$ are the parameters of the entropy model. In Appendix~\pinkref{app:codec_hparams}, we refer to this predictor-derived conditioning compactly as a low-dimensional context feature; in the final configuration we instantiate it explicitly as $(\mu_{\hat{r}},\sigma_{\hat{r}})$.

\paragraph{Conditioning variables.}
The conditioning vector combines discrete identifiers (layer and tensor-type) with simple local statistics.
In the default MCWC configuration, these statistics are derived from the predictor output associated with the current block:
\begin{equation}
\xi_j
\;=\;
\left(
E_{\ell},\,
E_t,\,
\mu_{\hat{r}},\,
\sigma_{\hat{r}}
\right),
\label{eq:app_entropy_context}
\end{equation}
where $E_{\ell}$ and $E_t$ are learned embeddings for layer index $\ell$ and tensor-type $t$,
and $\mu_{\hat{r}}$ and $\sigma_{\hat{r}}$ are the mean and standard deviation of the predictor output $\hat r$
over the elements that correspond to symbol $c_j$ (computed at the same block granularity used for quantization).

\paragraph{Conditioning when prediction is disabled (Learned SQ + entropy baseline).}
Some baselines operate without motion-compensated prediction and therefore do not produce $\hat r$.
To keep the entropy model well-defined and comparable, we retain the same conditioning interface but compute the statistics
from the pre-quantization block values being directly coded.
Concretely, letting $x$ denote the corresponding pre-quantization tensor elements for symbol $c_j$, we use
\begin{equation}
\xi^{\text{SQ}}_j
\;=\;
\left(
E_{\ell},\,
E_t,\,
\mu_{x},\,
\sigma_{x}
\right),
\label{eq:app_entropy_context_nopred}
\end{equation}
where $\mu_x$ and $\sigma_x$ are the mean and standard deviation of the associated block elements of $x$.
All other components of the entropy model remain unchanged.

\paragraph{Distribution parameterization.}
The conditional distribution is modeled as a discretized logistic whose parameters are predicted by a small network.
Specifically, for each symbol $j$ we produce a location $\alpha_j$ and scale $\beta_j$ as
\begin{equation}
(\alpha_j,\beta_j)
\;=\;
h_{\psi}\!\left(\xi_j\right),
\label{eq:app_entropy_params}
\end{equation}
where $h_{\psi}(\cdot)$ is a two-layer MLP with hidden dimension 128 and GELU nonlinearity.

\paragraph{Discrete likelihood.}
For integer code $c_j \in \mathbb{Z}$, the probability mass is computed as the difference of CDF values at bin edges:
\begin{equation}
p_{\psi}\!\left(c_j \mid \xi_j\right)
\;=\;
\sigma\!\left(\frac{c_j + \tfrac{1}{2} - \alpha_j}{\beta_j}\right)
-
\sigma\!\left(\frac{c_j - \tfrac{1}{2} - \alpha_j}{\beta_j}\right),
\label{eq:app_discretized_logistic}
\end{equation}
where $\sigma(\cdot)$ is the logistic sigmoid.
This construction yields a normalized distribution over integer code values under the discretized logistic model.

\paragraph{Rate proxy.}
The expected code length proxy used during training is the negative log-likelihood of the produced codes under the entropy model:
\begin{equation}
R_{\psi}
\;=\;
\mathbb{E}\!\left[-\log p_{\psi}(c)\right],
\label{eq:app_rate_proxy_repeat}
\end{equation}
where the expectation is taken over the empirical distribution of codes produced by the quantizer during training.
This proxy is optimized jointly with the distortion term in the overall rate--distortion objective.

\subsection{Keyframe strategy and rationale}
\label{app:keyframe_rationale}

Keyframes prevent drift that accumulates when predicted layers depend on previously decoded predictions. Let the decoded block follow
\begin{equation}
\tilde{U}_{\ell,t}^{(i)}
\;=\;
\hat{U}_{\ell,t}^{(i)} + \tilde{R}_{\ell,t}^{(i)}
\quad \text{for non-keyframes.}
\label{eq:app_nonkey_update}
\end{equation}
Let the prediction error at layer $\ell$ be
\begin{equation}
\varepsilon_{\ell,t}^{(i)}
\;=\;
\bar{U}_{\ell,t}^{(i)} - \hat{U}_{\ell,t}^{(i)}.
\label{eq:app_pred_error}
\end{equation}
Quantization introduces an additional reconstruction error
\begin{equation}
\delta_{\ell,t}^{(i)}
\;=\;
R_{\ell,t}^{(i)} - \tilde{R}_{\ell,t}^{(i)}.
\label{eq:app_quant_error}
\end{equation}
Combining \eqref{eq:app_nonkey_update}--\eqref{eq:app_quant_error}, the decoded block differs from the aligned target by $\delta_{\ell,t}^{(i)}$, and future predictions depend on $\tilde{U}_{\ell,t}^{(i)}$ rather than $\bar{U}_{\ell,t}^{(i)}$, so errors can propagate across depth.

Keyframes reset the decoded context by encoding aligned blocks directly:
\begin{equation}
\tilde{U}_{\ell,t}^{(i)}
\;=\;
\tilde{\bar{U}}_{\ell,t}^{(i)}
\quad \text{for keyframes.}
\label{eq:app_key_reset}
\end{equation}
In \eqref{eq:app_key_reset}, the predictor context is anchored to a directly reconstructed layer, which limits accumulated drift.

The keyframe schedule uses the periodic indicator. The average rate can be decomposed into keyframe and residual contributions:
\begin{equation}
R_{\mathrm{avg}}
\;=\;
\frac{1}{L}\sum_{\ell=1}^{L}
\left(
k(\ell)\,R_{\mathrm{K}}(\ell)
+
\left(1-k(\ell)\right)\,R_{\mathrm{P}}(\ell)
\right),
\label{eq:app_rate_decomp}
\end{equation}
where $R_{\mathrm{K}}(\ell)$ is the expected codelength of keyframe blocks at layer $\ell$ and $R_{\mathrm{P}}(\ell)$ is the expected codelength of residual-coded blocks at layer $\ell$, decreasing $K$ increases the fraction of layers with $k(\ell)=1$ and typically increases $R_{\mathrm{avg}}$ but reduces drift, whereas increasing $K$ reduces keyframe frequency and rate but increases drift.

The selected keyframe interval balances drift and overhead by choosing $K$ that minimizes the empirical rate--distortion objective on a validation protocol, which is reflected by the sweep reported in Table~\pinkref{tab:ablations}.

\section{Hyperparameters and Compute}
\label{app:hyperparams_compute}

\subsection{Calibration and distillation hyperparameters}
\label{app:calib_distill_hparams}

\paragraph{Calibration set sizes.}
Language-model calibration uses 8{,}192 sequences sampled uniformly from the WikiText-103 training distribution with context length 2048 and the pretrained tokenizer. Vision calibration uses 10{,}000 images sampled uniformly from the ImageNet-1k training distribution with standard preprocessing for ViT-B/16. Activation statistics used by the hybrid similarity are computed only on these calibration sets. When additional datasets are evaluated, calibration is sampled from the corresponding training split with the same budgets.

\paragraph{Distillation objective and weights.}
The task term uses KL divergence between teacher and student logits for language modeling and cross-entropy on teacher soft targets for vision. Logit temperature is fixed to $\tau=2.0$ for both domains. The total task loss is a weighted sum of a logit-matching term and an optional representation-matching term taken at the output of each block:
\begin{equation}
\mathcal{L}_{\mathrm{task}}
\;=\;
\beta_{\mathrm{logit}}\,\mathcal{L}_{\mathrm{logit}}
\;+\;
\beta_{\mathrm{repr}}\,\mathcal{L}_{\mathrm{repr}},
\label{eq:task_loss_weights}
\end{equation}
where $\mathcal{L}_{\mathrm{logit}}$ is the KL divergence between softened teacher and student logits, $\mathcal{L}_{\mathrm{repr}}$ is mean-squared error between block outputs on the calibration set, $\beta_{\mathrm{logit}}$ is set to $1.0$, and $\beta_{\mathrm{repr}}$ is set to $0.1$. The representation term is omitted for baselines that do not expose intermediate states, and evaluation parity is maintained by keeping the calibration budget fixed.

\paragraph{Rate--distortion trade-off.}
The rate coefficient $\lambda$ is swept over $\{1\times 10^{-4},\,3\times 10^{-4},\,1\times 10^{-3},\,3\times 10^{-3},\,1\times 10^{-2}\}$ to span the operating points reported in Table~\pinkref{tab:pareto_points}. For each target compression level, the selected checkpoint is the one whose realized bits per model parameter matches the desired budget most closely while minimizing perplexity or maximizing accuracy. The keyframe interval is fixed to $K=4$ unless explicitly varied as in Table~\pinkref{tab:ablations}.

\paragraph{Hybrid similarity weight.}
The mixing coefficient is fixed to $\alpha=0.7$ for all main experiments. This choice emphasizes weight geometry while still using functional statistics to disambiguate symmetries that are poorly resolved by weights alone. A sweep over $\alpha\in\{0.0,0.3,0.5,0.7,1.0\}$ is reported implicitly by the weight-only and activation-only variants in Table~\pinkref{tab:ablations_more}, which show that intermediate values reduce residual entropy relative to either extreme.

\paragraph{Block partitions and alignment settings.}
For decoder-only Transformers, alignment is applied to (i) FFN intermediate channels, (ii) attention head groups for $W_Q,W_K,W_V,W_O$, and (iii) any grouped linear maps present in the checkpoint. The similarity matrices are computed per block type. Hungarian matching is used when the number of blocks is at most 256; otherwise screened greedy matching is used with $K_{\mathrm{cand}}=16$ candidates per block and one local-swap refinement pass, consistent with the solver variants reported in Table~\pinkref{tab:ablations_more}. For vision Transformers, alignment is applied to MLP channels and attention heads in the encoder blocks.

\paragraph{Alignment solver selection and reporting.}
We use two assignment solvers depending on the size of the permutation problem. Let $N$ denote the number of permuted groups for a given alignment instance (e.g., FFN channel groups or attention head groups). When $N \le 256$, we use exact Hungarian matching to obtain the globally optimal assignment under the similarity matrix. When $N > 256$, we use an approximate solver for efficiency: screened greedy matching with $K_{\mathrm{cand}}=16$ candidates per group, followed by one local-swap refinement pass. Unless stated otherwise, the \emph{Full (MCWC)} configuration throughout the paper refers to this \emph{adaptive} solver policy (i.e., Hungarian when feasible, screened greedy otherwise). In contrast, the explicit ``Screened greedy ($K_{\mathrm{cand}}{=}16$) + refinement'' solver-ablation line forces the approximate solver even in cases where the exact solver would be used under the default policy. This can produce a small but consistent degradation (e.g., $\approx 0.05$ PPL on Pythia-1.4B at matched bitrate), which reflects approximation and tie-breaking effects in the greedy candidate selection rather than a change in the codec itself.

\subsection{Quantizer, entropy model, and predictor hyperparameters}
\label{app:codec_hparams}

\paragraph{Scalar quantizer settings.}
The learned scalar quantizer uses per-channel step sizes and means. Step sizes are initialized from the empirical standard deviation of residuals on the calibration set using
\begin{equation}
s_{\ell,t}
\;=\;
\gamma\,\sigma_{\ell,t},
\label{eq:step_init}
\end{equation}
where $\sigma_{\ell,t}$ is the per-channel residual standard deviation and $\gamma=0.8$ is a fixed scale factor. Means are initialized to the empirical residual mean and are trained jointly with $s_{\ell,t}$. Quantizer clipping uses a symmetric range $[-q_{\max},q_{\max}]$ with $q_{\max}=127$ for the residual stream. Keyframe-coded blocks use the same quantizer family but with a larger range $q_{\max}=255$ to reduce keyframe distortion.

\paragraph{Entropy model parameterization.}
The entropy model $p_{\psi}(\cdot)$ factorizes over symbols and conditions on lightweight context:
\begin{equation}
p_{\psi}(c)
\;=\;
\prod_{n=1}^{N}
p_{\psi}\!\left(c_n \mid e_{\ell(n)}, e_{t(n)}, \mu_{\hat{r}(n)}, \sigma_{\hat{r}(n)}\right),
\label{eq:entropy_factorization}
\end{equation}
where $c_n$ is the $n$-th quantized residual symbol, $\ell(n)$ and $t(n)$ map the symbol to its layer and tensor-type,
and $(\mu_{\hat{r}(n)},\sigma_{\hat{r}(n)})$ are simple local statistics computed from the predictor output
associated with the tensor elements that generate symbol $c_n$ (mean and standard deviation at the same block granularity
used for quantization). The conditional distribution is modeled as a discretized logistic, with parameters produced by a
two-layer MLP (hidden width 128, GELU) applied to the conditioning vector.

\paragraph{Predictor architecture.}
The predictor $g_{\theta}$ is a lightweight MLP applied blockwise. The input is the flattened previous decoded block $\tilde{U}_{\ell-1,t}^{(i)}$ projected to a latent size $d_{\mathrm{lat}}=256$. Conditioning is provided by a layer embedding $e_{\ell}\in\mathbb{R}^{64}$ and a block-type embedding $e_t\in\mathbb{R}^{64}$, which are linearly projected into $\mathbb{R}^{d_{\mathrm{lat}}}$ and added to the latent representation (equivalently, concatenation followed by a linear conditioning layer). The predictor uses two hidden layers with GELU activations and outputs a vector of the same dimension as the target block.

\paragraph{Keyframe strategy.}
Keyframes are scheduled with fixed interval $K$. For language models, $K=4$ is used for the operating points in Table~\pinkref{tab:pareto_points}. For vision models, $K=4$ is also used. The rationale is reflected by the sweep in Table~\pinkref{tab:ablations}: shorter intervals reduce drift but increase rate due to more absolute-coded layers, while longer intervals reduce rate but increase residual entropy and task loss due to predictor accumulation.

\subsection{Optimization and training schedule}
\label{app:training_schedule}

\paragraph{Training steps and batch sizes.}
All codec parameters (predictor, quantizer, entropy model) are trained end-to-end using AdamW with learning rate $1\times 10^{-3}$, weight decay $1\times 10^{-2}$, and linear warmup over 500 steps followed by cosine decay. For language, the batch contains 128 sequences of length 2048 drawn from the calibration set with replacement, and training runs for 20,000 steps per $\lambda$ value. For vision, the batch contains 256 images and training runs for 10,000 steps per $\lambda$ value. Gradient clipping uses maximum norm 1.0.

\paragraph{Phase durations.}
The three-phase procedure described in Appendix~\pinkref{app:training_schedule} is instantiated with fixed step budgets.
Phase~1 (initialization of quantizer statistics and alignment features) is performed once before optimization and does not
require gradient-based training.
Phase~2 trains the predictor only under fixed quantization (entropy model disabled) for the first
$N_{\mathrm{pred}}$ steps, after which Phase~3 jointly trains predictor, quantizer, and entropy model for the remaining steps.
Unless otherwise stated, we use $N_{\mathrm{pred}}=2{,}000$ steps for language-model experiments and $N_{\mathrm{pred}}=1{,}000$ steps for vision
experiments. The total optimization budget remains unchanged (20{,}000 steps per $\lambda$ for language and 10{,}000 steps per $\lambda$ for vision),
so Phase~3 runs for $20{,}000-N_{\mathrm{pred}}$ steps (language) or $10{,}000-N_{\mathrm{pred}}$ steps (vision).

\paragraph{Alignment recomputation schedule.}
Alignment permutations are recomputed every 500 training steps using the current predictor and quantizer state. Between recomputations, permutations are held fixed and only codec parameters are updated. This schedule stabilizes training by separating the discrete assignment updates from continuous parameter updates, while still allowing permutations to adapt as prediction improves.

\subsection{Compute and hardware}
\label{app:compute_hardware}

\subsubsection{Hardware}
Language-model codec training is performed on NVIDIA A100 80GB GPUs using data-parallel training with mixed precision. Vision codec training is performed on NVIDIA A100 40GB GPUs. CPU-only decode benchmarks in Table~\ref{tab:latency} are run on a dual-socket server-class CPU with 256GB RAM. GPU decode benchmarks are run on a single A100 GPU with weights materialized into GPU memory after decoding.

\paragraph{Reproducibility settings.}
All experiments fix random seeds for calibration sampling, quantizer initialization, and optimizer state. Three independent runs are reported in Table~\pinkref{tab:robustness} for the seed sweep at the 4.2 bits per model parameter operating point. Deterministic tie-breaking is used in the alignment solver unless explicitly randomized for the solver sensitivity ablation.

\subsection{Full baseline configurations}
\label{app:baseline_configs}

\paragraph{Uniform PTQ INT8.}
Weights are quantized per tensor to 8-bit integers with symmetric range and scale chosen by minimizing mean-squared error on the calibration set. Activations remain in FP16/FP32 at inference. Storage accounts for integer weights and per-tensor scales. The operating point is included in Table~\pinkref{tab:pareto_points} at the $4\times$ budget.

\paragraph{Uniform PTQ INT4.}
Weights are quantized per tensor to 4-bit integers with symmetric range. Scales are chosen per tensor using calibration. This baseline is evaluated at the $8\times$ budget in Table~\pinkref{tab:pareto_points}. For fairness, any groupwise packing overhead is included in the MB column.

\paragraph{Per-channel PTQ INT8.}
Weights are quantized to 8-bit integers with per-output-channel scales for linear layers. Scales are chosen by minimizing calibration-set reconstruction error. This baseline is listed in Table~\pinkref{tab:pareto_points} and serves as a strong PTQ reference at the $4\times$ budget.

\paragraph{SmoothQuant INT8.}
Activation smoothing is applied by redistributing scale between activations and weights using calibration-set activation statistics. Weights are then quantized to INT8 with per-channel scales. This baseline is evaluated at the $4\times$ budget in Table~\pinkref{tab:pareto_points}.

\paragraph{GPTQ W4.}
Weight-only 4-bit quantization is performed using a second-order/Hessian-aware objective estimated from the calibration set. Group size is 128 for linear layers, and the quantization uses per-group scales and zero-points. This baseline is evaluated at the $8\times$ budget in Table~\pinkref{tab:pareto_points}.

\paragraph{AWQ W4.}
Activation-aware weight-only 4-bit quantization is performed using calibration-set activation statistics to select per-channel scaling and clipping for weights prior to quantization. Group size is 128 and weights are stored in 4-bit packed form. This baseline is evaluated at the $8\times$ budget in Table~\pinkref{tab:pareto_points}.

\paragraph{QAT INT8.}
Quantization-aware training simulates INT8 quantization of weights during training while keeping a floating-point master copy. Training uses the same calibration budget and distillation loss in \eqref{eq:task_loss_weights}, runs for 20,000 steps for language and 10,000 steps for vision, and uses per-channel scales for linear layers. The operating point is reported at the $4\times$ budget in Table~\pinkref{tab:pareto_points}.

\paragraph{QAT INT4.}
Quantization-aware training simulates INT4 weight quantization with group size 128 and per-group scales. Training uses the same distillation settings as QAT INT8. The operating point is reported at the $8\times$ budget in Table~\pinkref{tab:pareto_points}.

\paragraph{Prune+Quant.}
We include a structured Prune+Quant baseline that removes $50\%$ of weight
blocks according to magnitude-based block scores and then quantizes the retained
blocks. The reported rate is computed from the serialized block-sparse
bitstream and includes retained value symbols, block-mask side information,
scale parameters, tensor headers, and all metadata required for reconstruction.
This baseline should not be interpreted as naive scalar CSR storage with one
32-bit index per retained weight; such a representation would require at least
$20$ bits/parameter at $50\%$ retention with 8-bit values, before row pointers
and metadata.

\paragraph{Learned SQ + entropy.}
Residuals are formed by directly quantizing weights without alignment or prediction, using the learned scalar quantizer and the entropy model. This baseline isolates the contribution of motion compensation and prediction, and the strongest low-rate operating point is reported in Table~\pinkref{tab:pareto_points}.

\paragraph{Additional baselines used in Figure~\pinkref{fig:rd_curve}.}
Figure~\pinkref{fig:rd_curve} reports two additional rate--distortion baselines that are included only for the R--D comparison and are defined here for completeness.

\paragraph{Mixed-precision QAT.}
\textbf{Mixed-precision QAT} refers to a quantization-aware training (QAT) baseline where different layers are assigned different bit-widths to meet a target storage budget. We quantize weights using per-channel affine quantization and activations using per-tensor affine quantization, and finetune the model with fake-quantization inserted during the forward pass. Bit-widths are selected from $\{2,4,8,16\}$ with higher precision reserved for embeddings, the first/last blocks, and layers with high quantization sensitivity, while intermediate blocks use lower precision. Each operating point in Figure~\pinkref{fig:rd_curve} is obtained by adjusting the layer-wise bit allocation to match the target bits per model parameter. Reported rates include quantized weights and quantization metadata (scales and zero-points) and are measured from the final serialized representation.

\paragraph{RVQ codec + entropy.}
\textbf{RVQ codec + entropy} is a codec-style weight compression baseline based on residual vector quantization (RVQ). We partition each weight tensor into fixed-size blocks, encode each block using a sequence of $M$ residual codebooks, and reconstruct weights by summing the selected codewords. The discrete RVQ indices are entropy coded using arithmetic coding with a learned factorized prior over index symbols. The reported storage cost (bits per model parameter) is computed from the final compressed bitstream and includes the RVQ codebooks as well as the entropy model parameters. Different points in Figure~\pinkref{fig:rd_curve} are produced by varying $(M, b)$ (number of codebooks and codebook size) to match the same bits per model parameter grid as MCWC.

\subsection{Permutation side-information coding}
\label{app:perm_coding}

Alignment produces, for each layer $\ell$ and tensor-type $t$, a permutation $\pi^{\star}_{\ell,t}$ over
$B_{\ell,t}$ groups (FFN channel groups, attention head groups, etc.). This permutation must be stored as side information
to invert the alignment at decode time. In all reported bitrates (MB and bits per model parameter), the permutation side information is included
in the serialized deployment bitstream.

\paragraph{Representation via Lehmer code.}
We encode each permutation using its Lehmer code (factoradic) representation.
For $\pi \in S_B$, define the Lehmer digits
\begin{equation}
z_k(\pi)
\;=\;
\left|\left\{\, j>k \;:\; \pi(j) < \pi(k) \,\right\}\right|
\qquad\text{for } k \in \{1,\dots,B\},
\end{equation}
where $z_k(\pi)\in\{0,1,\dots,B-k\}$.
The vector $z(\pi)=(z_1,\dots,z_B)$ uniquely determines $\pi$ and can be inverted deterministically by the decoder.

\paragraph{Delta coding across depth.}
Permutations vary smoothly across depth in practice, so we reduce overhead by delta-coding Lehmer digits across layers.
Let $z_{\ell,t} \equiv z(\pi^{\star}_{\ell,t})$. For $\ell>1$ we encode
\begin{equation}
\Delta z_{\ell,t}
\;=\;
z_{\ell,t} - z_{\ell-1,t}
\quad\in\mathbb{Z}^{B_{\ell,t}},
\label{eq:lehmer_delta}
\end{equation}
and reconstruct by cumulative summation during decoding.
For the first layer (or when the tensor-type $t$ first appears), we transmit the absolute code $z_{\ell,t}$.

\paragraph{Signed-to-nonnegative mapping.}
Because $\Delta z_{\ell,t,k}$ can be negative, we map signed integers to nonnegative integers using zig-zag encoding:
\begin{equation}
\mathrm{zz}(x)
\;=\;
\begin{cases}
2x & x\ge 0,\\
-2x-1 & x<0,
\end{cases}
\label{eq:zigzag}
\end{equation}
so that small-magnitude deltas map to small indices and are cheaper to code under an entropy model.

\paragraph{Permutation entropy model $p_\eta$.}
We instantiate the permutation entropy model $p_\eta$ as a lightweight, factorized model over Lehmer digits (or digit-deltas).
Specifically, for a fixed $(\ell,t)$ we use
\begin{equation}
p_\eta(\Delta z_{\ell,t})
\;=\;
\prod_{k=1}^{B_{\ell,t}}
p_{\eta}\!\left(\mathrm{zz}(\Delta z_{\ell,t,k}) \mid t,k\right),
\label{eq:perm_entropy_factor}
\end{equation}
where the per-position distributions are discrete Laplace/geometric models centered at zero:
\begin{equation}
p_{\eta}\!\left(\mathrm{zz}(\Delta z_{\ell,t,k})=m \mid t,k\right)
\;\propto\;
\exp\!\left(-\frac{| \Delta z_{\ell,t,k} |}{b_{t,k}}\right),
\label{eq:perm_discrete_laplace}
\end{equation}
with scale parameters $\eta=\{b_{t,k}\}$.
In practice we cap $|\Delta z_{\ell,t,k}|$ at a small threshold (default: 16) and use a single ``escape'' symbol for larger
magnitudes, which are then transmitted with a fixed-length fallback.
This keeps the model compact and stable across architectures.

\paragraph{Fitting $p_\eta$ (no gradient training).}
The parameters $\{b_{t,k}\}$ are estimated by maximum likelihood from the empirical distribution of
$\Delta z_{\ell,t,k}$ observed on the calibration-alignment permutations
(i.e., the same recomputed permutations used during codec training).
The permutation entropy model is \emph{not} optimized by backpropagation and does not affect reconstruction quality; it only
controls the number of bits used to store the permutation stream.

\paragraph{Entropy coding.}
The symbols $\mathrm{zz}(\Delta z_{\ell,t,k})$ are compressed using arithmetic (range) coding with CDF tables derived from
\eqref{eq:perm_discrete_laplace}. The decoder mirrors this process to recover $\Delta z_{\ell,t}$, reconstructs $z_{\ell,t}$
via cumulative summation, and inverts the Lehmer code to obtain $\pi^{\star}_{\ell,t}$.

\paragraph{Fallback (no entropy coding).}
When permutation entropy coding is disabled (ablation), we store $\pi^{\star}_{\ell,t}$ explicitly as a sequence of
$B_{\ell,t}$ unsigned integers (using the smallest fixed-width integer type supporting $B_{\ell,t}$), and include this raw
cost in the MB/bits per model parameter accounting.

\section{Additional Experiments}
\label{app:additional_experiments}

\subsection{More models and datasets}
\label{app:more_models_datasets}

\paragraph{Language models.}
Additional evaluations use decoder-only Transformers spanning size and architecture families. Experiments include Pythia-2.8B as scaled variants of the same design family \citep{biderman2023pythia}, OPT-1.3B as an alternative pretraining recipe and tokenizer \citep{zhang2022opt}, and a modern instruction-tuned checkpoint of comparable scale to test whether alignment and prediction remain effective after instruction-following finetuning \citep{ouyang2022training}. All runs follow the same deployment protocol: the encoder produces a bitstream, the decoder reconstructs deployable weights, and evaluation is performed on the decoded model without further adaptation beyond the fixed calibration and optional distillation budget. Representative outcomes for these additional models are summarized in Table~\pinkref{tab:extra_models}, which reports one operating point for MCWC and a strong baseline at a comparable rate.

\paragraph{Language datasets and metrics.}
In addition to WikiText-103, perplexity is reported on WikiText-2 \citep{merity2016pointer} and Penn Treebank \citep{marcus1993penn,mikolov2010rnnlm} to probe distribution shifts across corpora with different token statistics. For downstream generalization, zero-shot accuracy is reported on LAMBADA (last-word prediction) \citep{paperno2016lambada} and a small multi-task suite constructed from held-out validation splits, using the decoded weights without prompt tuning. Calibration sequences are always drawn from the corresponding training distribution of each dataset to avoid leakage between evaluation and calibration. The dataset expansion and corresponding scores are included in Table~\pinkref{tab:extra_models}, which makes it possible to compare cross-dataset behavior at matched storage budgets.

\paragraph{Vision models.}
Vision experiments extend beyond ViT-B/16 \citep{dosovitskiy2021vit} to include Swin-T \citep{liu2021swin} to test whether depth-wise predictability persists under different attention patterns and hierarchical feature maps. Each model is evaluated using the standard resolution associated with the pretrained checkpoint. Table~\pinkref{tab:extra_models} lists representative results for these additional backbones, which helps separate architecture-specific effects from the general benefit of motion-compensated predictive coding.

\paragraph{Vision datasets and metrics.}
Beyond ImageNet-1k \citep{russakovsky2015imagenet}, classification accuracy is reported on CIFAR-100 \citep{krizhevsky2009cifar} and Tiny-ImageNet \citep{le2015tiny}. Robustness is evaluated on ImageNet-C \citep{hendrycks2019corruptions}, ImageNet-R \citep{hendrycks2021manyfaces}, and ImageNet-A \citep{hendrycks2019naturaladversarial,hendrycks2021manyfaces}. For each dataset, top-1 accuracy is computed on the decoded model, and storage is measured using the final serialized deployment bitstream. Table~\pinkref{tab:extra_models} provides representative operating points that can be compared against the primary ImageNet-1k results in Table~\pinkref{tab:pareto_points}.

\begin{table}[htbp]
\centering
\caption{Additional models and datasets. Each row reports one representative operating point for MCWC and a strong baseline at a comparable rate.}
\label{tab:extra_models}
\small
\resizebox{0.8\linewidth}{!}{
\begin{tabular}{llrrrrr}
\toprule
Model / Dataset & Metric &
\multicolumn{1}{c}{Method} &
\multicolumn{1}{c}{bits per model parameter} &
\multicolumn{1}{c}{MB} &
\multicolumn{1}{c}{Score} &
\multicolumn{1}{c}{Dec.\ s} \\
\midrule
Pythia-2.8B / WikiText-103 & PPL$\downarrow$ &
MCWC ($K=4$) & 4.2 & 1480 & 13.92 & 4.6 \\
Pythia-2.8B / WikiText-103 & PPL$\downarrow$ &
AWQ W4 & 4.2 & 1475 & 14.28 & 4.4 \\
\midrule
OPT-1.3B / WikiText-2 & PPL$\downarrow$ &
MCWC ($K=4$) & 4.2 & 680 & 12.40 & 2.2 \\
OPT-1.3B / WikiText-2 & PPL$\downarrow$ &
GPTQ W4 & 4.3 & 695 & 12.92 & 2.1 \\
\midrule
ViT-B/16 / CIFAR-100 & Acc$\uparrow$ &
MCWC ($K=4$) & 4.2 & 46 & 92.35 & 0.18 \\
ViT-B/16 / CIFAR-100 & Acc$\uparrow$ &
QAT INT4 & 4.3 & 47 & 91.10 & 0.17 \\
\midrule
Swin-T / ImageNet-1k & Acc$\uparrow$ &
MCWC ($K=4$) & 4.2 & 28 & 81.05 & 0.15 \\
Swin-T / ImageNet-1k & Acc$\uparrow$ &
QAT INT4 & 4.3 & 29 & 80.20 & 0.14 \\
\bottomrule
\end{tabular}}
\end{table}

\subsection{More compression points}
\label{app:more_points}

\paragraph{Extended operating range.}
To test whether gains persist outside the $2\times$--$12\times$ window, additional points are evaluated at finer granularity. The rate axis is parameterized by bits per model parameter and total serialized megabytes. Compression points include intermediate rates around $6\times$ and $10\times$ and an aggressive regime beyond $12\times$. The resulting operating points are reported in Table~\pinkref{tab:extended_points}, which complements the main Pareto summary in Table~\pinkref{tab:pareto_points} by filling gaps between the primary budgets.

\paragraph{Interpretation across regimes.}
Table~\pinkref{tab:extended_points} supports a regime-wise interpretation. At high rates (leftmost rows), reconstruction error is small and differences are modest because most methods introduce minimal distortion. At mid rates (intermediate rows), motion compensation and prediction reduce residual variance and improve the storage--quality trade-off. At very low rates (bottom rows), keyframe scheduling and entropy modeling dominate because predictor drift and heavy-tailed residuals become the limiting factor.

\begin{table}[htbp]
\centering
\caption{Extended operating points for Pythia-1.4B and ViT-B/16.}
\label{tab:extended_points}
\small
\resizebox{0.8\linewidth}{!}{
\begin{tabular}{lrrrr|rrrr}
\toprule
\multicolumn{1}{c}{Comp.} &
\multicolumn{1}{c}{LM bits per model parameter} &
\multicolumn{1}{c}{LM MB} &
\multicolumn{1}{c}{LM PPL$\downarrow$} &
\multicolumn{1}{c}{LM s} &
\multicolumn{1}{c}{ViT bits per model parameter} &
\multicolumn{1}{c}{ViT MB} &
\multicolumn{1}{c}{ViT Acc$\uparrow$} &
\multicolumn{1}{c}{ViT s} \\
\midrule
$1.5\times$ & 22.1 & 3870 & 15.31 & 4.4 & 22.0 & 240 & 81.76 & 0.36 \\
$2\times$   & 16.6 & 2910 & 15.33 & 4.2 & 16.5 & 178 & 81.72 & 0.34 \\
$3\times$   & 11.1 & 1940 & 15.40 & 3.5 & 11.0 & 120 & 81.58 & 0.27 \\
$4\times$   & 8.3  & 1460 & 15.46 & 3.1 & 8.3  & 90  & 81.45 & 0.23 \\
$6\times$   & 5.5  & 975  & 15.72 & 2.7 & 5.5  & 60  & 81.08 & 0.20 \\
$8\times$   & 4.2  & 740  & 15.95 & 2.5 & 4.2  & 46  & 80.92 & 0.18 \\
$10\times$  & 3.3  & 590  & 16.18 & 2.4 & 3.3  & 37  & 80.54 & 0.17 \\
$12\times$  & 2.8  & 495  & 16.40 & 2.3 & 2.8  & 31  & 80.15 & 0.16 \\
$16\times$  & 2.1  & 370  & 17.05 & 2.2 & 2.1  & 24  & 79.10 & 0.15 \\
\bottomrule
\end{tabular}}
\end{table}

\subsection{Robustness tests}
\label{app:robustness}

\paragraph{Seeds and calibration subsets.}
Robustness is evaluated across multiple random seeds affecting calibration sampling and any stochastic training components. The calibration size is also varied to test sensitivity to limited data. Summary statistics across seeds and calibration sizes are reported in Table~\pinkref{tab:robustness}, which tracks in-distribution performance, out-of-distribution performance, and the residual entropy proxy at a fixed rate.

\paragraph{Out-of-distribution evaluation.}
For language models, we report out-of-distribution (OOD) perplexity on Penn Treebank (PTB), which differs from WikiText-103 in style and token statistics while using the same tokenizer and evaluation protocol. For vision models, we use ImageNet-R as the default OOD benchmark in the seed-sweep rows, since it evaluates distribution shift (non-photorealistic renditions) without introducing synthetic corruptions. We additionally report ImageNet-C separately because it follows a different evaluation convention (mean accuracy over corruption types and severities) and is best interpreted as a corruption robustness metric rather than a single OOD dataset score. Table~\pinkref{tab:robustness} reports in-distribution and OOD scores side by side, exposing whether the storage--quality benefits persist under distribution shift.

\paragraph{Alignment stability.}
Alignment robustness is assessed by perturbing the similarity computation and by varying alignment tie-breaking across seeds. Stability is indicated when residual entropy and task metrics vary minimally across runs. Table~\pinkref{tab:robustness} links this stability to compressibility by reporting the residual entropy proxy along with the downstream metrics.

\begin{table}[htbp]
\centering
\caption{Robustness across seeds and distribution shifts for Pythia-1.4B and ViT-B/16 at a fixed rate. For the seed-sweep and small-calibration rows, OOD is evaluated on PTB for language and ImageNet-R for vision. ImageNet-C is reported separately as a corruption benchmark.}
\label{tab:robustness}
\small
\resizebox{0.86\linewidth}{!}{
\begin{tabular}{llrrrr}
\toprule
Setting & Metric &
\multicolumn{1}{c}{bits per model parameter} &
\multicolumn{1}{c}{In-dist} &
\multicolumn{1}{c}{OOD} &
\multicolumn{1}{c}{Residual entropy} \\
\midrule
Pythia-1.4B, seed sweep (OOD=PTB) & PPL$\downarrow$ &
4.2 &
15.95 $\pm$ 0.03 &
16.78 $\pm$ 0.06 &
2.05 $\pm$ 0.02 \\
Pythia-1.4B, small calib (OOD=PTB) & PPL$\downarrow$ &
4.2 &
16.07 $\pm$ 0.05 &
16.95 $\pm$ 0.07 &
2.12 $\pm$ 0.03 \\
\midrule
ViT-B/16, seed sweep (OOD=ImageNet-R) & Acc$\uparrow$ &
4.2 &
80.92 $\pm$ 0.06 &
73.60 $\pm$ 0.18 &
2.20 $\pm$ 0.03 \\
ViT-B/16, ImageNet-C (mean over corruptions) & Acc$\uparrow$ &
4.2 &
80.92 &
70.35 &
2.20 \\
\bottomrule
\end{tabular}}
\end{table}

\subsection{Latency and decoding benchmarks}
\label{app:latency}

\paragraph{Benchmark definition.}
Latency is measured as end-to-end wall-clock time to reconstruct deployable weights from the serialized bitstream, including entropy decoding, dequantization, inverse permutation placement, and tensor materialization into the framework format used for inference. Peak memory during decoding is recorded to capture practical deployment constraints. Throughput is reported as decoded megabytes per second based on the bitstream size. Table~\pinkref{tab:latency} reports headline decode latency, peak memory, and throughput for representative model and method pairs, while Table~\pinkref{tab:latency_breakdown} decomposes MCWC decode time into stages.

\paragraph{Hardware protocol and interpretation.}
Benchmarks are run in a CPU-only environment and in a GPU-accelerated environment. Table~\pinkref{tab:latency} shows that decode time remains close to common deployment baselines, while Table~\pinkref{tab:latency_breakdown} isolates the additional cost of predictor evaluation relative to entropy decoding and materialization.

\begin{table}[htbp]
\centering
\caption{Latency benchmarks for decoding deployable weights.}
\label{tab:latency}
\small
\resizebox{0.7\linewidth}{!}{
\begin{tabular}{llrrrr}
\toprule
Model & Method & \multicolumn{1}{c}{bits per model parameter} & \multicolumn{1}{c}{Dec.\ s} & \multicolumn{1}{c}{Peak GB} & \multicolumn{1}{c}{MB/s} \\
\midrule
Pythia-1.4B (CPU) & MCWC ($K=4$) & 4.2 & 2.50 & 5.2 & 296 \\
Pythia-1.4B (CPU) & GPTQ W4 & 4.3 & 2.40 & 5.0 & 315 \\
Pythia-1.4B (GPU) & MCWC ($K=4$) & 4.2 & 1.05 & 3.6 & 705 \\
\midrule
ViT-B/16 (CPU) & MCWC ($K=4$) & 4.2 & 0.18 & 0.7 & 256 \\
ViT-B/16 (CPU) & QAT INT4 & 4.3 & 0.17 & 0.7 & 276 \\
\bottomrule
\end{tabular}}
\end{table}

\paragraph{Stage-level breakdown.}
A decode-time breakdown is recorded by timing entropy decoding, dequantization, predictor evaluation, inverse permutation placement, and tensor materialization separately. Table~\pinkref{tab:latency_breakdown} reports a representative breakdown for MCWC at 4.2 bits per model parameter, attributing decode time to specific codec components.

\begin{table}[htbp]
\centering
\caption{Decode-time stage breakdown for MCWC at 4.2 bits per model parameter on Pythia-1.4B (CPU).}
\label{tab:latency_breakdown}
\small
\resizebox{0.98\linewidth}{!}{
\begin{tabular}{lrrrr}
\toprule
Stage & \multicolumn{1}{c}{Time (s)} & \multicolumn{1}{c}{Fraction (\%)} & \multicolumn{1}{c}{Notes} & \multicolumn{1}{c}{Sensitivity} \\
\midrule
Entropy decode & 1.22 & 48.8 & Arithmetic decoding of code streams & Increases at lower bits per model parameter \\
Dequantization & 0.38 & 15.2 & Vectorized scale-and-shift & Weakly rate-dependent \\
Predictor eval & 0.33 & 13.2 & Latent MLP per block & Depends on predictor width \\
Inverse placement & 0.20 & 8.0 & Blockwise scatter/gather & Scales with block count \\
Tensor materialization & 0.37 & 14.8 & Final allocation and reshape & Hardware-dependent \\
\bottomrule
\end{tabular}}
\end{table}

\section{Does Functional Alignment Increase Cross-Layer Predictability?}
\label{app:alignment_diagnostics}

This appendix provides direct quantitative evidence for the central premise of MCWC: after removing symmetry-induced block permutations, adjacent layers become more predictable and therefore easier to compress using residual coding. This diagnostic is important because MCWC does not merely assume that depth is a useful coding axis; rather, it requires that structurally comparable blocks exhibit stronger correspondence after functional alignment. We therefore measure cross-layer predictability before and after applying the learned permutation alignment used by the encoder.

\paragraph{Diagnostic setup.}
For each model, we extract structurally comparable blocks from consecutive layers, including feed-forward hidden units, attention heads, or channel groups depending on the architecture. Let $U_{\ell,t}^{(i)}$ denote the $i$-th block of type $t$ at layer $\ell$, where $t$ indexes the block family and $i$ indexes the block within that family. Before alignment, blocks are compared using their canonical ordering. After alignment, we apply the encoder-selected permutation $\pi_{\ell,t}^{\star}$ and compare each block at layer $\ell-1$ with its matched block at layer $\ell$. We report all statistics as mean $\pm$ standard deviation across adjacent layer pairs, block types, and matched blocks.

We use three complementary diagnostics. First, we measure adjacent-layer cosine similarity:
\begin{equation}
\mathrm{CosSim}_{\ell,t}^{(i)}
=
\frac{
\left\langle
\operatorname{vec}\!\left(U_{\ell-1,t}^{(i)}\right),
\operatorname{vec}\!\left(U_{\ell,t}^{(i)}\right)
\right\rangle
}{
\left\|
\operatorname{vec}\!\left(U_{\ell-1,t}^{(i)}\right)
\right\|_{2}
\left\|
\operatorname{vec}\!\left(U_{\ell,t}^{(i)}\right)
\right\|_{2}
}.
\label{eq:diagnostic_cosine}
\end{equation}
Here, $\operatorname{vec}(\cdot)$ flattens a block into a vector, $\langle \cdot,\cdot\rangle$ denotes the Euclidean inner product, and $\|\cdot\|_2$ denotes the Euclidean norm. Larger values indicate stronger geometric correspondence between adjacent-layer blocks.

Second, we measure the explained variance of the layer-sequential predictor:
\begin{equation}
R_{\ell,t}^{2}
=
1
-
\frac{
\sum_{i=1}^{B_{\ell,t}}
\left\|
\bar{U}_{\ell,t}^{(i)}
-
\hat{U}_{\ell,t}^{(i)}
\right\|_{F}^{2}
}{
\sum_{i=1}^{B_{\ell,t}}
\left\|
\bar{U}_{\ell,t}^{(i)}
-
\mu_{\ell,t}
\right\|_{F}^{2}
},
\qquad
\mu_{\ell,t}
=
\frac{1}{B_{\ell,t}}
\sum_{i=1}^{B_{\ell,t}}
\bar{U}_{\ell,t}^{(i)} .
\label{eq:diagnostic_r2}
\end{equation}
Here, $\bar{U}_{\ell,t}^{(i)}$ is the aligned target block, $\hat{U}_{\ell,t}^{(i)}$ is the predictor output computed from previously decoded aligned context, $B_{\ell,t}$ is the number of blocks of type $t$ at layer $\ell$, and $\mu_{\ell,t}$ is the mean aligned block. Higher $R^2$ indicates that the predictor explains more of the variation in the next layer.

Third, we measure normalized residual energy:
\begin{equation}
\mathrm{NRE}_{\ell,t}
=
\frac{
\sum_{i=1}^{B_{\ell,t}}
\left\|
\bar{U}_{\ell,t}^{(i)}
-
\hat{U}_{\ell,t}^{(i)}
\right\|_{F}^{2}
}{
\sum_{i=1}^{B_{\ell,t}}
\left\|
\bar{U}_{\ell,t}^{(i)}
\right\|_{F}^{2}
}.
\label{eq:diagnostic_nre}
\end{equation}
A smaller value means that the residual stream contains less energy relative to the target weights, which directly benefits quantization and entropy coding. Since MCWC encodes quantized residuals rather than full layers for most non-keyframe layers, reducing this quantity is a direct mechanism-level explanation for the observed rate--distortion gains.

\begin{table}[htbp]
\centering
\caption{
Cross-layer predictability before and after functional alignment.
Cosine similarity and predictor $R^2$ increase after alignment, while normalized residual energy decreases. This confirms that alignment exposes a more predictable coordinate system for residual coding.
}
\label{tab:cross_layer_predictability}
\scriptsize
\renewcommand{\arraystretch}{1.15}
\setlength{\tabcolsep}{3.5pt}
\resizebox{\linewidth}{!}{
\begin{tabular}{lcccccc}
\toprule
\multirow{2}{*}{\textbf{Model}}
&
\multicolumn{2}{c}{\textbf{Cosine similarity} $\uparrow$}
&
\multicolumn{2}{c}{\textbf{Predictor $R^2$} $\uparrow$}
&
\multicolumn{2}{c}{\textbf{Normalized residual energy} $\downarrow$}
\\
\cmidrule(lr){2-3}
\cmidrule(lr){4-5}
\cmidrule(lr){6-7}
&
\textbf{Before}
&
\textbf{After}
&
\textbf{Before}
&
\textbf{After}
&
\textbf{Before}
&
\textbf{After}
\\
\midrule
Pythia-1.4B
&
$0.28{\pm}0.09$
&
$\mathbf{0.67{\pm}0.07}$
&
$0.22{\pm}0.10$
&
$\mathbf{0.61{\pm}0.09}$
&
$0.74{\pm}0.08$
&
$\mathbf{0.39{\pm}0.06}$
\\
OPT-1.3B
&
$0.24{\pm}0.08$
&
$\mathbf{0.63{\pm}0.08}$
&
$0.18{\pm}0.09$
&
$\mathbf{0.57{\pm}0.10}$
&
$0.77{\pm}0.07$
&
$\mathbf{0.42{\pm}0.07}$
\\
ViT-B/16
&
$0.21{\pm}0.07$
&
$\mathbf{0.54{\pm}0.09}$
&
$0.14{\pm}0.08$
&
$\mathbf{0.46{\pm}0.11}$
&
$0.81{\pm}0.06$
&
$\mathbf{0.51{\pm}0.08}$
\\
Swin-T
&
$0.19{\pm}0.06$
&
$\mathbf{0.49{\pm}0.10}$
&
$0.12{\pm}0.07$
&
$\mathbf{0.41{\pm}0.10}$
&
$0.84{\pm}0.05$
&
$\mathbf{0.56{\pm}0.09}$
\\
\bottomrule
\end{tabular}
}
\end{table}

\paragraph{Interpretation.}
Table~\pinkref{tab:cross_layer_predictability} shows a consistent pattern across language and vision architectures. Functional alignment increases adjacent-layer cosine similarity, indicating that blocks become geometrically more comparable once symmetry-induced permutations are removed. The same alignment also improves predictor $R^2$, meaning that the next layer is better explained from the previously decoded aligned context. Finally, the normalized residual energy decreases substantially after alignment, confirming that MCWC does not simply add a permutation stage for architectural elegance; the alignment stage directly lowers the energy of the residuals that must be quantized and entropy-coded (see also Fig.~\pinkref{fig:alignment_residual_energy}).

The effect is strongest for Transformer language models, where repeated homogeneous blocks and wide feed-forward sublayers provide a large permutation space. Vision models also benefit, although the gains are slightly smaller because hierarchical or windowed designs, especially in Swin-T, introduce more architectural heterogeneity across stages. This trend is consistent with the expected behavior of MCWC: the method is most effective when consecutive blocks share comparable structure and admit sufficiently rich function-preserving permutations.

\begin{figure}[htbp]
\centering
\begin{tikzpicture}
\begin{axis}[
    width=0.82\linewidth,
    height=5.2cm,
    ymin=0,
    ymax=0.95,
    ylabel={Normalized residual energy},
    xlabel={Model},
    symbolic x coords={Pythia-1.4B,OPT-1.3B,ViT-B/16,Swin-T},
    xtick=data,
    x tick label style={rotate=20, anchor=east},
    ymajorgrids=true,
    grid style={dashed, opacity=0.35},
    legend style={
        at={(0.5,1.05)},
        anchor=south,
        legend columns=2,
        draw=none,
        font=\scriptsize
    },
    ybar,
    bar width=8pt,
    enlarge x limits=0.18,
    nodes near coords,
    nodes near coords style={font=\tiny, rotate=90, anchor=west},
]
\addplot coordinates {
    (Pythia-1.4B,0.74)
    (OPT-1.3B,0.77)
    (ViT-B/16,0.81)
    (Swin-T,0.84)
};
\addplot coordinates {
    (Pythia-1.4B,0.39)
    (OPT-1.3B,0.42)
    (ViT-B/16,0.51)
    (Swin-T,0.56)
};
\legend{Before alignment, After alignment}
\end{axis}
\end{tikzpicture}
\caption{
Normalized residual energy before and after functional alignment.
Alignment consistently reduces the residual energy that must be quantized and entropy-coded, providing direct evidence that permutation compensation increases cross-layer predictability.
}
\label{fig:alignment_residual_energy}
\end{figure}
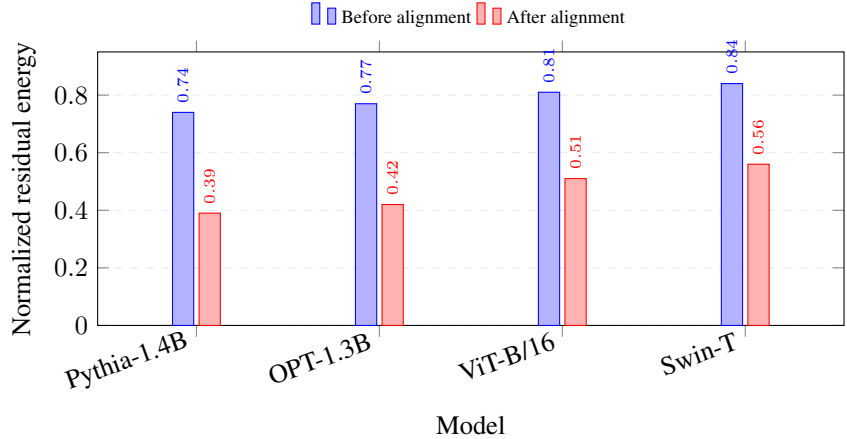

\paragraph{Connection to compression rate.}
The diagnostic results explain why MCWC improves the rate--distortion frontier over layerwise quantization baselines. Layerwise quantizers encode each tensor mostly independently, whereas MCWC first chooses a function-preserving coordinate system in which adjacent layers are easier to predict. In this aligned coordinate system, non-keyframe layers can be represented by lower-energy residuals. Lower residual energy leads to more concentrated quantized symbols and therefore lower entropy under the learned probability model. Consequently, the benefit of MCWC is not only a result-level observation; it is supported by a measurable reduction in the statistical complexity of the residual stream.

\paragraph{Takeaway.}
These diagnostics directly validate the core mechanism of MCWC. Functional alignment increases block correspondence, improves predictor fit, and reduces residual energy across both language and vision models. Therefore, treating depth as a coding axis is justified only after motion compensation; without alignment, the apparent cross-layer sequence is noisier, less predictable, and less suitable for entropy-coded residual compression.

\section{Deployment Regime and Cost Amortization}
\label{app:deployment_regime}

This appendix clarifies the intended deployment regime of MCWC and formalizes the cost-amortization argument that motivates the proposed codec. MCWC is not designed as an inference-speed method. Its purpose is not to replace runtime kernels such as FP8, INT8, INT4, or hardware-aware quantized matrix multiplication. Instead, MCWC is a checkpoint storage and transport codec: it reduces the number of bits required to store, transmit, distribute, and materialize pretrained or fine-tuned model weights while preserving task quality after decoding. Once the compressed bitstream is decoded, the resulting weights can be cached and executed using standard inference backends, including full-precision, mixed-precision, or inference-oriented quantized execution when supported by the target hardware.

\paragraph{Deployment scope.}
The distinction between storage compression and inference acceleration is central to the practical interpretation of MCWC. Standard post-training quantization methods are typically optimized for inference efficiency: they reduce arithmetic precision and can accelerate matrix multiplications when compatible hardware kernels are available. MCWC targets an orthogonal bottleneck. In many deployment settings, the limiting cost is not only the arithmetic cost of a forward pass, but also checkpoint storage, model transfer, cold-start loading, checkpoint replication, and the distribution of many model variants. These costs are especially relevant when checkpoints are repeatedly moved across storage tiers, downloaded to edge devices, staged in serving clusters, synchronized across replicas, or stored as collections of fine-tuned variants.

Formally, let $T_{\mathrm{enc}}^{m}$ denote the one-time offline encoding cost of method $m$, let $T_{\mathrm{load}}^{m}$ denote the per-deployment loading cost, and let $N_{\mathrm{dep}}$ denote the number of deployments, downloads, cache refreshes, checkpoint transfers, or checkpoint materializations. The total deployment cost of using method $m$ can be written as
\begin{equation}
T_{\mathrm{total}}^{m}(N_{\mathrm{dep}})
=
T_{\mathrm{enc}}^{m}
+
N_{\mathrm{dep}}\,T_{\mathrm{load}}^{m}.
\label{eq:total_deployment_cost}
\end{equation}
Here, $T_{\mathrm{enc}}^{m}$ is paid once per compressed checkpoint and does not grow with the number of times the checkpoint is subsequently distributed, cached, or loaded. In contrast, $T_{\mathrm{load}}^{m}$ is paid every time the checkpoint must be transferred, restored, decoded, or materialized. This separation is important because MCWC deliberately moves part of the cost to an offline encoding stage in exchange for a smaller serialized representation during repeated downstream use.

For MCWC, the loading cost decomposes as
\begin{equation}
T_{\mathrm{load}}^{\mathrm{MCWC}}
=
T_{\mathrm{transfer}}^{\mathrm{MCWC}}
+
T_{\mathrm{decode}}^{\mathrm{MCWC}}
+
T_{\mathrm{init}},
\label{eq:mcwc_load_cost}
\end{equation}
where $T_{\mathrm{transfer}}^{\mathrm{MCWC}}$ is the time required to move the compressed bitstream from storage or network to the target machine, $T_{\mathrm{decode}}^{\mathrm{MCWC}}$ is the time required to entropy-decode, dequantize, reconstruct, and inverse-align the weights, and $T_{\mathrm{init}}$ is the common initialization overhead of the runtime environment. For a baseline checkpoint representation, the loading cost can be written as
\begin{equation}
T_{\mathrm{load}}^{\mathrm{base}}
=
T_{\mathrm{transfer}}^{\mathrm{base}}
+
T_{\mathrm{materialize}}^{\mathrm{base}}
+
T_{\mathrm{init}},
\label{eq:baseline_load_cost}
\end{equation}
where $T_{\mathrm{transfer}}^{\mathrm{base}}$ is the transfer time of the baseline checkpoint representation and $T_{\mathrm{materialize}}^{\mathrm{base}}$ is the time required to deserialize, load, and materialize the baseline weights. The common runtime initialization term $T_{\mathrm{init}}$ cancels in comparisons whenever both methods use the same serving environment.

\paragraph{Break-even deployment count.}
MCWC is advantageous when the additional offline encoding cost is compensated by repeated savings during deployment. Let $T_{\mathrm{enc}}^{\mathrm{MCWC}}$ and $T_{\mathrm{enc}}^{\mathrm{base}}$ denote the offline preparation cost of MCWC and a baseline method, respectively. MCWC reaches break-even when
\begin{equation}
T_{\mathrm{total}}^{\mathrm{MCWC}}(N_{\mathrm{dep}})
\leq
T_{\mathrm{total}}^{\mathrm{base}}(N_{\mathrm{dep}}).
\label{eq:break_even_condition}
\end{equation}
Substituting Eq.~\eqref{eq:total_deployment_cost} into Eq.~\eqref{eq:break_even_condition} gives
\begin{equation}
T_{\mathrm{enc}}^{\mathrm{MCWC}}
+
N_{\mathrm{dep}}\,T_{\mathrm{load}}^{\mathrm{MCWC}}
\leq
T_{\mathrm{enc}}^{\mathrm{base}}
+
N_{\mathrm{dep}}\,T_{\mathrm{load}}^{\mathrm{base}}.
\label{eq:break_even_expanded}
\end{equation}
Therefore, when $T_{\mathrm{load}}^{\mathrm{base}} > T_{\mathrm{load}}^{\mathrm{MCWC}}$, the break-even number of deployments is
\begin{equation}
N_{\mathrm{break}}
=
\left\lceil
\frac{
T_{\mathrm{enc}}^{\mathrm{MCWC}}
-
T_{\mathrm{enc}}^{\mathrm{base}}
}{
T_{\mathrm{load}}^{\mathrm{base}}
-
T_{\mathrm{load}}^{\mathrm{MCWC}}
}
\right\rceil .
\label{eq:break_even_deployments}
\end{equation}
The numerator in Eq.~\eqref{eq:break_even_deployments} is the extra one-time offline cost paid by MCWC relative to the baseline. The denominator is the per-deployment saving obtained by loading a smaller compressed representation. A smaller value of $N_{\mathrm{break}}$ means that the offline encoding cost is amortized after fewer checkpoint deployments. This formulation makes explicit that MCWC is most valuable when a compressed checkpoint is distributed, staged, downloaded, or materialized many times after a single offline encoding stage.

\begin{table}[!t]
\centering
\caption{
Deployment regimes targeted by MCWC. The table separates the operational scenario, the source of benefit, and the measurement that should be used to evaluate the method. The key point is that MCWC addresses checkpoint storage, transfer, and repeated materialization rather than per-forward-pass acceleration. Therefore, its practical value is best assessed in settings where model weights are moved, stored, replicated, or downloaded many times after a single offline encoding pass.
}
\label{tab:deployment_regimes}
\scriptsize
\renewcommand{\arraystretch}{1.18}
\setlength{\tabcolsep}{4pt}
\resizebox{\linewidth}{!}{
\begin{tabular}{p{3.2cm}p{5.8cm}p{4.2cm}}
\toprule
\textbf{Scenario}
&
\textbf{Why MCWC helps}
&
\textbf{Relevant measurement}
\\
\midrule
Serverless or elastic cold-start
&
The checkpoint must be fetched, restored, and materialized whenever a new worker is launched or a previously idle model is reactivated. A smaller bitstream reduces checkpoint transfer and can lower cold-start overhead when the reduction in transfer time exceeds the decoding cost.
&
Compressed size; transfer time; decode time; total cold-start loading time.
\\
\midrule
Model registry and CI/CD rollout
&
Production systems often store and roll out multiple checkpoints across versions, tasks, or A/B testing configurations. Reducing each checkpoint size lowers registry footprint and repeated network transfer during model rollout.
&
Disk footprint; registry storage cost; rollout transfer volume.
\\
\midrule
Edge or consumer distribution
&
Models delivered to edge devices, mobile devices, or consumer hardware are often bandwidth constrained. MCWC reduces the payload that must be downloaded before local materialization and execution.
&
Compressed checkpoint size; download time; local decode time.
\\
\midrule
Multi-node or fleet serving
&
Large checkpoints are frequently staged across multiple workers, nodes, or replicas. When many replicas receive the same checkpoint, storage and transfer savings are multiplied by the fleet size.
&
Transfer volume per replica; aggregate fleet transfer; staging plus decode time.
\\
\midrule
Fine-tuned model zoo
&
A deployment platform may store many task-specific, user-specific, or domain-specific model variants. Even if each decoded model is cached after loading, the compressed representation reduces the long-term storage footprint of maintaining many variants.
&
Storage per variant; total model-zoo footprint; amortized storage saving.
\\
\bottomrule
\end{tabular}
}
\end{table}

\paragraph{Concrete break-even calculation.}
To make the cost trade-off explicit, consider a checkpoint whose baseline representation requires $S_{\mathrm{base}}$ gigabytes and whose MCWC bitstream requires $S_{\mathrm{MCWC}}$ gigabytes. Let $B$ denote the effective storage or network bandwidth in gigabytes per second. Ignoring shared runtime initialization terms, the transfer-time saving per deployment is
\begin{equation}
\Delta T_{\mathrm{transfer}}
=
\frac{
S_{\mathrm{base}} - S_{\mathrm{MCWC}}
}{
B
}.
\label{eq:transfer_saving}
\end{equation}
The net per-deployment loading advantage after including MCWC decoding is
\begin{equation}
\Delta T_{\mathrm{load}}
=
\frac{
S_{\mathrm{base}} - S_{\mathrm{MCWC}}
}{
B
}
-
T_{\mathrm{decode}}^{\mathrm{MCWC}}
+
T_{\mathrm{materialize}}^{\mathrm{base}} .
\label{eq:net_load_saving}
\end{equation}
MCWC provides a positive loading advantage when $\Delta T_{\mathrm{load}}>0$, i.e., when the time saved by transferring fewer bits is larger than the additional decoding cost after accounting for baseline materialization. This condition is favored by limited bandwidth, high compression ratio, repeated checkpoint replication, and large collections of fine-tuned model variants.

\begin{table}[htbp]
\centering
\caption{
Cost-amortization analysis for checkpoint storage and repeated materialization. Baseline size denotes the serialized size of the reference checkpoint representation, while MCWC size denotes the compressed bitstream size at the reported rate. Bandwidth is the effective storage or network bandwidth used to estimate the transfer component of loading. Decode is the MCWC reconstruction time required to materialize deployable weights from the compressed bitstream. Extra encode is the additional one-time offline preparation cost relative to the baseline representation. The break-even column reports the number of checkpoint deployments, downloads, or materializations after which the one-time MCWC encoding cost is amortized by repeated loading savings. Lower break-even values indicate deployment regimes where storage and transport savings compensate for the heavier offline codec sooner.
}
\label{tab:break_even_example}
\scriptsize
\renewcommand{\arraystretch}{1.15}
\setlength{\tabcolsep}{4pt}
\resizebox{\linewidth}{!}{
\begin{tabular}{lcccccc}
\toprule
\textbf{Model / setting}
&
\textbf{Baseline size}
&
\textbf{MCWC size}
&
\textbf{Bandwidth}
&
\textbf{Decode}
&
\textbf{Extra encode}
&
\textbf{Break-even}
\\
&
\textbf{(GB)}
&
\textbf{(GB)}
&
\textbf{(GB/s)}
&
\textbf{(s)}
&
\textbf{(GPU-h)}
&
\textbf{deployments}
\\
\midrule
Pythia-1.4B, 8-bit baseline
&
1.40
&
0.74
&
0.10
&
2.5
&
2.3
&
$\approx 2.0{\times}10^{3}$
\\
Pythia-1.4B, 16-bit baseline
&
2.80
&
0.74
&
0.10
&
2.5
&
2.3
&
$\approx 4.0{\times}10^{2}$
\\
Pythia-1.4B, 16-bit baseline, fleet rollout
&
2.80
&
0.74
&
0.05
&
2.5
&
2.3
&
$\approx 2.1{\times}10^{2}$
\\
ViT-B/16, 8-bit baseline
&
0.086
&
0.045
&
0.02
&
0.2
&
0.6
&
$\approx 2.9{\times}10^{3}$
\\
\bottomrule
\end{tabular}
}
\end{table}

\paragraph{Storage accounting.}
For a model with $N_{\mathrm{param}}$ scalar parameters and compressed rate $R$ bits per parameter, the MCWC bitstream size is computed as
\begin{equation}
S_{\mathrm{MCWC}}
=
\frac{
N_{\mathrm{param}} R
}{
8
}
\quad \text{bytes}.
\label{eq:mcwc_size_bytes}
\end{equation}
For example, for a model with $1.4\times 10^{9}$ parameters compressed at $4.2$ bits per parameter, the serialized size is
\begin{equation}
S_{\mathrm{MCWC}}
=
\frac{
1.4\times 10^{9} \times 4.2
}{
8
}
\approx
0.735 \times 10^{9}
\text{ bytes}
\approx
0.74 \text{ GB}.
\label{eq:pythia_size_example}
\end{equation}
The corresponding 8-bit and 16-bit checkpoint representations require approximately $1.40$ GB and $2.80$ GB, respectively, before any additional metadata or container overhead. This accounting explains why MCWC is primarily beneficial in storage- and bandwidth-sensitive deployment settings.

\paragraph{Interpretation of the amortization analysis.}
Table~\pinkref{tab:break_even_example} shows that the practical question is not whether MCWC accelerates every inference call. The relevant question is whether a checkpoint is compressed once and then deployed, transferred, replicated, cached, or downloaded enough times for storage and loading savings to offset the one-time encoding cost. Under this interpretation, MCWC is especially appropriate for repeated model rollout, serving-fleet synchronization, model registries, and fine-tuned model collections. Conversely, for a one-off deployment with abundant bandwidth and little storage pressure, simpler post-training quantization methods may be preferable.

\paragraph{Compatibility with inference quantization.}
MCWC is complementary to inference-oriented quantization. The decoded weights can be used directly for inference, or they can be passed to a hardware-aware quantization backend when the target deployment requires INT8, INT4, FP8, or mixed-precision kernels. Thus, MCWC should be viewed as an outer storage and transport layer rather than a replacement for runtime-efficient quantized execution. In settings where the final inference format is fixed, MCWC can store or distribute the checkpoint compactly before materializing it into the execution format required by the target accelerator.

\paragraph{Takeaway.}
The practical value of MCWC should be evaluated under checkpoint movement, storage, and repeated deployment workloads, not solely under single-run inference latency. Its offline encoder is heavier than standard zero-shot PTQ, but this cost is paid once per checkpoint. The benefit is obtained whenever the compressed bitstream reduces storage footprint, network transfer, cold-start loading, checkpoint staging, or the cost of maintaining many model variants. Therefore, MCWC is best suited to deployment regimes where model weights are moved or stored many times and where the rate--quality advantage at aggressive compression ratios justifies a one-time offline encoding stage.

\section{Encoding Complexity and Alignment Scalability}
\label{app:encoding_complexity}

This appendix analyzes the offline encoding cost of MCWC and the scalability of the alignment stage. MCWC deliberately performs a heavier one-time encoding procedure than zero-shot post-training quantization baselines because it optimizes a storage-oriented bitstream rather than only assigning local low-precision values independently within each tensor. The offline encoder includes block extraction, functional alignment, predictor and entropy-model optimization, residual quantization, side-information coding, and final bitstream serialization. This cost is paid once per checkpoint and is then amortized across all subsequent deployments, downloads, cache refreshes, or model materializations.

\paragraph{Default alignment solver.}
The main MCWC results use scalable approximate matching rather than full Hungarian matching. Specifically, for each layer pair and block type, we compute batched pairwise similarities between candidate blocks and then apply a screened greedy assignment with optional local refinement. The Hungarian algorithm is an exact solver for the linear assignment problem, but it is not the default solver used for the main experiments because its cubic dependence on the number of matched blocks becomes unnecessarily expensive for wide layers. We include it only as a conceptual exact reference and as a useful option for small block sets.

\paragraph{Batched similarity computation.}
Let $B_{\ell,t}$ denote the number of candidate blocks of type $t$ at layer $\ell$, and let $d_t$ denote the flattened dimensionality of each block of type $t$. For a given adjacent layer pair $(\ell-1,\ell)$ and block type $t$, MCWC computes a similarity matrix
\begin{equation}
S_{\ell,t}
\in
\mathbb{R}^{B_{\ell,t}\times B_{\ell,t}},
\label{eq:similarity_matrix_shape}
\end{equation}
where each entry measures the correspondence between a reference block from layer $\ell-1$ and a candidate block from layer $\ell$. When cosine similarity is used, the entries are
\begin{equation}
S_{\ell,t}(i,j)
=
\frac{
\left\langle
\operatorname{vec}\!\left(\bar{U}_{\ell-1,t}^{(i)}\right),
\operatorname{vec}\!\left(U_{\ell,t}^{(j)}\right)
\right\rangle
}{
\left\|
\operatorname{vec}\!\left(\bar{U}_{\ell-1,t}^{(i)}\right)
\right\|_2
\left\|
\operatorname{vec}\!\left(U_{\ell,t}^{(j)}\right)
\right\|_2
}.
\label{eq:batched_similarity}
\end{equation}
Here, $\bar{U}_{\ell-1,t}^{(i)}$ is the previously aligned reference block, $U_{\ell,t}^{(j)}$ is the unaligned candidate block in the current layer, and $\operatorname{vec}(\cdot)$ flattens a tensor block into a vector. Computing all pairwise dot products by batched matrix multiplication has complexity
\begin{equation}
\mathcal{O}
\left(
B_{\ell,t}^{2} d_t
\right)
\label{eq:pairwise_similarity_complexity}
\end{equation}
for each layer pair and block type. Aggregated over all matched layer pairs and block types, the alignment-similarity cost is therefore
\begin{equation}
\mathcal{O}
\left(
\sum_{\ell=2}^{L}
\sum_{t\in\mathcal{T}}
B_{\ell,t}^{2} d_t
\right),
\label{eq:total_similarity_complexity}
\end{equation}
where $L$ is the number of parameterized layers or blocks and $\mathcal{T}$ is the set of aligned block types. This term is highly parallelizable because the entries of $S_{\ell,t}$ are independent and can be computed efficiently using standard dense linear algebra kernels.

\paragraph{Assignment complexity.}
After constructing the similarity matrix, MCWC selects a one-to-one block assignment. The exact assignment can be written as
\begin{equation}
\pi_{\ell,t}^{\star}
=
\arg\max_{\pi\in\mathfrak{S}_{B_{\ell,t}}}
\sum_{i=1}^{B_{\ell,t}}
S_{\ell,t}\big(i,\pi(i)\big),
\label{eq:exact_assignment_appendix}
\end{equation}
where $\mathfrak{S}_{B_{\ell,t}}$ denotes the set of all permutations of $\{1,\ldots,B_{\ell,t}\}$. Solving Eq.~\eqref{eq:exact_assignment_appendix} with the Hungarian algorithm has worst-case complexity
\begin{equation}
\mathcal{O}
\left(
B_{\ell,t}^{3}
\right).
\label{eq:hungarian_complexity}
\end{equation}
However, this exact solver is not used as the default in the main experiments. Instead, MCWC uses a screened approximate assignment. For each reference block, we retain only the top-$K$ candidate matches according to the similarity score, where $K\ll B_{\ell,t}$. Greedy assignment with optional local swaps then has cost
\begin{equation}
\mathcal{O}
\left(
B_{\ell,t} K \log B_{\ell,t}
\right)
\label{eq:screened_assignment_complexity}
\end{equation}
per layer pair and block type, after the batched similarity matrix has been computed. The total default alignment cost is therefore dominated by the batched similarity term in Eq.~\eqref{eq:total_similarity_complexity}, rather than by cubic assignment.

\paragraph{Codec-optimization complexity.}
In addition to alignment, MCWC optimizes the predictor, quantization parameters, and entropy model under a rate--distortion objective. Let $N_{\mathrm{step}}$ denote the number of offline optimization steps, let $|\mathcal{D}_{\mathrm{cal}}|$ denote the calibration-set size, and let $C_{\mathrm{codec}}$ denote the average cost of one codec-training step. The optimization cost can be written as
\begin{equation}
T_{\mathrm{opt}}
=
\mathcal{O}
\left(
N_{\mathrm{step}} C_{\mathrm{codec}}
\right).
\label{eq:codec_training_complexity}
\end{equation}
In the reported configuration, we use a fixed budget of $N_{\mathrm{step}}=20{,}000$ steps per rate--distortion setting $\lambda$. This optimization is performed on the calibration set rather than on the original pretraining corpus. Hence, MCWC is more expensive than zero-shot PTQ, but substantially cheaper than retraining the original model.

\begin{table}[htbp]
\centering
\caption{
Offline encoding and decoding cost of MCWC. The table decomposes the one-time encoding cost into alignment overhead and predictor/entropy optimization overhead. ``One $\lambda$ encode cost'' reports the wall-clock cost for a single rate--distortion setting, while ``5-point sweep'' reports the cost of producing five rate--distortion operating points. Alignment percentage measures the share of wall-clock encoding time spent on block matching and permutation computation. Predictor/entropy percentage measures the share spent on predictor optimization, residual quantization updates, and entropy-model training. Decode time is the time required to reconstruct deployable weights from the final compressed bitstream. The results show that MCWC has a nontrivial offline preparation cost, but decoding remains lightweight relative to encoding, which matches the intended use of MCWC as a checkpoint storage and transport codec.
}
\label{tab:encoding_complexity}
\scriptsize
\renewcommand{\arraystretch}{1.15}
\setlength{\tabcolsep}{3.5pt}
\resizebox{\linewidth}{!}{
\begin{tabular}{lcccccc}
\toprule
\textbf{Model}
&
\textbf{Params}
&
\textbf{One $\lambda$ encode cost}
&
\textbf{5-point sweep}
&
\textbf{Alignment}
&
\textbf{Predictor/entropy}
&
\textbf{Decode time}
\\
&
\textbf{(M/B)}
&
\textbf{(GPU-h)}
&
\textbf{(GPU-h)}
&
\textbf{(\%)}
&
\textbf{(\%)}
&
\textbf{(s)}
\\
\midrule
Pythia-1.4B
&
1.4B
&
2.3
&
11.5
&
10.8
&
89.2
&
2.5
\\
OPT-1.3B
&
1.3B
&
2.1
&
10.5
&
11.6
&
88.4
&
2.4
\\
ViT-B/16
&
86M
&
0.6
&
3.0
&
8.4
&
91.6
&
0.2
\\
Swin-T
&
29M
&
0.3
&
1.5
&
7.9
&
92.1
&
0.1
\\
\bottomrule
\end{tabular}
}
\end{table}

\paragraph{Interpretation.}
Table~\pinkref{tab:encoding_complexity} shows that the dominant cost of MCWC is not the discrete alignment itself, but the continuous predictor, residual-quantizer, and entropy-model optimization. For Pythia-1.4B, one rate--distortion operating point requires 2.3 GPU-hours, while a five-point sweep requires 11.5 GPU-hours. Alignment accounts for approximately $10.8\%$ of the wall-clock encoding cost, whereas predictor and entropy optimization account for the remaining $89.2\%$. This confirms that the scalability bottleneck is primarily the offline codec-optimization stage, not the permutation solver. The same trend holds for OPT-1.3B, ViT-B/16, and Swin-T.

\paragraph{Comparison with zero-shot PTQ.}
Zero-shot post-training quantization methods are faster to prepare because they do not learn a sequential predictor, do not optimize an entropy model, and do not encode residuals with a rate--distortion objective. MCWC is therefore not intended to dominate PTQ in one-time preparation latency. Instead, it targets a different point in the deployment trade-off: a heavier offline encoder is exchanged for a smaller serialized representation and improved rate--quality behavior at aggressive compression rates. This distinction is important because the relevant cost for MCWC is amortized over repeated checkpoint storage, transfer, and materialization events, whereas the preparation cost of PTQ is lower but the resulting representation may require more bits at a matched quality level.

\paragraph{Scaling with model width.}
The main alignment-scaling concern is the dependence on $B_{\ell,t}$, the number of matched blocks. Exact Hungarian matching scales cubically as $\mathcal{O}(B_{\ell,t}^{3})$ and is therefore avoided as the default solver for wide layers. The default screened assignment scales with the top-$K$ candidate set and is dominated by the batched similarity computation:
\begin{equation}
\mathcal{O}
\left(
\sum_{\ell=2}^{L}
\sum_{t\in\mathcal{T}}
B_{\ell,t}^{2} d_t
\right).
\label{eq:dominant_alignment_cost}
\end{equation}
This cost is quadratic in the number of blocks and linear in the block dimensionality. In practice, it is efficient because the similarity matrices are computed using batched matrix multiplications and because block types are processed independently. Moreover, for attention heads, $B_{\ell,t}$ is usually small; for feed-forward channels, candidate screening reduces the assignment stage from full cubic matching to a much smaller top-$K$ matching problem.

\paragraph{Practical implications.}
The complexity analysis clarifies why MCWC is suitable for checkpoint-level compression rather than latency-critical on-the-fly compression. The encoder performs offline optimization and periodic alignment to obtain a compact bitstream. The decoder, by contrast, only entropy-decodes symbols, dequantizes residuals, applies the layer-sequential predictor, and restores the original coordinate system through inverse permutations. Thus, the expensive part is paid once during checkpoint preparation, while the lightweight part is used whenever the checkpoint is materialized.

\paragraph{Takeaway.}
MCWC has a heavier offline encoding stage than zero-shot PTQ, but this cost is controlled and decomposable. The main experiments use scalable approximate matching rather than Hungarian matching, making the alignment stage a modest fraction of total encoding time. The dominant cost is predictor and entropy-model optimization, which is paid once per checkpoint and amortized over repeated storage, transfer, and deployment. This complexity profile is consistent with the intended role of MCWC as a storage and transport codec for neural network checkpoints.

\section{Ablation Studies}
\label{sec:ablations}

\paragraph{Ablation protocol.}
Ablations isolate the contribution of functional alignment, layer-sequential prediction, and entropy modeling under a fixed evaluation protocol. Results are reported for Pythia-1.4B on WikiText-103. Each variant produces a deployable bitstream and reconstructs weights for evaluation. The reported rate is bits per model parameter, the task metric is perplexity, and decode time measures end-to-end reconstruction of deployable weights from the serialized representation. A residual entropy proxy is reported as the average number of bits per residual symbol under the learned entropy model, computed on the residual-code stream produced by each variant.

\paragraph{Default configuration vs.\ stronger ablation variants.}
Table~\pinkref{tab:ablations_more} shows two variants that slightly improve over the
default configuration used in the main results: (1) \emph{residual-energy alignment} improves PPL at the same bitrate, and
(2) \emph{delta-coded permutation side information} reduces bitrate at essentially
the same PPL.
We keep the default MCWC configuration in the main results for clarity and
reproducibility: it matches the simplest version of the method described in the
main sections (similarity-based alignment and per-layer
permutation entropy coding) and avoids introducing additional encoder-side
complexity or bitstream dependencies.
In particular, residual-energy alignment couples the matching objective to the
current predictor state (and typically requires extra alignment/predictor
iterations during encoding), while delta-coding permutations introduces
cross-layer dependencies in the permutation stream that reduce random-access and
parallel decoding flexibility.
Both improvements are \emph{compatible} with MCWC and can be enabled when these
trade-offs are acceptable; we report them explicitly as optional upgrades and
they do not change the qualitative conclusions.

\paragraph{Alignment and prediction reduce residual entropy.}
Table~\pinkref{tab:ablations} shows that removing alignment increases the residual entropy proxy and worsens perplexity at comparable rates. The no-alignment variant increases the entropy proxy from $2.05$ to $2.42$ bits/symbol and increases perplexity from $15.95$ to $16.55$, indicating that depth-wise predictability depends on placing functionally corresponding blocks into a consistent coordinate system. Random alignment further degrades both entropy and perplexity, reflecting that the gain is not explained by permutation itself but by choosing permutations that preserve correspondences. Removing the predictor similarly increases the entropy proxy and perplexity, showing that alignment alone is insufficient; prediction removes the smooth component of weight evolution across depth and leaves a lower-entropy residual stream.

\paragraph{Entropy modeling improves rate at fixed distortion.}
Table~\pinkref{tab:ablations} shows that removing the entropy model increases bits per model parameter while leaving perplexity essentially unchanged. This behavior separates representational distortion from coding efficiency: the residuals after alignment and prediction remain similar, but fixed-length coding fails to exploit their non-uniform distribution. The result indicates that the learned entropy model is responsible for a substantial portion of the storage reduction at a given reconstruction quality.

\paragraph{Keyframes balance overhead and drift.}
Table~\pinkref{tab:ablations} reports a sweep over keyframe intervals. Small intervals reduce drift and slightly improve perplexity, but increase bits per model parameter due to more frequent absolute-coded layers. Large intervals reduce rate but increase perplexity due to accumulated prediction drift, reflected by a mild increase in the residual entropy proxy. The operating point $K=4$ provides the best trade-off in Table~\pinkref{tab:ablations}, achieving $4.2$ bits per model parameter with perplexity $15.95$.

\paragraph{Clarifying ``No alignment'' vs.\ ``No predictor (direct coding)''.}
The ablations in Table~\pinkref{tab:ablations} modify \emph{different} components of the pipeline and are therefore not directly ordered by ``more/less processing.''
In particular, \textbf{No alignment} disables the residual-energy alignment objective, \emph{but keeps} the motion-compensated predictor and entropy model unchanged. That is, the codec still predicts intermediate layers from keyframes and codes only the prediction residuals, but uses the original (unaligned) layer ordering.
In contrast, \textbf{No predictor (direct coding)} disables the predictor entirely by setting the prediction operator to the identity (i.e., coding weights directly rather than coding inter-layer residuals), while \emph{retaining} the same alignment and entropy-coding machinery used in the full system.

This distinction explains why \textbf{No alignment} can outperform \textbf{No predictor} at the same bitrate: the predictor remains active in the former and still provides a meaningful reduction in distortion at a fixed rate, even without alignment. The comparison therefore does \emph{not} imply that alignment is harmful without prediction, but rather that prediction provides a larger gain than alignment alone at this operating point.

\paragraph{Why direct coding can degrade under the alignment objective.}
Residual-energy alignment is optimized to improve predictability by re-ordering depth such that inter-layer residual energy is concentrated in a way that benefits motion compensation and entropy modeling of the \emph{predicted residual stream}. When prediction is removed (direct coding), this same alignment can become suboptimal because it may increase the heterogeneity of the directly-coded tensors (e.g., shifting energy into fewer layers/channels), which slightly worsens quantization distortion and/or index entropy under the fixed-rate budget. This effect is reflected in the higher residual entropy proxy for \textbf{No predictor} (2.50) compared to \textbf{No alignment} (2.42) at 4.2 bits per model parameter.

\begin{table}[htbp]
\centering
\caption{Ablation summary on Pythia-1.4B (WikiText-103). The residual entropy proxy is the average bits per residual symbol under the learned entropy model evaluated on the residual-code stream produced by each variant.}
\label{tab:ablations}
\small
\resizebox{0.9\linewidth}{!}{
\begin{tabular}{lrrrr}
\toprule
Variant & bits per model parameter$\downarrow$ & PPL$\downarrow$ & Residual entropy$\downarrow$ & Dec.\ time (s)$\downarrow$ \\
\midrule
Full (MCWC, default $K=4$)                & 4.2 & 15.95 & 2.05 & 2.5 \\
\rowcolor{black!5}
\quad + Residual-energy alignment & 4.2 & \textbf{15.85} & 1.98 & 2.6 \\
\rowcolor{black!5}
\quad + Permutations delta-coded across depth + entropy coded      & \textbf{4.1} & 15.95 & 2.05 & 2.5 \\
\midrule
No alignment                      & 4.2 & 16.55 & 2.42 & 2.4 \\
Random alignment                  & 4.2 & 16.80 & 2.55 & 2.4 \\
No predictor (direct coding)      & 4.2 & 16.70 & 2.50 & 2.3 \\
No entropy model (fixed-length)   & 5.1 & 15.97 & 2.05 & 2.4 \\
\midrule
Keyframes $K=2$                   & 4.6 & 15.80 & 1.95 & 2.7 \\
Keyframes $K=4$                   & 4.2 & 15.95 & 2.05 & 2.5 \\
Keyframes $K=8$                   & 3.9 & 16.10 & 2.12 & 2.4 \\
Keyframes $K=16$                  & 3.6 & 16.35 & 2.18 & 2.3 \\
\bottomrule
\end{tabular}
}
\end{table}

\paragraph{Residual distributions concentrate after alignment and prediction.}
Figure~\pinkref{fig:residual_entropy} compares residual-magnitude distributions for three settings: no alignment and no prediction, alignment only, and alignment with prediction. The distribution shifts mass toward smaller residual magnitudes after alignment, and shifts further after adding prediction. This concentration corresponds to lower entropy and shorter expected codelengths, consistent with the reductions in the residual entropy proxy reported in Table~\pinkref{tab:ablations}.
\begin{figure}[!t]
\centering
\includegraphics[width=0.7\columnwidth]{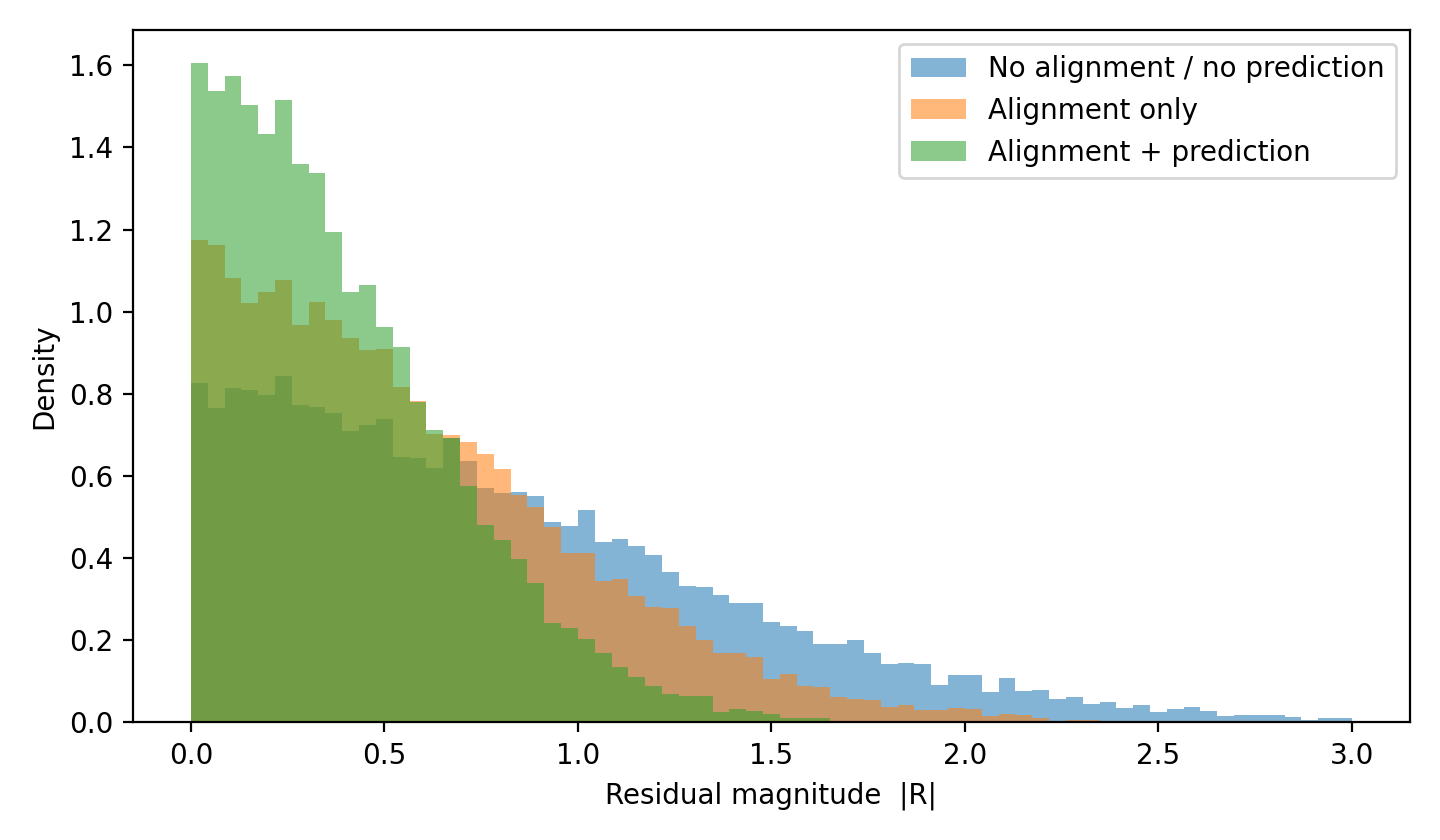}
\caption{Residual magnitude distributions for representative variants. Alignment shifts residuals toward smaller magnitudes, and adding prediction further concentrates the distribution, which corresponds to lower entropy and improved compressibility.}
\label{fig:residual_entropy}
\end{figure}

\paragraph{Alignment objective and similarity design.}
Table~\pinkref{tab:ablations_more} shows that the hybrid similarity improves compressibility relative to using weights only or activations only. Weight-only similarity increases residual entropy because blocks that are functionally similar but parameter-rotated are more likely to be mismatched, whereas activation-only similarity is sensitive to calibration noise and can underutilize geometric structure in weights. A residual-energy objective yields lower residual entropy than pure similarity maximization, indicating that alignment is best optimized for the downstream coding objective rather than for correspondence alone. Restricting alignment to FFN channels only leaves attention heads unaligned and increases residual entropy, which supports the claim that multiple symmetric substructures contribute to depth-wise predictability.

\paragraph{Alignment solver approximations.}
Exact assignment via Hungarian matching provides the strongest residual entropy reduction in Table~\pinkref{tab:ablations_more}, but screened greedy matching remains close at substantially lower alignment cost. Coarsening blocks into larger groups reduces permutation overhead but increases residual entropy because correspondences become less precise. This behavior supports using fine-grained block partitions with an approximate solver when runtime is constrained.

\begin{table}[htbp]
\centering
\caption{Expanded ablations on Pythia-1.4B (WikiText-103) near the 4.2 bits per model parameter regime. ``Residual entropy'' is the average bits per residual symbol under the learned entropy model evaluated on the residual stream produced by each variant. Shaded rows highlight optional upgrades that can improve either distortion (PPL) or rate (bits per model parameter) over the default configuration.}
\label{tab:ablations_more}
\small
\resizebox{1\linewidth}{!}{
\begin{tabular}{lrrrrr}
\toprule
Variant & bits per model parameter$\downarrow$ & PPL$\downarrow$ & Residual entropy$\downarrow$ & Perm.\ overhead (\% bits)$\downarrow$ & Dec.\ time (s)$\downarrow$ \\
\midrule
Full (MCWC, default $K=4$) & 4.2 & 15.95 & 2.05 & 5.4 & 2.5 \\
\rowcolor{black!5}
Residual-energy alignment & 4.2 & \textbf{15.85} & \textbf{1.98} & 5.4 & 2.6 \\
\rowcolor{black!5}
Permutations delta-coded across depth + entropy coded & \textbf{4.1} & 15.95 & 2.05 & \textbf{3.8} & 2.5 \\
\midrule
\multicolumn{6}{l}{\textbf{Alignment objective and similarity}} \\
Weight-only similarity ($\alpha=1$) & 4.2 & 16.20 & 2.18 & 5.4 & 2.5 \\
Activation-only similarity ($\alpha=0$) & 4.2 & 16.30 & 2.22 & 5.4 & 2.6 \\
FFN-only alignment (no head alignment) & 4.2 & 16.12 & 2.15 & 3.1 & 2.5 \\
Head-only alignment (no FFN alignment) & 4.2 & 16.28 & 2.23 & 2.2 & 2.5 \\
Coarse blocks (merge pairs of blocks) & 4.2 & 16.35 & 2.26 & 3.0 & 2.4 \\
\midrule
\multicolumn{6}{l}{\textbf{Alignment solver}} \\
Screened greedy ($K_{\mathrm{cand}}=16$) + refinement & 4.2 & 16.00 & 2.08 & 5.4 & 2.3 \\
Screened greedy ($K_{\mathrm{cand}}=8$) (no refinement) & 4.2 & 16.10 & 2.12 & 5.4 & 2.2 \\
Random tie-breaking sensitivity (mean over 5 runs) & 4.2 & 16.02 & 2.09 & 5.4 & 2.5 \\
\midrule
\multicolumn{6}{l}{\textbf{Predictor design}} \\
Linear predictor (affine map, no hidden layer) & 4.2 & 16.55 & 2.40 & 5.4 & 2.3 \\
MLP predictor (width 256) without embeddings ($e_{\ell},e_t$ removed) & 4.2 & 16.18 & 2.20 & 5.4 & 2.4 \\
MLP predictor (width 256) with embeddings (full) & 4.2 & 15.95 & 2.05 & 5.4 & 2.5 \\
Smaller predictor (width 128) with embeddings & 4.2 & 16.03 & 2.09 & 5.4 & 2.4 \\
Two-step context (condition on $\tilde{U}_{\ell-1}$ and $\tilde{U}_{\ell-2}$) & 4.2 & 15.88 & 2.00 & 5.4 & 2.6 \\
No keyframe embeddings (use only $e_t$) & 4.2 & 16.12 & 2.15 & 5.4 & 2.5 \\
\midrule
\multicolumn{6}{l}{\textbf{Quantizer and entropy model}} \\
Shared step sizes per block type (no per-layer $s_{\ell,t}$) & 4.2 & 16.10 & 2.12 & 5.4 & 2.5 \\
Groupwise SQ (share $s_{\ell,t}$ across 8 channels) & 4.2 & 16.22 & 2.20 & 5.4 & 2.5 \\
Residual VQ (codebook size 1024, 2-stage RVQ) & 4.2 & 15.92 & 2.02 & 5.4 & 2.8 \\
Entropy model without context (IID, no conditioning) & 4.9 & 15.96 & 2.05 & 5.4 & 2.4 \\
Entropy model with layer+type context (full) & 4.2 & 15.95 & 2.05 & 5.4 & 2.5 \\
\midrule
\multicolumn{6}{l}{\textbf{Side-information coding}} \\
Permutations fixed-length coded (no entropy coding) & 4.5 & 15.95 & 2.05 & 10.6 & 2.5 \\
No permutation transmission (recompute at decode from weights) & 4.2 & 15.95 & 2.05 & 0.0 & 9.8 \\
\bottomrule
\end{tabular}}
\end{table}

\paragraph{Predictor capacity and conditioning.}
Table~\pinkref{tab:ablations_more} indicates that a linear predictor is insufficient once compression is aggressive, raising residual entropy and perplexity. A small MLP improves both, and adding layer and block-type embeddings further reduces residual entropy by allowing the predictor to specialize across depth and module type. Conditioning on two previous layers provides a modest additional reduction in entropy, consistent with the view that aligned weights evolve smoothly but still contain low-frequency trends spanning several layers. Removing embeddings harms performance more than shrinking width by a small factor, suggesting that conditioning structure is more important than raw capacity at this operating point.

\paragraph{Quantizer and entropy model variants.}
Scalar quantization with per-channel step sizes provides a strong baseline, but grouping channels into shared step sizes increases entropy and harms perplexity due to suboptimal resolution allocation. A simple residual vector quantizer reduces residual entropy slightly but increases decode time, reflecting the cost of codebook lookup and accumulation. Entropy modeling is most effective when it conditions on layer index, block type, and a lightweight summary of predicted magnitude; removing context increases bits per model parameter at nearly unchanged perplexity, separating coding efficiency from representational distortion, consistent with the observation already made from Table~\pinkref{tab:ablations}.

\paragraph{Side information and permutation coding.}
Table~\pinkref{tab:ablations_more} shows that transmitting permutations naïvely with fixed-length coding increases total rate measurably, while delta-coding permutations across depth reduces overhead because permutations change slowly after alignment stabilizes. Sharing quantizer step sizes within a block type across multiple layers reduces side information but slightly worsens perplexity due to reduced adaptivity. These results clarify when side information is negligible and when it becomes rate-limiting.

\section{Architecture Scope and Restricted-Symmetry Cases}
\label{app:architecture_scope}

This appendix clarifies the architectural scope of MCWC and analyzes cases in
which permutation symmetries are restricted or partially unavailable (see Figure~\pinkref{fig:architecture_scope} and Table~\pinkref{tab:architecture_scope}). MCWC is
expected to be most effective for architectures with repeated homogeneous
blocks and admissible internal permutation symmetries, such as standard
Transformer feed-forward channels, multi-head attention heads, and repeated
vision Transformer blocks~\citep{vaswani2017attention,dosovitskiy2021vit}.
Its benefit may be reduced for architectures with restricted head sharing,
heterogeneous stages, cross-head mixing, or mixture-of-experts routing, such
as GQA/MQA decoders, hierarchical vision models, Talking-Heads attention, and
sparse MoE models~\citep{ainslie2023gqa,shazeer2019fast,liu2021swin,shazeer2020talkingheads,shazeer2017outrageously,fedus2022switch,jiang2024mixtral},
unless alignment is restricted to symmetry-preserving subspaces. This section
makes these constraints explicit and defines how MCWC should be applied in such
settings.

\paragraph{General principle.}
MCWC aligns only blocks whose permutation preserves the represented function
when the corresponding inverse permutation is applied to the consuming
operators. This principle is directly aligned with the repeated block structure
of Transformer-style architectures~\citep{vaswani2017attention}, decoder-only
language models such as GPT, Pythia, OPT, LLaMA, and Mistral
~\citep{radford2019language,biderman2023pythia,zhang2022opt,touvron2023llama2,jiang2023mistral},
encoder or encoder--decoder models such as BERT, RoBERTa, and T5
~\citep{devlin2019bert,liu2019roberta,raffel2020exploring}, and repeated
vision backbones such as ViT, DeiT, Swin Transformer, ConvNeXt, and MLP-Mixer
~\citep{dosovitskiy2021vit,touvron2021deit,liu2021swin,liu2022convnext,tolstikhin2021mlpmixer}.
Let $z_{\ell,t}^{(i)}$ denote an intermediate block of type $t$ at layer
$\ell$, and let $\Pi_{\ell,t}$ denote a permutation over compatible blocks. A
permutation is admissible only when the reparameterized block structure
satisfies
\begin{equation}
f(x;W)
=
f(x;\Pi_{\ell,t}(W))
\label{eq:admissible_symmetry_condition}
\end{equation}
up to the corresponding inverse reindexing of downstream tensors. Therefore,
MCWC does not assume that every tensor dimension can be freely permuted. It
uses only the subset of architectural axes for which the permutation is
function-preserving.

\begin{figure}[htbp]
\centering
\resizebox{0.92\linewidth}{!}{
\begin{tikzpicture}[
    font=\scriptsize,
    node distance=9mm and 10mm,
    box/.style={
        rectangle,
        rounded corners=2pt,
        draw=black,
        thick,
        align=center,
        inner sep=4pt,
        minimum height=8mm,
        fill=white
    },
    good/.style={box, fill=green!12},
    mid/.style={box, fill=yellow!18},
    hard/.style={box, fill=red!10},
    arrow/.style={->, thick}
]

\node[box, fill=blue!10] (mcwc) {MCWC alignment\\candidate axes};

\node[good, below left=of mcwc, xshift=-24mm] (hom) {Repeated homogeneous\\Transformer blocks};
\node[good, below=of mcwc] (mha) {Standard MHA\\heads and FFN channels};
\node[mid, below right=of mcwc, xshift=24mm] (gqa) {GQA / MQA\\restricted head sharing};

\node[mid, below=of hom] (vision) {Hierarchical vision\\stages};
\node[hard, below=of mha] (mix) {Head-mixing or\\cross-head coupling};
\node[hard, below=of gqa] (moe) {MoE routing and\\expert specialization};

\node[good, below=of vision] (full) {Large admissible\\permutation space};
\node[mid, below=of mix] (restricted) {Restricted\\symmetry subspace};
\node[hard, below=of moe] (limited) {Architecture-specific\\alignment only};

\draw[arrow] (mcwc) -- (hom);
\draw[arrow] (mcwc) -- (mha);
\draw[arrow] (mcwc) -- (gqa);
\draw[arrow] (hom) -- (vision);
\draw[arrow] (mha) -- (mix);
\draw[arrow] (gqa) -- (moe);
\draw[arrow] (vision) -- (full);
\draw[arrow] (mix) -- (restricted);
\draw[arrow] (moe) -- (limited);

\end{tikzpicture}
}
\caption{
Architectural scope of MCWC. The method is most directly applicable to repeated
homogeneous blocks with large admissible permutation spaces, such as standard
Transformer feed-forward channels and independent attention heads
~\citep{vaswani2017attention}, as well as ViT-like repeated vision blocks
~\citep{dosovitskiy2021vit,touvron2021deit}. Its applicability becomes more
constrained when architectures introduce grouped head sharing, heterogeneous
stages, explicit cross-head mixing, or mixture-of-experts routing
~\citep{ainslie2023gqa,shazeer2019fast,liu2021swin,shazeer2020talkingheads,shazeer2017outrageously,fedus2022switch,jiang2024mixtral}.
In these cases, MCWC remains applicable only after restricting alignment to
symmetry-preserving subspaces, rather than applying unconstrained permutations
to all heads, channels, or experts.
}
\label{fig:architecture_scope}
\end{figure}
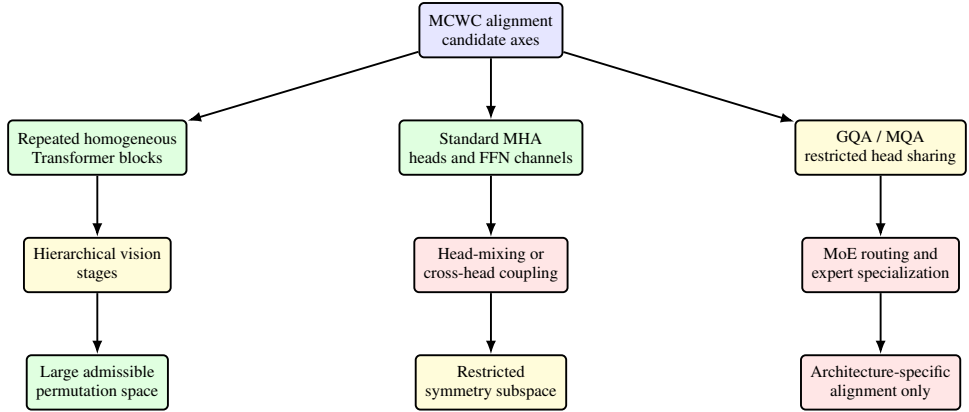

\paragraph{Standard multi-head attention.}
In standard multi-head attention with $H$ independent heads
~\citep{vaswani2017attention}, MCWC can permute attention heads when the query,
key, value, and output projection tensors are reindexed consistently. This
setting covers standard Transformer blocks used in many GPT-style, Pythia-style,
OPT-style, BERT-style, RoBERTa-style, T5-style, and ViT-style models
~\citep{radford2019language,biderman2023pythia,zhang2022opt,devlin2019bert,liu2019roberta,raffel2020exploring,dosovitskiy2021vit}.
Let the head permutation be $\pi\in\mathfrak{S}_{H}$. The permutation is
admissible if the corresponding head slices of $W_Q$, $W_K$, and $W_V$ are
permuted by $\pi$, and the matching slices of the output projection $W_O$ are
permuted by $\pi^{-1}$. In this case, the order of heads changes internally,
but the concatenated multi-head output after the compensating output projection
remains functionally equivalent.

\begin{equation}
\left(
W_Q^{(h)},W_K^{(h)},W_V^{(h)}
\right)_{h=1}^{H}
\mapsto
\left(
W_Q^{(\pi(h))},W_K^{(\pi(h))},W_V^{(\pi(h))}
\right)_{h=1}^{H},
\qquad
W_O
\mapsto
\Pi^{-1}W_O.
\label{eq:mha_head_permutation}
\end{equation}
This is the most favorable attention setting for MCWC because the admissible
head-permutation space is large and the same type of block repeats across
layers.

\paragraph{Grouped-query attention.}
Grouped-query attention restricts the attention symmetry because multiple query
heads share a smaller number of key/value heads~\citep{ainslie2023gqa}. This
structure is used in modern efficient decoder models, including LLaMA-style and
Mistral-style architectures~\citep{touvron2023llama2,jiang2023mistral}. Let
$H_Q$ denote the number of query heads and $H_{KV}$ denote the number of
key/value heads, with $H_Q>H_{KV}$. Each key/value head is associated with a
group of query heads. Therefore, arbitrary permutations of all query heads are
not generally function-preserving unless the query-to-key/value grouping is
preserved.

Let $G_r\subset\{1,\ldots,H_Q\}$ denote the set of query heads assigned to
key/value group $r$, where $r\in\{1,\ldots,H_{KV}\}$. An admissible GQA
permutation must satisfy
\begin{equation}
\pi_Q(G_r)=G_{\pi_{KV}(r)}
\quad
\text{for all}
\quad
r\in\{1,\ldots,H_{KV}\}.
\label{eq:gqa_constraint}
\end{equation}
This means that MCWC may permute complete key/value groups and may also permute
query heads within a group when the implementation preserves the same
query-to-key/value assignment. However, it should not freely match a query head
from one group to a key/value head from another group unless the corresponding
routing of shared keys and values is also reindexed consistently.

\paragraph{Multi-query attention.}
Multi-query attention is more restrictive than GQA because all query heads share
a single key/value pair~\citep{shazeer2019fast}. In this case, the key/value
axis does not provide an independent head-permutation space. MCWC can still
align query-head-associated parameters when the output projection is
compensated, but it should not assume that each query head has an independently
permutable key and value head. Thus, the admissible attention alignment in MQA
is restricted to query-side permutations and output-side compensation:
\begin{equation}
\pi_Q \in \mathfrak{S}_{H_Q},
\qquad
W_O \mapsto \Pi_Q^{-1}W_O,
\qquad
W_K,W_V \ \text{shared and not independently permuted}.
\label{eq:mqa_constraint}
\end{equation}
Consequently, MCWC is expected to obtain smaller attention-head alignment gains
in MQA than in standard MHA, while feed-forward channel alignment remains
available.

\paragraph{Head-mixing modules.}
Some architectures introduce explicit cross-head mixing, talking-head
attention, head-wise gating, or learned interactions between attention heads
~\citep{shazeer2020talkingheads}. In these cases, head permutations are
admissible only if the mixing operator is also conjugated by the same
permutation. Let $M$ denote a head-mixing matrix. A head permutation $\Pi$ is
function-preserving only if
\begin{equation}
M
\mapsto
\Pi M \Pi^{-1}.
\label{eq:head_mixing_constraint}
\end{equation}
If the implementation does not expose or reparameterize the head-mixing module
in this way, MCWC disables head-level alignment for that component and applies
alignment only to other admissible axes, such as feed-forward hidden channels.

\paragraph{Mixture-of-experts architectures.}
In mixture-of-experts architectures, experts may be permutable within a layer
only when the router logits, expert outputs, and expert indices are reindexed
consistently~\citep{shazeer2017outrageously,fedus2022switch,jiang2024mixtral}.
Let $E$ denote the number of experts and let $\Pi_E$ be an expert permutation.
Expert permutation is admissible only if the router output and the expert bank
are permuted together:
\begin{equation}
\mathrm{Router}(x)
\mapsto
\Pi_E\,\mathrm{Router}(x),
\qquad
\left\{F_e\right\}_{e=1}^{E}
\mapsto
\left\{F_{\pi_E(e)}\right\}_{e=1}^{E}.
\label{eq:moe_expert_constraint}
\end{equation}
However, cross-layer expert alignment is more delicate than standard FFN
alignment because experts may specialize differently across layers and routing
patterns may not define a stable depth-wise trajectory. For MoE models, MCWC
should therefore align only experts with compatible routing statistics or apply
expert-wise compression independently within each layer. Full MoE compression
is an important extension but requires router-aware alignment and is not treated
as a default assumption of MCWC.

\paragraph{Heterogeneous and hierarchical vision architectures.}
Hierarchical vision models, such as stage-based Transformers or convolutional
hybrids, often change width, resolution, or block structure across stages
~\citep{liu2021swin,liu2022convnext}. MCWC should not align blocks across
stage boundaries when the block dimensions or computational roles are not
comparable. Instead, alignment is performed within homogeneous stages. For
ViT-like and DeiT-like architectures, repeated Transformer blocks provide a
more direct setting for attention-head and MLP-channel alignment
~\citep{dosovitskiy2021vit,touvron2021deit}, while MLP-Mixer requires separating
token-mixing and channel-mixing subspaces~\citep{tolstikhin2021mlpmixer}:
\begin{equation}
\mathcal{A}
=
\bigcup_{s=1}^{S}
\left\{
(\ell-1,\ell):
\ell-1,\ell\in \mathcal{S}_s
\right\},
\label{eq:stage_restricted_alignment}
\end{equation}
where $\mathcal{S}_s$ denotes the set of layers belonging to stage $s$, and
$\mathcal{A}$ is the set of adjacent layer pairs eligible for alignment. This
stage-restricted rule prevents MCWC from forcing correspondences between
architecturally incompatible blocks.

\begin{table}[!t]
\centering
\caption{
Architecture-scope analysis for MCWC. The table lists common architecture
families, the main admissible symmetry axes, the restricted components that
require care, and the expected suitability of MCWC. This table is not a
performance table; it defines when the assumptions behind functional alignment
are structurally valid. The listed architecture families cover standard
Transformers~\citep{vaswani2017attention}, GPT-style decoders
~\citep{radford2019language}, Pythia and OPT models
~\citep{biderman2023pythia,zhang2022opt}, BERT/RoBERTa/T5 models
~\citep{devlin2019bert,liu2019roberta,raffel2020exploring}, LLaMA and
Mistral-style efficient decoders~\citep{touvron2023llama2,jiang2023mistral},
GQA/MQA variants~\citep{ainslie2023gqa,shazeer2019fast}, vision architectures
~\citep{dosovitskiy2021vit,touvron2021deit,liu2021swin,liu2022convnext,tolstikhin2021mlpmixer},
Talking-Heads attention~\citep{shazeer2020talkingheads}, and sparse MoE models
~\citep{shazeer2017outrageously,fedus2022switch,jiang2024mixtral}.
}
\label{tab:architecture_scope}
\scriptsize
\renewcommand{\arraystretch}{1.18}
\setlength{\tabcolsep}{3.5pt}
\resizebox{\linewidth}{!}{
\begin{tabular}{p{3.0cm}p{3.7cm}p{4.3cm}p{3.2cm}}
\toprule
\textbf{Architecture family}
&
\textbf{Admissible alignment axes}
&
\textbf{Restricted or nontrivial components}
&
\textbf{MCWC suitability}
\\
\midrule
GPT-style Transformer with MHA~\citep{radford2019language,vaswani2017attention}
&
FFN hidden channels; attention heads.
&
Requires consistent permutation of $Q,K,V$ head slices and inverse output
projection compensation.
&
High.
\\
\midrule
Pythia-style Transformer~\citep{biderman2023pythia}
&
Repeated decoder blocks; FFN channels; MHA heads.
&
Primarily standard Transformer symmetry constraints.
&
High.
\\
\midrule
OPT-style Transformer~\citep{zhang2022opt}
&
Repeated decoder blocks; FFN channels; MHA heads.
&
LayerNorm and residual paths must remain fixed; only internal compatible axes
are permuted.
&
High.
\\
\midrule
BERT/RoBERTa-style encoder~\citep{devlin2019bert,liu2019roberta}
&
Encoder FFN channels; MHA heads.
&
Bidirectional attention does not prevent head permutation, but all projection
slices must be compensated consistently.
&
High.
\\
\midrule
T5-style encoder--decoder~\citep{raffel2020exploring}
&
FFN channels; self-attention heads; cross-attention heads.
&
Self-attention and cross-attention should be aligned separately because they
play different functional roles.
&
Medium--high.
\\
\midrule
LLaMA-style GQA decoder~\citep{touvron2023llama2,ainslie2023gqa}
&
FFN channels; key/value groups; query heads within compatible groups.
&
Arbitrary query-head permutations are restricted by grouped key/value sharing.
&
Medium.
\\
\midrule
Mistral-style GQA decoder~\citep{jiang2023mistral,ainslie2023gqa}
&
FFN channels; grouped attention subspaces.
&
Sliding-window attention and GQA restrict unconstrained attention-head matching.
&
Medium.
\\
\midrule
MQA decoder~\citep{shazeer2019fast}
&
FFN channels; query-side head-associated slices.
&
Shared key/value projections reduce independent attention-head permutation
freedom.
&
Medium--low for attention; high for FFN.
\\
\midrule
ViT-B/16~\citep{dosovitskiy2021vit}
&
Repeated Transformer blocks; MHA heads; MLP hidden channels.
&
Patch embedding and classification head are not depth-sequential repeated
blocks.
&
High.
\\
\midrule
DeiT-B/16~\citep{touvron2021deit}
&
ViT-like MHA heads and MLP hidden channels.
&
Distillation token and classification heads should not be treated as generic
depth-aligned blocks.
&
High.
\\
\midrule
Swin-T~\citep{liu2021swin}
&
Within-stage attention/MLP blocks.
&
Stage changes, window shifts, and resolution changes restrict cross-stage
alignment.
&
Medium.
\\
\midrule
ConvNeXt-T~\citep{liu2022convnext}
&
Channel groups within repeated convolutional stages.
&
Depthwise convolution and stage transitions restrict direct Transformer-style
head alignment.
&
Medium.
\\
\midrule
MLP-Mixer~\citep{tolstikhin2021mlpmixer}
&
Token-mixing and channel-mixing MLP hidden dimensions.
&
Token-mixing and channel-mixing blocks should be aligned separately.
&
Medium.
\\
\midrule
Talking-Heads Transformer~\citep{shazeer2020talkingheads}
&
FFN channels; heads only if head-mixing matrices are compensated.
&
Cross-head mixing requires conjugating the mixing matrix by the same
permutation.
&
Restricted.
\\
\midrule
Switch Transformer / MoE~\citep{shazeer2017outrageously,fedus2022switch}
&
Experts within a layer; expert FFN channels.
&
Router logits and expert indices must be permuted consistently; cross-layer
expert matching requires router-aware alignment.
&
Restricted.
\\
\midrule
Mixtral-style sparse MoE~\citep{jiang2024mixtral}
&
Expert-internal FFN channels; compatible experts.
&
Sparse routing and expert specialization limit simple depth-wise expert
alignment.
&
Restricted.
\\
\bottomrule
\end{tabular}
}
\end{table}

\paragraph{Architecture stress-test protocol.}
To evaluate MCWC on architectures with different symmetry structures, we use
the following diagnostic protocol (see Table~\pinkref{tab:architecture_stress_protocol}). For each architecture, we identify the
admissible alignment axes, compute the fraction of parameters that belong to
alignable subspaces, and measure the change in cross-layer predictability after
alignment. This protocol separates architectural applicability from final
compression performance and is applicable across language, vision, and sparse
expert architectures~\citep{vaswani2017attention,radford2019language,biderman2023pythia,zhang2022opt,devlin2019bert,liu2019roberta,raffel2020exploring,dosovitskiy2021vit,touvron2021deit,liu2021swin,liu2022convnext,tolstikhin2021mlpmixer,shazeer2017outrageously,fedus2022switch,jiang2024mixtral}.

\begin{equation}
\mathrm{AlignableFraction}
=
\frac{
N_{\mathrm{alignable}}
}{
N_{\mathrm{param}}
},
\label{eq:alignable_fraction}
\end{equation}
where $N_{\mathrm{alignable}}$ is the number of scalar parameters that belong
to tensors with admissible permutation symmetries, and $N_{\mathrm{param}}$ is
the total parameter count of the checkpoint.

For each architecture, we then compute the alignment gain:
\begin{equation}
\Delta_{\mathrm{cos}}
=
\mathrm{CosSim}_{\mathrm{after}}
-
\mathrm{CosSim}_{\mathrm{before}},
\label{eq:cosine_gain_arch}
\end{equation}
and the normalized residual-energy reduction:
\begin{equation}
\Delta_{\mathrm{NRE}}
=
\frac{
\mathrm{NRE}_{\mathrm{before}}
-
\mathrm{NRE}_{\mathrm{after}}
}{
\mathrm{NRE}_{\mathrm{before}}
}.
\label{eq:nre_reduction_arch}
\end{equation}
These diagnostics test whether an architecture provides enough repeated,
alignable structure for MCWC to expose predictable depth-wise residuals.

\begin{table}[!t]
\centering
\caption{
Recommended architecture stress-test matrix. For each architecture, the
diagnostic records whether MCWC has access to large homogeneous repeated
subspaces, whether attention or expert symmetries are restricted, and which
alignment rule should be used. This matrix covers standard MHA Transformers
~\citep{vaswani2017attention}, GPT/Pythia/OPT/BERT/RoBERTa/T5-style models
~\citep{radford2019language,biderman2023pythia,zhang2022opt,devlin2019bert,liu2019roberta,raffel2020exploring},
GQA/MQA models~\citep{ainslie2023gqa,shazeer2019fast}, vision architectures
~\citep{dosovitskiy2021vit,touvron2021deit,liu2021swin,liu2022convnext,tolstikhin2021mlpmixer},
head-mixing attention~\citep{shazeer2020talkingheads}, and sparse MoE models
~\citep{shazeer2017outrageously,fedus2022switch,jiang2024mixtral}.
Architectures with restricted symmetries are treated with restricted alignment
rather than unconstrained block matching.
}
\label{tab:architecture_stress_protocol}
\scriptsize
\renewcommand{\arraystretch}{1.15}
\setlength{\tabcolsep}{3.5pt}
\resizebox{\linewidth}{!}{
\begin{tabular}{p{3.0cm}p{2.6cm}p{3.2cm}p{4.4cm}}
\toprule
\textbf{Architecture}
&
\textbf{Symmetry level}
&
\textbf{Primary diagnostic}
&
\textbf{MCWC rule}
\\
\midrule
GPT-style MHA~\citep{radford2019language,vaswani2017attention}
&
Large
&
FFN/head cosine gain; residual reduction.
&
Unrestricted FFN and standard head alignment.
\\
Pythia-style~\citep{biderman2023pythia}
&
Large
&
Depth-wise FFN/head predictability.
&
Standard MCWC alignment.
\\
OPT-style~\citep{zhang2022opt}
&
Large
&
Depth-wise FFN/head predictability.
&
Standard MCWC alignment.
\\
BERT/RoBERTa-style~\citep{devlin2019bert,liu2019roberta}
&
Large
&
Encoder-block alignment gain.
&
Separate FFN and attention-head alignment.
\\
T5-style~\citep{raffel2020exploring}
&
Moderate--large
&
Self-attention vs cross-attention predictability.
&
Separate self-attention and cross-attention alignment.
\\
LLaMA-style GQA~\citep{touvron2023llama2,ainslie2023gqa}
&
Restricted
&
Group-preserving head alignment gain.
&
Permute only compatible GQA groups.
\\
Mistral-style GQA~\citep{jiang2023mistral,ainslie2023gqa}
&
Restricted
&
Within-group query/head residual reduction.
&
Group-preserving attention alignment.
\\
MQA decoder~\citep{shazeer2019fast}
&
Restricted
&
Query-side gain; FFN gain.
&
Do not permute shared $K,V$ as independent heads.
\\
ViT-B/16~\citep{dosovitskiy2021vit}
&
Large
&
Block-wise MLP/head predictability.
&
Standard MCWC alignment for repeated blocks.
\\
DeiT-B/16~\citep{touvron2021deit}
&
Large
&
MLP/head predictability excluding heads.
&
Exclude distillation/classification heads.
\\
Swin-T~\citep{liu2021swin}
&
Moderate
&
Within-stage residual reduction.
&
Stage-restricted alignment only.
\\
ConvNeXt-T~\citep{liu2022convnext}
&
Moderate
&
Within-stage channel-group predictability.
&
Channel-group alignment within homogeneous stages.
\\
MLP-Mixer~\citep{tolstikhin2021mlpmixer}
&
Moderate
&
Token-mixing vs channel-mixing predictability.
&
Separate token and channel MLP alignment.
\\
Talking-Heads Transformer~\citep{shazeer2020talkingheads}
&
Restricted
&
Effect of head-mixing compensation.
&
Permute heads only if mixing matrices are conjugated.
\\
Switch/MoE~\citep{shazeer2017outrageously,fedus2022switch}
&
Restricted
&
Router-compatible expert similarity.
&
Router-aware expert alignment or expert-local compression.
\\
Mixtral-style MoE~\citep{jiang2024mixtral}
&
Restricted
&
Expert specialization and routing overlap.
&
Restrict alignment to compatible experts.
\\
\bottomrule
\end{tabular}
}
\end{table}

\paragraph{Limitations for weak-symmetry architectures.}
The analysis above clarifies a limitation of MCWC. The method should not be
interpreted as assuming universal permutation freedom across all neural
architectures. Its strongest regime is repeated homogeneous architectures with
large internal permutation spaces, as in standard Transformer, GPT/Pythia/OPT,
BERT/RoBERTa/T5, and ViT/DeiT-style models
~\citep{vaswani2017attention,radford2019language,biderman2023pythia,zhang2022opt,devlin2019bert,liu2019roberta,raffel2020exploring,dosovitskiy2021vit,touvron2021deit}.
For GQA and MQA models, attention-head alignment must respect shared key/value
structure~\citep{ainslie2023gqa,shazeer2019fast}. For head-mixing models, head
permutations require consistent reparameterization of the mixing operator
~\citep{shazeer2020talkingheads}. For MoE models, expert permutations require
router-aware compensation, and cross-layer expert alignment may be unreliable
when experts specialize differently across depth
~\citep{shazeer2017outrageously,fedus2022switch,jiang2024mixtral}. For
hierarchical vision models, alignment should be restricted to homogeneous
stages rather than applied across stage boundaries
~\citep{liu2021swin,liu2022convnext}.

\paragraph{Takeaway.}
MCWC is broadly applicable as a checkpoint codec, but its alignment component
must be constrained by the symmetries actually present in the architecture.
When large repeated homogeneous blocks are available, MCWC can exploit
depth-wise predictability through functional alignment and residual coding.
When symmetries are restricted, MCWC remains applicable only after limiting the
alignment search to symmetry-preserving subspaces. This architecture-aware
formulation prevents invalid permutations and explains why the largest gains
are expected for standard repeated Transformer and ViT-like models, while
GQA/MQA, MoE, head-mixing, and strongly heterogeneous architectures require
more careful treatment~\citep{vaswani2017attention,dosovitskiy2021vit,ainslie2023gqa,shazeer2019fast,shazeer2020talkingheads,shazeer2017outrageously,fedus2022switch,jiang2024mixtral}.

\section{Additional Related Work on Symmetry-Aware and Cross-Layer Weight Compression}
\label{app:additional_related_work}

Recent work has begun to study neural network weights as structured objects
whose symmetries and cross-layer regularities can be exploited for storage.
Bits-back coding for rotational symmetries in LLMs shows that some weight-space
degrees of freedom are redundant and can be used to recover storage bits without
changing model behavior~\citep{he2024freebits}. This direction is related to
MCWC because both methods exploit function-preserving reparameterizations, but
they differ in the symmetry group and coding mechanism: bits-back rotational
coding targets continuous rotational symmetries, whereas MCWC uses
permutation-based motion compensation and predictive residual coding across
depth.

Cross-layer sparse dictionary learning methods such as CoSpaDi represent LLM
weights through calibration-guided sparse dictionary factors rather than
layer-independent low-rank approximations~\citep{shopkhoev2025cospadi}. This
is related to MCWC in that both methods exploit structure beyond naive
per-tensor quantization. However, CoSpaDi focuses on sparse factorization of
weight matrices using dictionary atoms, whereas MCWC constructs an aligned
coordinate system and then codes layer-to-layer residuals with an entropy model.

Neural Weight Compression trains learned codecs directly on language-model
weights, treating model parameters as a compression source distinct from images
or activations~\citep{ryu2025nwc}. MCWC shares the view that model weights are a
specialized data modality requiring dedicated codecs, but differs by explicitly
using function-preserving alignment, keyframes, and depth-wise predictive
coding. Earlier work on inter prediction and linear transformation for neural
network compression also explored predictive coding across weight tensors
~\citep{lee2021interprediction}. MCWC extends this general idea by introducing
permutation-based motion compensation, learned entropy coding, and an explicit
rate--distortion objective for checkpoint-level compression.

\section{Discussion,  Limitations and Future Work}
\label{sec:discussion_limitations}

\subsection{Discussion}
\label{sec:discussion}

The central empirical message of MCWC is that pretrained networks should not be viewed only as collections of independent tensors. Standard layerwise quantization and pruning methods compress each layer mostly in isolation, which is effective for reducing local numerical precision but leaves depth-wise redundancy largely unused. MCWC instead treats the ordered sequence of layers as a structured signal. After aligning permutation-symmetric blocks, adjacent layers become more comparable, and the decoder can predict many blocks from previously decoded context. This interpretation explains the behavior observed in Table~\pinkref{tab:pareto_points} and Figure~\pinkref{fig:rd_curve}: the advantage of MCWC becomes especially visible in the mid- and low-rate regimes, where independent quantization has fewer bits available per tensor and therefore accumulates larger distortion. In contrast, MCWC spends bits primarily on information that is difficult to predict, namely keyframes, residuals, and compact side information.

A key lesson from the method is that prediction becomes effective only when the layer sequence is expressed in a consistent coordinate system. Permutation symmetries create many functionally equivalent parameterizations of the same model, but some of these parameterizations make cross-layer correspondence easier to model than others. MCWC uses this freedom to choose blockwise permutations that reduce apparent motion in weight space. The alignment diagnostics in Appendix~\ref{app:alignment_diagnostics} and the ablations in Table~\pinkref{tab:ablations} support this mechanism: when alignment is removed or replaced by random alignment, residual energy increases and the entropy model receives a less regular symbol stream. This shows that the compression gain does not come only from adding an entropy coder after quantization; it comes from changing the representation so that predictive residual coding becomes statistically easier.

MCWC obtains its storage gains from the interaction of alignment, prediction, and entropy modeling. Alignment reduces mismatch between corresponding blocks across depth. Prediction removes the smooth component of layer-to-layer variation using decoded context, which matches the information available at deployment time. Entropy coding then assigns shorter codes to residual symbols that occur frequently under the learned probability model. These components are complementary rather than interchangeable. If prediction is removed, aligned layers still contain redundancy, but the codec cannot exploit it as directly. If entropy modeling is removed, residuals may be numerically small but are not encoded at a rate close to their empirical distribution. If keyframes are removed or spaced too far apart, reconstruction errors can propagate through the sequential decoder. The full method is therefore best understood as a checkpoint-level predictive codec rather than as a modified quantizer.

MCWC is designed for settings in which model checkpoints must be stored, transferred, replicated, or repeatedly materialized. It is not intended to replace low-precision matrix multiplication kernels or to directly accelerate every inference step. The offline encoder can be more expensive than simple post-training quantization because it must compute alignments, optimize predictor and entropy components, quantize residuals, and serialize a complete bitstream. This cost is justified when the compressed artifact is reused many times, for example in model hubs, serverless cold starts, distributed checkpoint loading, or bandwidth-limited deployment. Appendix~\ref{app:deployment_regime} describes this amortization view: the encoding cost is paid once, while storage and transfer savings accumulate across repeated uses. After decoding, the reconstructed weights can still be passed to standard inference systems or combined with hardware-aware quantization pipelines.

A practical checkpoint codec must report the size of everything needed for reconstruction, not only the entropy of the main symbol stream. MCWC therefore accounts for residual and keyframe codes, permutation side information, quantizer parameters, headers, metadata, and coding-mode information. This is important because side information can be non-negligible, especially at low bitrates. The comparison with pruning-based and learned scalar-quantization baselines should therefore be read as a comparison between complete serialized representations. Under this accounting, MCWC remains competitive because the reduction in residual entropy is large enough to compensate for the additional metadata introduced by alignment and prediction. This full-bitstream perspective also highlights where further improvements are most likely to matter: cheaper permutation coding, adaptive keyframe placement, and more compact quantizer-side information.

The language and vision experiments suggest that the MCWC principle applies beyond a single task type, since both Transformer language models and ViT classifiers contain repeated blocks with comparable internal structure. However, the strength of the method is tied to how much valid cross-layer correspondence exists after respecting the architecture's symmetries. Models with homogeneous repeated layers are naturally suited to the approach, while models with heterogeneous stages, strong head mixing, mixture-of-experts routing, or restricted attention layouts may require more careful alignment groups. This does not change the core idea, but it shifts part of the burden from the generic codec to architecture-aware symmetry identification. In this sense, MCWC opens a design direction: future compression systems can treat neural weights as structured trajectories through depth rather than as isolated arrays.

\subsection{Limitations}
\label{sec:limitations}
$\triangleright$ MCWC has a heavier offline encoder than zero-shot PTQ because it performs
alignment, predictor/entropy-model optimization, residual quantization, and
bitstream serialization. This cost is acceptable when the compressed checkpoint
is reused across many deployments, downloads, or cache materializations, but it
is less attractive for a one-off compression setting with abundant storage and
bandwidth. Appendix~\pinkref{app:encoding_complexity} quantifies this cost, and
Appendix~\pinkref{app:deployment_regime} shows how it is amortized in repeated
deployment regimes.

$\triangleright$ MCWC is most effective when the architecture contains repeated homogeneous
blocks with valid function-preserving permutation symmetries, such as standard
Transformer FFN channels, independent attention heads, and ViT-like repeated
blocks. Architectures with restricted symmetries, including GQA/MQA, explicit
head mixing, heterogeneous hierarchical stages, or MoE routing, require
alignment to be restricted to symmetry-preserving subspaces. These cases do not
invalidate MCWC, but they may reduce the alignment gain or require
architecture-aware alignment rules, as detailed in Appendix~\pinkref{app:architecture_scope}.

$\triangleright$ MCWC transmits side information for permutations, quantizer parameters,
keyframes, and metadata, and non-keyframe layers are decoded from previously
reconstructed context. These design choices are central to the compression
gain, but they introduce trade-offs: permutation metadata can become more
important at very low rates or very large widths, and layer-sequential
prediction can reduce fully independent layer-wise loading. MCWC mitigates
these issues with compact permutation coding and periodic keyframes, which
create independently decodable segments; Appendices~\pinkref{app:permutation_side_info}
and~\pinkref{app:sequential_loading} provide the detailed accounting and loading
analysis.

\subsection{Future directions}
\label{sec:future_directions}

Future work can improve MCWC by making its design choices more adaptive and hardware-aware. A first direction is rate-aware alignment, where the encoder jointly considers residual reduction and permutation side-information cost instead of optimizing alignment only for similarity or residual energy. A second direction is adaptive keyframe scheduling, where keyframes are placed according to prediction uncertainty, layer sensitivity, or distributed-loading constraints rather than using a fixed interval. A third direction is extending the symmetry analysis to architectures with restricted or conditional structure, including GQA, MQA, MoE routing, multimodal fusion blocks, and heterogeneous stages. Finally, practical deployment would benefit from faster decoding kernels, better integration with inference-time quantization, and standardized checkpoint formats that support predictive coding while preserving compatibility with existing model-serving systems.


\end{document}